\documentclass[11pt]{article}
\usepackage[utf8]{inputenc}

\title{Generalization Guarantees for Multi-Input Neural Operator Learning in Sobolev Spaces}

\author{
Yahong Yang\thanks{School of Mathematics, Georgia Institute of Technology,
686 Cherry Street, Atlanta, GA 30332-0160, USA.
Emails: \texttt{yyang3194@gatech.edu},
\texttt{weizhu@gatech.edu},
\texttt{wliao60@gatech.edu}}
\and
Zecheng Zhang\thanks{Department of Applied and Computational Mathematics and Statistics,
University of Notre Dame, Notre Dame, IN 46556, USA.
Email: \texttt{zzhang48@nd.edu}}
\and
Wei Zhu\footnotemark[1]
\and
Wenjing Liao\footnotemark[1]
\and
Hao Liu\thanks{Department of Mathematics, Hong Kong Baptist University,
FSC1202, Fong Shu Chuen Building, Hong Kong Baptist University,
Kowloon Tong, Hong Kong.
Email: \texttt{haoliu@hkbu.edu.hk}. Corresponding author.}
}
\usepackage[numbers]{natbib}
\usepackage{amsmath}
\usepackage{epsfig,amsthm}
\usepackage{amssymb,amsfonts}
\usepackage{dsfont}
\usepackage{subfigure}
\usepackage{indentfirst}
\usepackage{mathrsfs}
\usepackage{xcolor}
\usepackage{graphicx}
\usepackage{zymacros}
\usepackage{hyperref}
\usepackage[ruled,linesnumbered]{algorithm2e} 
\usepackage{appendix}
\usepackage{ulem}

\DeclareMathOperator*{\argmin}{arg\,min}
\newtheorem{setting}{Setting}

\allowdisplaybreaks

\date{}
\begin{document}
\maketitle
\begin{abstract}
We develop approximation and generalization error estimates for multi-input
neural operators, with the output error measured in Sobolev norms. In contrast
to standard operator-learning settings with a single input function, our
framework allows multiple input functions defined on possibly different
domains, with different dimensions and Sobolev regularities. The derived rates
explicitly quantify the contribution of each input space to the final error
bound. In particular, in the balanced regime, the approximation and
generalization rates are governed by the interaction between the input
dimensions, regularities, and Sobolev orders, while the dependence on the model
complexity retains a \(\log\log/\log\)-type structure. Our analysis provides a
general theoretical framework for multi-input operator learning, including
Sobolev training, and is applicable to operator learning problems arising from
partial differential equations and scientific computing.
\end{abstract}

\section{Introduction}

Operator learning aims to approximate mappings between infinite-dimensional
function spaces and has become an important tool for solving parametric families
of partial differential equations (PDEs). Prominent examples include DeepONet
\cite{lu2019deeponet,wang2021learning,liu2022deep,liu2024neural,lanthaler2022error},
Fourier neural operators (FNOs) \cite{li2020fourier}, and other structure-preserving operator-learning
architectures
\cite{li2020fourier,li2023fourier,kovachki2021universal,he2024mgno}.
Compared with classical solvers that must be repeatedly applied for each new
instance, a trained neural operator can approximate the solution map for new
input functions, source terms, boundary conditions, or physical parameters with
a single forward evaluation. This feature makes neural operators especially
useful in scientific computing tasks involving many-query or real-time
simulation of PDE models. {A simple example is the elliptic boundary value problem
\begin{equation}
-\Delta u+a(x)u=f,\qquad u|_{\partial\Omega}=g.\label{eq: PDEs}
\end{equation}
In this case, the solution operator depends on three inputs, namely the
coefficient \(a\), the source term \(f\), and the boundary condition \(g\), and
can therefore be written as
\(\mathcal G(a,f,g)=u\). This illustrates a basic feature of many PDE problems
in scientific computing: the solution often depends simultaneously on several
varying quantities. 
For instance, the coefficient \(a\) may encode material
properties or the underlying medium, the source term \(f\) may represent
external forcing, and the boundary condition \(g\) may also vary from one
instance to another. 
Thus, it is natural to study operator learning in a
multi-input setting, where the learned operator takes all of these quantities as
inputs rather than treating only one or two of them as variable.}

Motivated by the need to apply operator learning to more complex scientific
computing problems, multi-input operator learning has recently attracted
increasing attention
\cite{back2002universal, hu2025hybrid,weihs2026generalization,weihs2025deep,jin2022mionet,weihs2026multiple}.
For example in \cite{back2002universal}, the authors firstly proposed the universal approximation to an operator indexed by a function which is a two-input operator setting.
More general, the target operator depends on multiple input functions.
Precisely, one considers an operator
\begin{equation}
\mathcal G:\prod_{i=1}^{\lambda}\mathcal X_i\to \mathcal V,
\qquad \lambda\ge 2,
\label{eq:introG}
\end{equation}
and aims to construct a neural operator \(\mathcal G_{\boldsymbol\theta}\) that
approximates \(\mathcal G\) uniformly or in an appropriate statistical sense.
Universal approximation results for multi-input operator-learning architectures
have been established in
\cite{hu2025hybrid,back2002universal,jin2022mionet}, and both theoretical and
numerical studies have demonstrated the effectiveness of such frameworks.
However, quantitative scaling laws and error rates for multi-input neural
operators remain much less understood.

{Recent works~\cite{weihs2026generalization,weihs2025deep,weihs2026multiple}
established scaling laws for operator learning indexed by functions \cite{back2002universal}.
Their framework mainly focuses on the two-input setting and assumes
Lipschitz continuity in the \(L^\infty\) sense.
In this work, we analyze the general multiple-input operator
learning, where the number of inputs can be any integer greater than or equal to
two. 
In addition, what is important is we allow the inputs to lie in Sobolev spaces, so that the regularity of each input is explicitly reflected in the rate, making it
possible to quantify how the smoothness of different inputs influences the final
generalization rate and the corresponding neural network design. 
Finally, we
measure the output error in Sobolev norms rather than only in the
\(L^\infty\) norm, which is more natural and informative in scientific
computing applications. 
This paper aims to quantify precisely how input-wise dimensions and
regularities influence the final learning rate, and to design a neural operator
architecture whose complexity reflects this structure. 

These extensions are
particularly important for scientific computing applications. First, many
partial differential equations depend on more than two inputs. For example, a
typical boundary value problem \eqref{eq: PDEs} may involve a coefficient \(a\), a source term
\(f\), and a boundary condition \(g\). Therefore, the two-input setting is not
sufficient for capturing many practically relevant PDE models. Second,
different inputs often play fundamentally different roles in the underlying
equation, and may possess very different regularities and intrinsic dimensions.
As a result, the network structures associated with different inputs should not be
treated uniformly. 
Understanding how to design neural network architectures that
reflect the dimension and regularity of each input is thus an important and
challenging problem. In practice, such design choices are often made only
through extensive trial and error, whereas our framework provides a
mathematically grounded perspective on this issue. Finally, if operator learning is to be applied to PDE problems in scientific
computing, it is essential to measure the error in Sobolev norms rather than
only in \(L^\infty\). Even for weak solutions, one often needs derivative information, and hence the extension to Sobolev norms is both
natural and necessary. 
This is also closely related to a very active direction
in scientific machine learning, usually referred to as Sobolev training
\cite{czarnecki2017sobolev,vlassis2021sobolev,vlassis2020geometric,srinivas2018knowledge,hill2026geometric,goswami2022physics}.}

We first state an informal version of our main approximation result; the full
statement is given in Corollary~\ref{cor:more-balanced-main-final-functionL}.

\begin{theorem}[Informal Sobolev approximation result]
\label{thm:inform-approx}
Let \(\ell\in\{0,1\}\). Let
\(\mathcal X_i\subset W^{n_i,\infty}([-1,1]^{d_i})\),
\(i=1,\ldots,\lambda\), be uniformly bounded input classes, and let
\(\mathcal G:\prod_{i=1}^{\lambda}\mathcal X_i
\to W^{\ell,\infty}(\Omega_{\lambda+1})\) be a \(\lambda\)-input operator.
Assume that \(\mathcal G\) is separately Lipschitz continuous with respect to
the input \(L^\infty\)-norms, and that its output
\(W^{n_{\lambda+1},\infty}(\Omega_{\lambda+1})\)-norm is uniformly controlled,
with \(n_{\lambda+1}>\ell\). Then there exists a ReLU neural operator
\(\mathcal G_{\boldsymbol\theta}\) with \(N_{\mathrm{tot}}\) trainable
parameters such that, up to lower-order logarithmic factors,
\[
\sup_{\prod_{i=1}^{\lambda}\mathcal X_i}
\left\|
\mathcal G(f_1,\ldots,f_\lambda)
-
\mathcal G_{\boldsymbol\theta}(f_1,\ldots,f_\lambda)
\right\|_{W^{\ell,\infty}(\Omega_{\lambda+1})}
\lesssim
\left(
\frac{\log N_{\mathrm{tot}}}{\log\log N_{\mathrm{tot}}}
\right)^{-\frac{1}{Q_{\max}}}
(\log\log N_{\mathrm{tot}})^{\sum_{i=1}^{\lambda}d_i},
\]
where \(Q_{\max}:=\max_{1\le i\le\lambda} d_i/n_i\).
\end{theorem}

The rate above shows that the final approximation complexity is
controlled by the most difficult input space. More precisely, an input space
\(\mathcal X_i\) contributes through the effective complexity \(d_i/n_i\). If
\(\mathcal X_i\) has a low-dimensional domain or high regularity, then
\(d_i/n_i\) is small. Such an input does not change the leading exponent unless
it becomes the largest among all input complexities. Conversely, an input with
large dimension or low regularity may dominate the maximum \(Q_{\max}\), thereby
determining the final approximation rate.

The proof also suggests how the neural operator should be constructed. The
architecture has a shared branch--trunk form
\begin{equation}
\begin{aligned}
\mathcal G_{\boldsymbol\theta}(f_1,\ldots,f_\lambda)(\boldsymbol{x})
:=
\sum_{s_1=1}^{J_1}\cdots
\sum_{s_\lambda=1}^{J_\lambda}
\sum_{p=1}^{P}
e_{s_1,\ldots,s_\lambda,p}
\prod_{i=1}^{\lambda}
\mathcal B_{i,s_i}(\mathcal D_{m_i}f_i)
\,\mathcal T_p(\boldsymbol{x}).
\end{aligned}
\label{eq:1lambda-shared-architecture}
\end{equation}
Here \(\mathcal D_{m_i}f_i\) denotes the finite-dimensional discretization of
the \(i\)-th input function, \(\mathcal B_{i,s_i}\) are input-wise branch
networks, and \(\mathcal T_p\) are shared trunk networks. A schematic diagram of
this neural-operator architecture is shown in Fig.~\ref{fig:NN-architecture}.

\begin{figure}[htbp]
    \centering

    \includegraphics[width=0.97\linewidth]{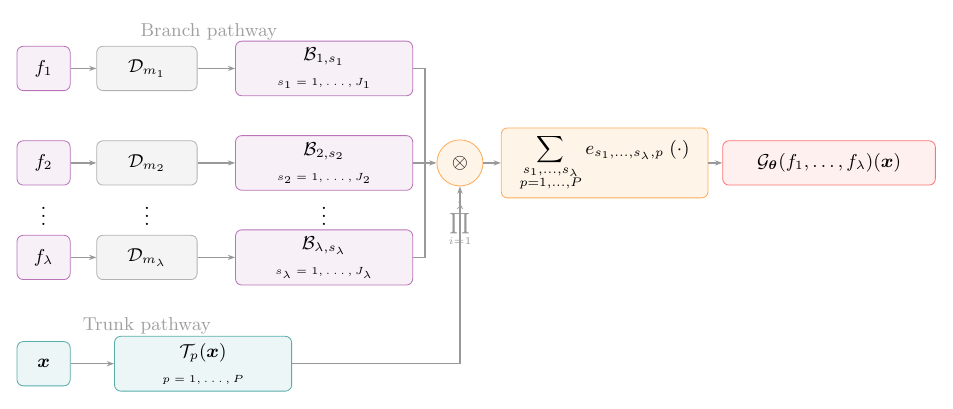}
    \caption{Schematic illustration of the shared branch--trunk architecture
    for the \(\lambda\)-input neural operator. Each input function \(f_i\) is
    first discretized by \(\mathcal D_{m_i}\), then passed through an
    input-wise branch network \(\mathcal B_{i,s_i}\). The branch outputs are
    multiplied, combined with the shared trunk network \(\mathcal T_p\), and
    summed over the indices \(s_1,\ldots,s_\lambda,p\).}
    \label{fig:NN-architecture}
\end{figure}
If the input space \(\mathcal X_i\) has high dimension or low regularity, then a
larger discretization level \(m_i\) is needed to control the discretization
error. This, in turn, requires a larger branch rank \(J_i\). In the balanced
construction, up to logarithmic factors,
\[
(m_i+1)^{d_i}
\asymp
\varepsilon^{-d_i/n_i},
\qquad
\log J_i
\asymp
(m_i+1)^{d_i}\log (m_i+1).
\]
Thus, the derived rate not only quantifies the approximation complexity, but
also provides guidance on how to allocate network capacity across different
input functions. A more detailed discussion is given in
Remark~\ref{rmk:bla-dis}.

Using the approximation error estimates established in
Sec.~\ref{sec:approx}, together with explicit bounds on the network size and
parameter magnitudes, we derive generalization error estimates for the proposed
multi-input neural operator class in Sec.~\ref{sec:gen}. The analysis is carried
out under a hierarchical sampling setting: for each outer input, one observes
multiple inner input samples and spatial evaluation points. Since samples
sharing the same outer input are not fully independent at the outer level, the
dominant statistical error is controlled by the number of outer samples. The proof follows the empirical-process strategy of \cite{liu2024neural}, but
with an additional difficulty caused by Sobolev training. In our setting, the
loss involves derivative observations, and the derivative of the ReLU activation
is not Lipschitz continuous. Therefore, the standard parameter-perturbation
covering argument is not directly applicable to the derivative classes. To
overcome this issue, we use uniform empirical covering numbers and bound them
through pseudo-dimension estimates, based on
\cite{anthony1999neural,yang2025deeponet}. This yields quantitative
generalization bounds depending explicitly on the network complexity and the
training sample size. The following informal theorem summarizes the resulting
generalization rate.
\begin{theorem}[Informal Sobolev generalization result]
\label{thm:inform-generalization}
Let \(\ell\in\{0,1\}\). Let
\(\mathcal X_i\subset W^{n_i,\infty}([-1,1]^{d_i})\),
\(i=1,\ldots,\lambda\), be uniformly bounded input classes, and let
\(\mathcal G:\prod_{i=1}^{\lambda}\mathcal X_i
\to W^{\ell,\infty}(\Omega_{\lambda+1})\) be a \(\lambda\)-input operator.
Assume that \(\mathcal G\) is separately Lipschitz continuous with respect to
the input \(L^{\infty}\)-norms, and that the output has sufficient
Sobolev regularity \(n_{\lambda+1}>\ell\). Let \(n_1^{\rm samp}\) denote the number of outer training samples in the
hierarchical sampling setting, and define
\(
Q_{\max}:=
\max_{1\le i\le\lambda}
\frac{d_i}{n_i}.
\)
Then there exists a ReLU neural operator class such that
the empirical risk minimizer \(\widehat{\mathcal G}_{\mathcal S}\), trained with
the Sobolev empirical loss, satisfies
\begin{align}
&\mathbb E_{\mathcal S}
\mathbb E_{\substack{
f_i\sim\mu_i,\ i=1,\ldots,\lambda\\
\vz\sim\mu_z
}}
\left[
\sum_{|\gamma|\le \ell}
\left|
\partial^\gamma
\mathcal G(f_1,\ldots,f_\lambda)(\vz)
-
\partial^\gamma
\widehat{\mathcal G}_{\mathcal S}
(\mathcal D_{m_1}f_1,\ldots,\mathcal D_{m_\lambda}f_\lambda)(\vz)
\right|^2
\right]
\notag\\
&\qquad\lesssim
\left(
\frac{\log n_1^{\rm samp}}
{\log\log n_1^{\rm samp}}
\right)^{-\frac{2}{Q_{\max}}}
(\log\log n_1^{\rm samp})^{2\sum_{i=1}^{\lambda}d_i}.
\end{align}
Here \(\mathbb E_{\mathcal S}\) is taken over the training set, while
\(\mathbb E_{f_i\sim\mu_i,\,\vz\sim\mu_z}\) is taken over independent test
inputs and test output locations.
\end{theorem}

Theorems~\ref{thm:inform-approx} and \ref{thm:inform-generalization} not only
capture the regularity of the input functions, but also measure the output
error in the Sobolev norm \(W^{\ell,\infty}\). To the best of our knowledge,
such a Sobolev-output error estimate has not appeared in previous works on
multi-input neural operator learning. In this paper, we focus on the cases \(\ell=0\) and \(\ell=1\), since our
network architecture is based on ReLU activation functions. Extending the
results to \(\ell\ge2\) would require smoother activation functions, together
with corresponding Sobolev approximation results based on partition-of-unity
constructions; see, for example, \cite{yang2025deep}. We expect that the
overall approximation framework and proof strategy can be extended to this
higher-order Sobolev setting, although we do not pursue this direction here.

\subsection*{Related Work}

There have been many works on the error analysis of functional learning and
operator learning with inputs taken from infinite-dimensional spaces. For
functional learning, see
\cite{chen1993approximations,li2026sparse,yang2026spherical,shi2025nonlinear,yang2022approximation,mhaskar1997neural};
for operator learning, see
\cite{chen1995universal,back2002universal,lanthaler2022error,lanthaler2023operator,kovachki2021universal,liu2022deep,jin2022mionet,hu2025hybrid,liu2024neural,hao2026multiscale,yang2025deeponet,weihs2026generalization,weihs2025deep,weihs2026multiple,liu2025generalization,schwab2026deep,marcati2023exponential,yang2026efficient}.
Among these works, except for
\cite{back2002universal,jin2022mionet,hu2025hybrid,weihs2026generalization,weihs2025deep,weihs2026multiple},
most focus on the single-input setting. If we specialize our results
(Theorems~\ref{thm:inform-approx} and \ref{thm:inform-generalization}) to the
case \(\lambda=1\), then our framework reduces to the single-input setting, and
the resulting rates are consistent with those in the previous literature.

In some works, the rates are better than those obtained here because they
consider smaller or more structured input spaces, such as mixed-Sobolev spaces
\cite{li2026sparse}, finite basis expansion spaces \cite{liu2024neural}, or
infinitely smooth spaces \cite{lanthaler2022error}. Improved rates are also
possible when the target functional or operator admits additional low-complexity
structure, such as Barron-type structure \cite{yang2022approximation} or
Green's function structure \cite{hao2026multiscale}. These settings go beyond
the scope of the present paper and will be considered in future work.

Among the existing multi-input works
\cite{back2002universal,jin2022mionet,hu2025hybrid,weihs2026generalization,weihs2025deep,weihs2026multiple}, the work \cite{back2002universal} first established universal approximation of two-input operator; 
the papers \cite{jin2022mionet,hu2025hybrid} establish
general multi-input universal approximation results, but do not provide quantitative scaling laws.
The papers
\cite{weihs2026generalization,weihs2025deep,weihs2026multiple}
derive such scaling laws, with \cite{weihs2026multiple} giving the strongest
result among them. That framework only considers the two-input setting
and assumes that the inputs belong merely to Lipschitz classes, while the
operator is Lipschitz continuous with respect to both inputs and outputs in the
\(L^\infty\) sense. 
These assumptions and the proof strategy generalize those in \cite{liu2024neural} for the single-input setting. 

In contrast, our work is motivated by scientific computing applications, where
both the input and output spaces are more naturally measured in Sobolev norms,
and where Lipschitz continuity is more appropriately formulated in Sobolev
spaces. Moreover, since each input is allowed to belong to a Sobolev space with
its own regularity, our rate explicitly reflects the smoothness of each input.
In this sense, our result is sharper than those in
\cite{weihs2026multiple,liu2024neural}. The rate in
\cite{liu2024neural,weihs2026multiple} takes the form
\[
\left(
\frac{\log N_{\mathrm{tot}}}{\log\log N_{\mathrm{tot}}}
\right)^{-\frac{1}{\max_{1\le i\le\lambda} d_i}},
\]
where \(\lambda=1\) corresponds to the setting in \cite{liu2024neural}, while
\(\lambda=2\) corresponds to the setting in \cite{weihs2026multiple}. By
contrast, our rate becomes
\[
\left(
\frac{\log N_{\mathrm{tot}}}{\log\log N_{\mathrm{tot}}}
\right)^{-\frac{1}{\max_{1\le i\le\lambda} d_i/n_i}}
\]
for any $\lambda>0$,
thereby incorporating the input-wise regularity \(n_i\). For this reason, the
framework of \cite{weihs2026multiple,liu2024neural}, which relies on
hat-function encoding, cannot be applied. Instead, we use a
pseudo-spectral projection, which is better suited to Sobolev spaces and allows
us to define and control Sobolev norms in a natural way.

Furthermore, measuring the output error in Sobolev norms \(W^{\ell,\infty}\)
instead of \(L^\infty\) leads to two essential differences in the analysis
compared with \cite{liu2024neural}. First, in the approximation step, we need
approximation results for the trunk network in Sobolev norms, which can be
obtained by following
\cite{guhring2020error,yang2023nearly,yang2023nearlys}. Second, in the
generalization error analysis, our empirical loss is based on derivatives,
namely,
\[
\mathcal L_{\mathcal S}(\mathcal G_{\rm NN})
:=
\frac{1}{N_{\rm train}}
\sum_{\operatorname{data}=1}^{N_{\rm train}}
\sum_{|\gamma|\le\ell}
\left|
{\partial^\gamma}
\mathcal G_{\rm NN}
\mid_{\operatorname{data}}
-
\text{data of derivative}
\right|^2,
\]
rather than the standard \(L^2\)-type loss. This is closely related to the
highly active topic of {Sobolev training}
\cite{czarnecki2017sobolev,vlassis2021sobolev,vlassis2020geometric,srinivas2018knowledge,hill2026geometric,goswami2022physics}.
Sobolev training requires matching derivatives of the target operator by
derivatives of the neural network, and thus we need to control the complexity of
derivatives of neural networks. 

We emphasize that the extensions developed in this paper require several
technical ingredients beyond existing neural-operator approximation and
generalization results. First, in order to obtain approximation rates measured
in Sobolev norms and to explicitly incorporate the smoothness of the input and
output functions into the final rates, we need a different encoding strategy
from those used in \cite{liu2024neural,weihs2026multiple}. At the same time,
we would like to keep the encoding as pointwise sampling of the input data,
rather than using basis coefficients that must be computed in advance, as in
\cite{song2023approximation}. To achieve this, we use a pseudo-spectral
projection based on pointwise samples, which preserves the sampling-based nature
of the neural-operator architecture while allowing us to control Sobolev
discretization errors.
Second, measuring the approximation error in Sobolev norms requires neural
network approximations not only of functions but also of their derivatives.
This leads to additional technical difficulties compared with \(L^\infty\) or
\(L^2\) approximation. We handle this by using local polynomial approximation
and partition-of-unity constructions, inspired by classical finite element
arguments \cite{brenner2008mathematical,guhring2021approximation}, together with ReLU network
realizations that preserve the required Sobolev approximation rates.
Third, the generalization analysis also requires new arguments. Since the loss
involves derivatives of the neural operator, one needs to estimate the
complexity of derivative classes of neural networks. This cannot be obtained by
directly using covering numbers over the whole function space, as in
\cite{liu2024neural}, because derivatives of ReLU networks are not Lipschitz
continuous and the corresponding global covering numbers may be infinite.
Instead, we use covering numbers defined on finite data points and follow the
strategy of \cite{anthony1999neural,yang2025deeponet} to control the empirical Sobolev loss.
These ingredients together show that extending existing approximation and
generalization theory to the present multi-input Sobolev setting is highly
nontrivial.

\section{Preliminaries}
\subsection{Notations}\label{preliminaries}
		Let us summarize all basic notations used in the deep neural networks as follows:
		
		\textbf{1}. Define $\sigma(x)=\operatorname{ReLU}(x)=\max\{0,x\}$. With the abuse of notations, we define $\sigma:\sR^d\to\sR^d$ as $\sigma(\vx)=\left[\begin{array}{c}
			\sigma(x_1) \\
			\vdots \\
			 \sigma(x_d)
		\end{array}\right]$ for any $\vx=\left[x_1, \cdots, x_d\right]^T \in\sR^d$.

		\textbf{2}. Let \(d,L\in \mathbb{N}_+\), and let
\[
N_0=d,\qquad N_{L+1}=1,\qquad N_\ell\in\mathbb{N}_+,\quad \ell=1,\dots,L.
\]
A feedforward neural network with activation function \(\sigma\), depth \(L\), and layer widths \(\{N_\ell\}_{\ell=0}^{L+1}\) is defined by
\begin{align}
\tilde{\vh}_0=\vx,\qquad
\vh_\ell=\vW_\ell \tilde{\vh}_{\ell-1}+\vb_\ell,\quad
\tilde{\vh}_\ell=\sigma(\vh_\ell),\quad \ell=1,\dots,L,
\end{align}
and
\begin{align}
\phi(\vx)=\vh_{L+1}=\vW_{L+1}\tilde{\vh}_L+\vb_{L+1},
\label{eq.NN}
\end{align}
where \(\vW_\ell\in\mathbb{R}^{N_\ell\times N_{\ell-1}}\) and \(\vb_\ell\in\mathbb{R}^{N_\ell}\) are the weight matrix and bias vector at the \(\ell\)-th layer, respectively. The width of the network is defined as
\(
\max_{1\le \ell\le L} N_\ell.
\) Given \(d_{\mathrm{in}},d_{\mathrm{out}},L,p,K,\kappa,M>0\), we denote by
\[
\fF_{\mathrm{NN}}(d_{\mathrm{in}},d_{\mathrm{out}},L,p,K,\kappa,M)
\]
the class of vector-valued neural networks
\(\Gamma:\mathbb R^{d_{\mathrm{in}}}\to\mathbb R^{d_{\mathrm{out}}}\) whose
components are of the form \eqref{eq.NN}. Each network in this class has depth
\(L\), width at most \(p\), at most \(K\) nonzero parameters, and satisfies
\[
\|\Gamma\|_{L^\infty}\le M,\qquad
\|\vW_\ell\|_{\infty,\infty}\le \kappa,\qquad
\|\vb_\ell\|_\infty\le \kappa,
\quad \ell=1,\ldots,L+1.
\]
Here $\|\vW_\ell\|_{\infty,\infty}$ denotes the largest magnitude of all elements of $\vW_{\ell}$, \(\|\vW_\ell\|_0\) and \(\|\vb_\ell\|_0\) denote the numbers of nonzero entries of \(\vW_\ell\) and \(\vb_\ell\), respectively.

        \textbf{3}. Denote $\Omega$ as $[-1,1]^d$, $D$ as the weak derivative of a single variable function and $D^{\valpha}=D^{\alpha_1}_1D^{\alpha_2}_2\ldots D^{\alpha_d}_d$ as the partial derivative where $\valpha=[\alpha_{1},\alpha_{2},\ldots,\alpha_d]^{\top}$ and $D_i$ is the derivative in the $i$-th variable. Let $n\in\sN$ and $1\le p\le \infty$. We define Sobolev spaces\[W^{n, p}(\Omega):=\left\{f \in L^p(\Omega): D^{\valpha} f \in L^p(\Omega) \text { for all } \boldsymbol{\alpha} \in \sN^d \text { with }|\boldsymbol{\alpha}| \leq n\right\}\] with a norm $$\|f\|_{W^{n, p}(\Omega)}:=\left(\sum_{0 \leq|\alpha| \leq n}\left\|D^{\valpha} f\right\|_{L^p(\Omega)}^p\right)^{1 / p},$$ if $p<\infty$, and $$\|f\|_{W^{n, \infty}(\Omega)}:=\max_{0 \leq|\alpha| \leq n}\left\|D^{\valpha} f\right\|_{L^\infty(\Omega)}.$$

\subsection{Problem Setting and Assumptions}

In this section, we introduce the operator-learning problem studied in this paper, together with the assumptions on the function classes and the neural network architecture used for the approximation.

Let \(\lambda, n_i\in \mathbb{N}_+\), and set
\[
\Omega_i=[-1,1]^{d_i},\qquad i=1,2,\ldots,\lambda+1.
\]
We consider a \(\lambda\)-input operator defined on the product space
\[
\mathcal{G}:\prod_{i=1}^{\lambda}W^{n_i,\infty}(\Omega_i)
\longrightarrow W^{n_{\lambda+1},\infty}(\Omega_{\lambda+1}).
\]
Although the operator is defined on the larger space
\(\prod_{i=1}^{\lambda}W^{n_i,\infty}(\Omega_i)\), our error analysis will be carried out on more regular input classes
\(
\mathcal{X}_i\subset W^{n_i,\infty}(\Omega_i).\)
This restriction allows us to incorporate the regularity of the input and output functions into the error rates and to work with compact subsets of the ambient function spaces. At the same time, it is important to keep \(\mathcal{G}\) defined on the larger product space
\(\prod_{i=1}^{\lambda}L^\infty(\Omega_i)\), since in the proof we will reconstruct certain auxiliary functions that may not belong to the compact regularity classes \(\mathcal{X}_i\). Equivalently, the approximation problem studied in this paper concerns the restriction
\[
\mathcal{G}:\prod_{i=1}^{\lambda}\mathcal{X}_i
\longrightarrow W^{n_{\lambda+1},\infty}(\Omega_{\lambda+1}).
\]

We next impose regularity and Lipschitz assumptions on the operator \(\mathcal{G}\).

\begin{assumption}\label{asspde}
Let \(k_i,\alpha_i\in\sN\) for \(i=1,2,\ldots,\lambda\) and $\ell=\{0,1\}$.
Assume that
\(
\mathcal{G}\)
is separately Lipschitz continuous with respect to the Sobolev norms
\(W^{k_i,\infty}(\Omega_i)\) in the following sense. There exist constants
\(L_i>0\), \(i=1,\ldots,\lambda\), such that for any \(1\le i_*\le \lambda\),
any \(g_{i_*},h_{i_*}\in W^{k_{i_*},\infty}(\Omega_{i_*})\), and any
\(f_i\in W^{k_i,\infty}(\Omega_i)\), \(i\neq i_*\), we have
\begin{align}
&\Bigl\|
\mathcal{G}(f_1,\ldots,f_{i_*-1},g_{i_*},f_{i_*+1},\ldots,f_{\lambda})
-
\mathcal{G}(f_1,\ldots,f_{i_*-1},h_{i_*},f_{i_*+1},\ldots,f_{\lambda})
\Bigr\|_{W^{\ell,\infty}(\Omega_{\lambda+1})}
\nonumber\\
\le&
L_{i_*}
\|g_{i_*}-h_{i_*}\|_{W^{k_{i_*},\infty}(\Omega_{i_*})}
\left(
\prod_{\substack{i=1\\ i\neq i_*}}^{\lambda}
\left(1+\|f_i\|_{W^{k_i,\infty}(\Omega_i)}\right)
\right)^{\alpha_{i_*}} .\notag
\end{align}
\end{assumption}

\begin{remark}
Assumption~\ref{asspde} can be strengthened by replacing the
\(W^{\ell,\infty}(\Omega_{\lambda+1})\)-norm on the left-hand side with a higher order Sobolev norm,
for example \(W^{s,\infty}(\Omega_{\lambda+1})\) for $s\ge 2$. However, this requires the neural network architecture to have sufficient
smoothness with respect to the output variable. In particular, standard ReLU
networks are not suitable for directly approximating output derivatives when
\(\ell>1\), since ReLU is not sufficiently smooth. One would need to replace
ReLU in the trunk networks by smoother activation functions, such as
\(\mathrm{ReLU}^r\), or by other sufficiently smooth activations, as in
\cite{yang2025deeponet}. 
\end{remark}

\begin{assumption}\label{assump:X-bound}
There exists a constant \(S,U>0\) such that, for every \(i=1,\ldots,\lambda\),
\[
\|f_i\|_{L^{\infty}(\Omega_i)}\le S, ~\|f_i\|_{W^{n_i,\infty}(\Omega_i)}\le U
\qquad \forall\, f_i\in\mathcal{X}_i .
\]
\end{assumption}

\begin{assumption}\label{assump:G-output-regularity}
Let \(\beta_i\in\sN\), \(i=1,\ldots,\lambda\). Assume that there exists a constant
\(C>0\) such that, for any \(f_i\in W^{\beta_i,\infty}(\Omega_i)\),
\(i=1,\ldots,\lambda\), we have
\[
\left\|
\mathcal{G}(f_1,\ldots,f_\lambda)
\right\|_{W^{n_{\lambda+1},\infty}(\Omega_{\lambda+1})}
\le
C\prod_{i=1}^{\lambda}
\left(
1+
\|f_i\|_{W^{\beta_i,\infty}(\Omega_i)}
\right).
\]
\end{assumption}

The above assumptions are natural in the PDE setting. For example, consider the
linear elliptic Dirichlet problem
\begin{equation}
\begin{cases}
\fL u=f & \text{in } \Omega,\\
u=g & \text{on } \partial\Omega .
\end{cases}\notag
\end{equation}
We view the solution operator as the map
\[
\fG:(f,g)\mapsto u .
\]
Under the standard Calderón--Zygmund assumptions on \(\fL\) and \(\Omega\) as shown in \cite{dong2022w,gilbarg1998elliptic}, the
global \(W^{2,p}\)-estimate gives, for any \(p\in(1,\infty)\),
\[
\|u\|_{W^{2,p}(\Omega)}
=
\|\fG(f,g)\|_{W^{2,p}(\Omega)}
\le
C\left(
\|f\|_{L^p(\Omega)}
+
\|g\|_{W^{2-1/p,p}(\partial\Omega)}
\right).
\]
Taking \(p>d\) and using the Sobolev embedding
\(W^{2,p}(\Omega)\hookrightarrow W^{1,\infty}(\Omega)\), we obtain
\[
\|\fG(f,g)\|_{W^{1,\infty}(\Omega)}
\le
C\left(
\|f\|_{L^\infty(\Omega)}
+
\|g\|_{W^{2-1/p,\infty}(\partial\Omega)}
\right).
\]
Thus, in Assumption~\ref{assump:G-output-regularity}, one may take
\(n_{\lambda+1}=1\), \(\beta_1=0\), and \(\beta_2=2-1/p\).

Moreover, by linearity of the PDE, the solution operator is separately Lipschitz.
Indeed,
\[
\|\fG(f_1,g)-\fG(f_2,g)\|_{W^{1,\infty}(\Omega)}
=
\|\fG(f_1-f_2,0)\|_{W^{1,\infty}(\Omega)}
\le
C\|f_1-f_2\|_{L^\infty(\Omega)},
\]
and
\[
\|\fG(f,g_1)-\fG(f,g_2)\|_{W^{1,\infty}(\Omega)}
=
\|\fG(0,g_1-g_2)\|_{W^{1,\infty}(\Omega)}
\le
C\|g_1-g_2\|_{W^{2-1/p,\infty}(\partial\Omega)}.
\]
Therefore, in this linear case, the exponents \(\alpha_i\) in
Assumption~\ref{asspde} can be chosen to be zero.

In contrast, if the input includes coefficients of the PDE, then the corresponding
Lipschitz bound generally depends on the size of the other inputs. For instance,
consider
\[
-\Delta u+a(x)u=f,\qquad u|_{\partial\Omega}=0,
\]
and regard the solution operator as \(\fG(a,f)=u\). If \(a\ge 0\), then for fixed
\(a\) and two right-hand sides \(f_1,f_2\), the difference
\(w=\fG(a,f_1)-\fG(a,f_2)\) satisfies
\[
-\Delta w+a(x)w=f_1-f_2,\qquad w|_{\partial\Omega}=0.
\]
For any \(p>d\), the elliptic \(W^{2,p}\)-estimate and the Sobolev embedding imply
\[
\|\fG(a,f_1)-\fG(a,f_2)\|_{W^{1,\infty}(\Omega)}
\le
C_{\Omega,p}
\left(
1+\|a\|_{L^\infty(\Omega)}
\right)
\|f_1-f_2\|_{L^\infty(\Omega)}.
\]
Hence, when the other input is the coefficient \(a\), the Lipschitz constant grows
at most linearly in \(\|a\|_{L^\infty(\Omega)}\). This corresponds to taking
\(\alpha_i=1\) in Assumption~\ref{asspde}.
\subsection{Neural network architecture}

Our goal is to construct a parametric neural operator
\[
\mathcal{G}_{\boldsymbol\theta}:
\prod_{i=1}^{\lambda}\mathcal{X}_i
\longrightarrow W^{n_{\lambda+1},\infty}(\Omega_{\lambda+1}).
\]
We first describe the two-input case \(\lambda=2\), which is the main case
used in the detailed construction below. The architecture is similar in spirit
to the multiple-input operator architecture considered in
\cite{weihs2026multiple}.

In view of the discretization procedure introduced later, for two inputs we
consider neural operators of the form
\begin{align}
\mathcal{G}_{\boldsymbol\theta}(f,g)(\vx)
=
\sum_{s=1}^J\sum_{h=1}^H\sum_{p=1}^P
e_{h,p,s}\,
\fE(\fD_{\bar m}f;\vtheta_{0,h})\,
\fB(\fD_m g;\vtheta_{1,s})\,
\fT(\vx;\vtheta_{2,p}) .
\label{eq:G}
\end{align}
Here \(\fE\) and \(\fB\) are branch subnetworks corresponding to the
discretized inputs \(f\) and \(g\), respectively, while \(\fT\) is the trunk
subnetwork associated with the spatial variable \(\vx\). The coefficients
\(e_{h,p,s}\in\mathbb R\) are trainable scalar coefficients and are independent
of the particular input pair \((f,g)\).

Let
\[
\fF_E
=
\fF_{\rm NN}
(d_{E,{\rm in}},d_{E,{\rm out}},L_E,p_E,K_E,\kappa_E,M_E),
\]
\[
\fF_B
=
\fF_{\rm NN}
(d_{B,{\rm in}},d_{B,{\rm out}},L_B,p_B,K_B,\kappa_B,M_B),
\]
and
\[
\fF_T
=
\fF_{\rm NN}
(d_{T,{\rm in}},d_{T,{\rm out}},L_T,p_T,K_T,\kappa_T,M_T)
\]
be the prescribed architecture classes for the two branch subnetworks and the
trunk subnetwork.

We define the corresponding neural operator class by
\begin{align}
\fF_G
=
\Bigl\{
\mathcal G_{\boldsymbol\theta}\;:\;
&\ \mathcal G_{\boldsymbol\theta} \text{ is of the form \eqref{eq:G}},\
\fE(\cdot;\vtheta_{0,h})\in\fF_E,\
\fB(\cdot;\vtheta_{1,s})\in\fF_B,\
\fT(\cdot;\vtheta_{2,p})\in\fF_T,
\max_{h,p,s}|e_{h,p,s}|\le \kappa_e
\Bigr\}.
\label{eq:GNN}
\end{align}
The precise choices of the architecture parameters
\[
L_E,p_E,K_E,\kappa_E,M_E,\qquad
L_B,p_B,K_B,\kappa_B,M_B,\qquad
L_T,p_T,K_T,\kappa_T,M_T,
\]
as well as the coefficient bound \(\kappa_e\), will be specified in the main
theorems.

For a general number of inputs \(\lambda\ge2\), the same construction extends
naturally to the shared branch--trunk form
\[
\mathcal G_{\boldsymbol\theta}(f_1,\ldots,f_\lambda)(\vx)
=
\sum_{s_1=1}^{J_1}\cdots
\sum_{s_\lambda=1}^{J_\lambda}
\sum_{p=1}^{P}
e_{s_1,\ldots,s_\lambda,p}
\prod_{i=1}^{\lambda}
\mathcal B_{i,s_i}(\mathcal D_{m_i}f_i)
\,
\mathcal T_p(\vx).
\]
Here \(\mathcal B_{i,s_i}\) is the branch subnetwork associated with the
discretized \(i\)-th input \(\mathcal D_{m_i}f_i\), and \(\mathcal T_p\) is the
shared trunk subnetwork associated with the output variable \(\vx\).

\section{Approximation Error}
\label{sec:approx}

In this section, we establish the approximation error for multi-input neural
operators. To make the presentation clear, we first treat the two-input case
\(\lambda=2\) with \(d_1=d_2=d_3=d\), that is,
\(\Omega_1=\Omega_2=\Omega_3=\Omega\). The extension to the general
\(\lambda\)-input case is given afterwards.

\subsection{Approximation error for \(\lambda=2\)}

When \(\lambda=2\), we use the neural operator architecture introduced in
\eqref{eq:GNN}. The approximation error is measured by
\[
\inf_{\mathcal G_{\vtheta}\in\fF_G}
\sup_{f\in\fX_1,\;g\in\fX_2}
\left\|
\mathcal G(f,g)-\mathcal G_{\vtheta}(f,g)
\right\|_{W^{\ell,\infty}(\Omega)} 
\]
for $\ell\in\{0,1\}$. Since the architecture depends on discretized input functions, we first specify
the discretization operator. In this paper, we use tensorized Chebyshev grids.
For \(N\in\mathbb N\), define
\[
\vx_{k_1,\dots,k_d}
=
\left(
\cos\left(\frac{(k_1+\frac12)\pi}{N+1}\right),
\dots,
\cos\left(\frac{(k_d+\frac12)\pi}{N+1}\right)
\right),
\qquad
0\le k_1,\dots,k_d\le N .
\]
The total number of grid points is \((N+1)^d\). We define the sampling operator
by
\[
\mathcal D_N(g)
=
\bigl(g(\vx_{k_1,\dots,k_d})\bigr)_{0\le k_1,\dots,k_d\le N}
\in \mathbb R^{(N+1)^d}.
\]

We now state the main approximation theorem for the two-input case with
\(d_1=d_2=d_3=d\), see a proof in Section \ref{proof:thm:main-final-functionL}.

\begin{theorem}
\label{thm:main-final-functionL}
Let \(S,U,L_i>0\), \(d,\bar m,m,k_i,\alpha_i,\beta_i\in\mathbb N\) for $i=1,2$, and
\(\ell\in\{0,1\}\).  Set
\[
\bar m_*=(\bar m+1)^d,
\qquad
m_*=(m+1)^d.
\]
Define
\[
A_{m,\bar m}^{(g)}
:=
 L_1 m^{2k_2}(\log m)^d
\left(
1+S\bar m^{2k_1}(\log \bar m)^d
\right)^{\alpha_2},
\]
\[
A_{\bar m,m}^{(f)}
:=
 L_2 \bar m^{2k_1}(\log \bar m)^d
\left(
1+S m^{2k_2}(\log m)^d
\right)^{\alpha_1},
\]
and
\begin{align}
B_{m,\bar m}
:=
\Bigl[
1
&+
L_1S m^{2k_2}(\log m)^d
+
L_2S\bar m^{2k_1}(\log \bar m)^d
\left(
1+S m^{2k_2}(\log m)^d
\right)^{\alpha_1}
\Bigr].
\end{align}
Furthermore, set
\[
R_{\bar m,m}^{(n_3)}
:=
\left(
1+S\bar m^{2\beta_1}(\log \bar m)^d
\right)
\left(
1+S m^{2\beta_2}(\log m)^d
\right).
\]

For any \(J,H,P\in\mathbb N\), choose branch networks
\[
\mathcal E_{h}\in
\mathcal F_{\mathrm{NN}}(\bar m_*,1,L_E,p_E,K_E,\kappa_E,1),
\qquad
h=1,\ldots,H,
\]
\[
\mathcal B_s\in
\mathcal F_{\mathrm{NN}}(m_*,1,L_B,p_B,K_B,\kappa_B,1),
\qquad
s=1,\ldots,J,
\]
and trunk networks
\[
\mathcal T_{p}\in
\mathcal F_{\mathrm{NN}}(d,1,L_T,p_T,K_T,\kappa_T,M_T),
\qquad
p=1,\ldots,P,
\]
and coefficients \(e_{h,p,s}\in\mathbb R\).
Set the branch network architecture for \(\mathcal B_s\)
as
\begin{align*}
&L_B=
\mathcal O\!\left(m_*^2\log m_*+m_*\log J\right),
\qquad p_B=\mathcal O(1),\\ 
&K_B=
\mathcal O\!\left(m_*^2\log m_*+m_*\log J\right), \qquad    \kappa_B
=
\mathcal O\!\left(
\sqrt{m_*}\,B_{m,\bar m}
J^{1+\frac1{m_*}}
\right).
\end{align*}
Set the branch network architecture for \(\mathcal E_h\) as,
\begin{align*}
&L_E=
\mathcal O\!\left(\bar m_*^2\log \bar m_*+\bar m_*\log H\right),
\qquad
p_E=\mathcal O(1)\\
&K_E=
\mathcal O\!\left(\bar m_*^2\log \bar m_*+\bar m_*\log H\right), \qquad
\kappa_E
=
\mathcal O\!\left(
\sqrt{\bar m_*}\,B_{m,\bar m}
H^{1+\frac1{\bar m_*}}
\right).
\end{align*}
Set the trunk network architecture for \(\mathcal T_p\) as
\[
L_T,K_T=\mathcal O(\log P),\qquad
p_T,M_T=\mathcal O(1),\qquad
\kappa_T=\mathcal O\!\left(P^{(n_3+d-\ell)/d}\right).
\]
Moreover, set the coefficients bound as
\[
|e_{h,p,s}|
\le
\mathcal O (R_{\bar m,m}^{(n_3)}).
\]
For any two-input operator \(\mathcal G\) satisfying
Assumptions~\ref{asspde}, \ref{assump:X-bound} and
\ref{assump:G-output-regularity}, there exist branch and trunk networks with architectures specified above such that for any 
\((f,g)\in\mathcal X_1\times\mathcal X_2\), we have
\begin{align}
&\left\|
\mathcal G(f,g)
-
\sum_{s=1}^{J}\sum_{h=1}^{H}\sum_{p=1}^{P}
e_{h,p,s}\,
\mathcal E_{h}(\mathcal D_{\bar m}f)\,
\mathcal B_s(\mathcal D_m g)\,
\mathcal T_{p}
\right\|_{W^{\ell,\infty}(\Omega)}
\notag\\
&\qquad\le
C\Bigl[
L_2\bar m^{-n_1}(\log \bar m)^d
\left(
1+S m^{2k_2}(\log m)^d
\right)^{\alpha_1}
+
L_1m^{-n_2}(\log m)^d
\notag\\
&\qquad\qquad
+
S A_{m,\bar m}^{(g)}\sqrt{m_*}\,J^{-1/m_*}
+
S A_{\bar m,m}^{(f)}\sqrt{\bar m_*}\,H^{-1/\bar m_*}
+
R_{\bar m,m}^{(n_3)}P^{-\frac{n_3-\ell}{d}}
\Bigr].
\label{eq:main-final-error-functionL}
\end{align}

The resulting neural operator satisfies the output bound
\begin{align}
\sup_{(f,g)\in\mathcal X_1\times\mathcal X_2}
\sup_{\vx\in\Omega}
\sum_{|\gamma|\le \ell}
\left|
\partial^\gamma
\left[
\sum_{s=1}^{J}\sum_{h=1}^{H}\sum_{p=1}^{P}
e_{h,p,s}\,
\mathcal E_h(\mathcal D_{\bar m}f)\,
\mathcal B_s(\mathcal D_m g)\,
\mathcal T_p(\vx)
\right]
\right|
\le
C R_{\bar m,m}^{(n_3)}P^{\ell/d}.
\label{eq:NN-output-bound-Well}
\end{align}
Here \(C>0\) is independent of \(\bar m,m,J,H,P,f\), and \(g\), but may depend
on \(d,n_i,k_i,\alpha_i,\beta_i,\ell\), the constants in the assumptions, and
the uniform Sobolev bounds of the input classes.
\end{theorem}

Based on Theorem~\ref{thm:main-final-functionL}, we now state the final
balanced approximation result. It is obtained by choosing the discretization
levels and the numbers of branch and trunk basis functions so that the four
error contributions are of the same order.

\begin{corollary}[Balanced rate in terms of the total number of parameters]
\label{cor:balanced-main-final-functionL}
Suppose that the two-input operator \(\mathcal G\) satisfies
Assumptions~\ref{asspde}, \ref{assump:X-bound}, and
\ref{assump:G-output-regularity} with \(n_3>\ell\), \(d_i=d\), \(k_i=0\),
\(\alpha_i=0\), and \(\beta_i=0\). Then there exists a neural operator
\(\mathcal G_{\vtheta}\) of the form \eqref{eq:G} such that, for every
\((f,g)\in\mathcal X_1\times\mathcal X_2\),
\begin{align}
\left\|
\mathcal G(f,g)
-
\mathcal G_{\vtheta}(f,g)
\right\|_{W^{\ell,\infty}(\Omega)}
\le
C
\left(
\frac{\log N_{\mathrm{tot}}}{\log\log N_{\mathrm{tot}}}
\right)^{-\frac{\min\{n_1,n_2\}}{d}}
(\log\log N_{\mathrm{tot}})^{2d}.
\end{align}
Here \(N_{\mathrm{tot}}\) denotes the total number of parameters of
\(\mathcal G_{\vtheta}\), and \(C>0\) is independent of \(N_{\mathrm{tot}}\).
\end{corollary}
Corollary \ref{cor:balanced-main-final-functionL} is proved in Appendix \ref{proof:corollary}.

\subsection{Approximation error for \(\lambda\ge 2\)}

We now extend the two-input approximation result to the general
\(\lambda\)-input case. The construction follows the same strategy as in the
case \(\lambda=2\): we first discretize each input function, then approximate
the dependence on the resulting finite-dimensional variables by branch
networks, and finally approximate the remaining spatial coefficient functions
by trunk networks. Since the argument is essentially iterative, we only state
the result here and defer the proof to the appendix.

\begin{corollary}[Balanced rate for the \(\lambda\)-input case]
\label{cor:more-balanced-main-final-functionL}
Suppose that the \(\lambda\)-input operator \(\mathcal G\) satisfies
Assumptions~\ref{asspde}, \ref{assump:X-bound}, and
\ref{assump:G-output-regularity}. Assume also that
\(n_i>2k_i\) for \(i=1,\ldots,\lambda\), and $n_{\lambda+1}>\ell$. For each input, let
\[
M_i:=(m_i+1)^{d_i}
\]
denote the number of sampling points used in the discretization
\(\mathcal D_{m_i}f_i\).
Then there exists a neural operator \(\mathcal G_{\boldsymbol\theta}\) of the
shared branch--trunk form
\begin{equation}
\begin{aligned}
\mathcal G_{\boldsymbol\theta}(f_1,\ldots,f_\lambda)(\vx)
:=
\sum_{s_1=1}^{J_1}\cdots
\sum_{s_\lambda=1}^{J_\lambda}
\sum_{p=1}^{P}
e_{s_1,\ldots,s_\lambda,p}
\prod_{i=1}^{\lambda}
\mathcal B_{i,s_i}(\mathcal D_{m_i}f_i)
\,
\mathcal T_p(\vx),
\end{aligned}
\label{eq:lambda-shared-architecture}
\end{equation}
such that, for every
\((f_1,\ldots,f_\lambda)\in\prod_{i=1}^{\lambda}\mathcal X_i\),
\begin{align}
&\left\|
\mathcal G(f_1,\ldots,f_\lambda)
-
\mathcal G_{\boldsymbol\theta}(f_1,\ldots,f_\lambda)
\right\|_{W^{\ell,\infty}(\Omega_{\lambda+1})}
\notag\\
&\quad\lesssim
\sum_{i=1}^{\lambda}
M_i^{-\frac{n_i-2k_i}{d_i}}(\log M_i)^{d_i}
\left[
\prod_{\substack{j=1\\j\neq i}}^{\lambda}
\left(
1+S M_j^{\frac{2k_j}{d_j}}(\log M_j)^{d_j}
\right)
\right]^{\alpha_i}
\notag\\
&\qquad+
\sum_{i=1}^{\lambda}
S L_i
\sqrt{M_i}\,
M_i^{\frac{2k_i}{d_i}}
(\log M_i)^{d_i}
J_i^{-1/M_i}
\left[
\prod_{\substack{j=1\\j\neq i}}^{\lambda}
\left(
1+S M_j^{\frac{2k_j}{d_j}}(\log M_j)^{d_j}
\right)
\right]^{\alpha_i}
\notag\\
&\qquad+
\left[
\prod_{i=1}^{\lambda}
\left(
1+S M_i^{\frac{2\beta_i}{d_i}}(\log M_i)^{d_i}
\right)
\right]
P^{-\frac{n_{\lambda+1}-\ell}{d_{\lambda+1}}}.
\label{eq:most-general-shared}
\end{align}
Here the implicit constant is independent of \(M_i,J_i,P\).

In general, the balancing of \eqref{eq:most-general-shared} depends on the
interaction among the parameters \(\alpha_i,\beta_i,k_i,n_i,d_i\). In the
special case
\(
\alpha_i=\beta_i=0,~
i=1,\ldots,\lambda,
\)
the dominant approximation rate can be balanced explicitly. Let
\(N_{\mathrm{tot}}\) denote the total number of trainable parameters of
\(\mathcal G_{\boldsymbol\theta}\). Define
\[
Q_{\max}
:=
\max_{1\le i\le \lambda}
\frac{d_i}{n_i-2k_i}.
\]
Then, up to higher logarithmic factors,
\[
\sup_{(f_1,\ldots,f_\lambda)\in\prod_{i=1}^{\lambda}\mathcal X_i}
\left\|
\mathcal G(f_1,\ldots,f_\lambda)
-
\mathcal G_{\boldsymbol\theta}(f_1,\ldots,f_\lambda)
\right\|_{W^{\ell,\infty}(\Omega_{\lambda+1})}
\lesssim
\left(
\frac{\log N_{\mathrm{tot}}}{\log\log N_{\mathrm{tot}}}
\right)^{-1/Q_{\max}}
(\log\log N_{\mathrm{tot}})^{\sum_{i=1}^{\lambda}d_i}.
\]
\end{corollary}
Corollay \ref{cor:more-balanced-main-final-functionL} is proved in Appendix \ref{proof:corollary}.
\begin{remark}
\label{rmk:bla-dis}
In the case \(\alpha_i=\beta_i=0\), under the shared branch--trunk
architecture with partition-of-unity-type branch bounds, the balanced exponent
is determined by the most difficult input direction. More precisely, define
\[
Q_{\max}
:=
\max_{1\le i\le\lambda}
\frac{d_i}{n_i-2k_i}.
\]
Then the dominant approximation rate is of the form
\[
\left(
\frac{\log N_{\rm tot}}{\log\log N_{\rm tot}}
\right)^{-1/Q_{\max}}
(\log\log N_{\rm tot})^{\sum_{i=1}^{\lambda}d_i}.
\]

This formula shows how each input affects the final rate. If an additional
input has high regularity, meaning that \(n_i\) is large relative to \(d_i\) and
\(k_i\), then the contribution
\[
\frac{d_i}{n_i-2k_i}
\]
is small. If this contribution is smaller than the current maximum
\(Q_{\max}\), then adding this input does not change the leading exponent of the
rate. It may only affect constants and logarithmic factors. On the other hand,
if the new input has large dimension \(d_i\), low smoothness \(n_i\), or a large
Sobolev Lipschitz order \(k_i\), then \(n_i-2k_i\) becomes small and
\(
\frac{d_i}{n_i-2k_i}
\)
may become the new maximum. In that case, this input becomes the bottleneck and
deteriorates the final rate.

The role of \(k_i\) is different from that of \(n_i\). The parameter \(n_i\)
measures the intrinsic smoothness of the \(i\)-th input class
\(\mathcal X_i\), while \(k_i\) measures the Sobolev order in which the operator
\(\mathcal G\) is Lipschitz with respect to the \(i\)-th input. In the
discretization step, the reconstruction error in \(W^{k_i,\infty}\) scales as
\[
M_i^{-\frac{n_i-2k_i}{d_i}},
\]
because the Chebyshev reconstruction has a derivative-stability cost of order
\(M_i^{2k_i/d_i}\). Therefore, larger \(k_i\) reduces the effective smoothness
from \(n_i\) to \(n_i-2k_i\), making the approximation problem harder.

The proof also gives a concrete guideline for choosing the discretization levels
and branch ranks. To achieve accuracy \(\varepsilon\), the balanced choice is,
up to logarithmic factors,
\[
M_i
\asymp
\varepsilon^{-\frac{d_i}{n_i-2k_i}}.
\]
The corresponding branch ranks are chosen independently in each input direction
according to
\[
\log J_i
\asymp
M_i\log M_i,
\qquad
i=1,\ldots,\lambda,
\]
up to lower-order logarithmic factors. Thus, harder input spaces require finer
discretization and larger branch subnetworks.

\end{remark}

\section{Generalization Error}\label{sec:gen}

Let \(\mu_i\) be a probability measure on \(\mathcal X_i\),
\(i=1,\ldots,\lambda\), and let \(\mu_z\) be a probability measure on
\(\Omega_{\lambda+1}\). We consider the following hierarchical Sobolev training
setting.

\begin{setting}[\(\lambda\)-input hierarchical Sobolev training setting]
\label{set.1}
Let
\[
\mathcal G:\prod_{r=1}^{\lambda}\mathcal X_r
\to W^{\ell,\infty}(\Omega_{\lambda+1})
\]
be a \(\lambda\)-input operator, where \(\ell\in\{0,1\}\). Suppose that the
training data are generated hierarchically as follows. First, sample
\[
\{f^{(1)}_{i_1}\}_{i_1=1}^{n_1^{\rm samp}}
\subset \mathcal X_1 .
\]
For each \(f^{(1)}_{i_1}\), sample
\[
\{f^{(2)}_{i_1,i_2}\}_{i_2=1}^{n_2^{\rm samp}}
\subset \mathcal X_2 .
\]
Proceeding recursively, for each
\((f^{(1)}_{i_1},\ldots,f^{(r-1)}_{i_1,\ldots,i_{r-1}})\), sample
\[
\{f^{(r)}_{i_1,\ldots,i_r}\}_{i_r=1}^{n_r^{\rm samp}}
\subset \mathcal X_r,
\qquad r=2,\ldots,\lambda .
\]
Finally, for each sampled \(\lambda\)-tuple, sample output locations
\[
\{\vz_{i_1,\ldots,i_\lambda,k}\}_{k=1}^{n_z}
\subset \Omega_{\lambda+1}.
\]

For every multi-index \(\gamma\) with \(|\gamma|\le\ell\), the observed
Sobolev training labels are
\[
y_{i_1,\ldots,i_\lambda,k}^{(\gamma)}
=
\partial^\gamma
\mathcal G
\left(
f^{(1)}_{i_1},
f^{(2)}_{i_1,i_2},
\ldots,
f^{(\lambda)}_{i_1,\ldots,i_\lambda}
\right)
(\vz_{i_1,\ldots,i_\lambda,k})
+
\varepsilon_{i_1,\ldots,i_\lambda,k}^{(\gamma)} .
\]
We assume that, for each \(|\gamma|\le\ell\), the noise variables
\(\varepsilon_{i_1,\ldots,i_\lambda,k}^{(\gamma)}\) are independent,
mean-zero, and sub-Gaussian with variance proxy bounded by \(\sigma^2\).

For a neural operator class \(\mathcal F_G\), define the Sobolev empirical risk
by
\[
\mathcal L_{\mathcal S}(\mathcal G_{\rm NN})
:=
\frac{1}{N_{\rm train}}
\sum_{i_1=1}^{n_1^{\rm samp}}
\cdots
\sum_{i_\lambda=1}^{n_\lambda^{\rm samp}}
\sum_{k=1}^{n_z}
\sum_{|\gamma|\le\ell}
\left|
\partial^\gamma
\mathcal G_{\rm NN}
\left(
\mathcal D_{m_1}f^{(1)}_{i_1},
\ldots,
\mathcal D_{m_\lambda}f^{(\lambda)}_{i_1,\ldots,i_\lambda}
\right)
(\vz_{i_1,\ldots,i_\lambda,k})
-
y_{i_1,\ldots,i_\lambda,k}^{(\gamma)}
\right|^2 .
\]
The empirical risk minimizer is defined by
\[
\widehat{\mathcal G}_{\mathcal S}
\in
\argmin_{\mathcal G_{\rm NN}\in\mathcal F_G}
\mathcal L_{\mathcal S}(\mathcal G_{\rm NN}) .
\]

We assume that the total number of inner samples and output locations does not
grow too fast compared with the number of outer samples. More precisely, for
some sufficiently small \(\delta>0\),
\[
\log\left(
n_z\prod_{r=2}^{\lambda}n_r^{\rm samp}
\right)
\le
\left(n_1^{\rm samp}\right)^\delta .
\]
\end{setting}

{Setting~\ref{set.1} is motivated by common sampling structures in scientific
computing. For example, consider
\[
-\Delta u+a(x)u=f,\qquad u|_{\partial\Omega}=0,
\]
with solution operator \(\mathcal G(a,f)=u\). The coefficient \(a(\vx)\) may
represent a fixed physical system, material property, or medium parameter.
For each fixed \(a\), one may observe several solutions corresponding to
different source terms \(f\). Thus the data are naturally hierarchical: one
first samples coefficients \(a\), and then, for each coefficient, samples
several source terms \(f\) and their corresponding solution observations. In
this case, \(a\) is the outer input with sample size \(n_1^{\rm samp}\), while
\(f\) is the inner input with sample size \(n_2^{\rm samp}\). }

\begin{theorem}[Sobolev generalization error for the two-input case]
\label{thm:gene-sobolev}
Consider Setting~\ref{set.1} with \(\lambda=2\) and \(\ell\in\{0,1\}\).
For simplicity, write
\[
n_f:=n_1^{\rm samp},
\qquad
n_g:=n_2^{\rm samp}.
\]
Suppose that the two-input operator
\(\mathcal G:\mathcal X\times\mathcal Y\to W^{\ell,\infty}(\Omega)\)
satisfies the assumptions of
Corollary~\ref{cor:balanced-main-final-functionL} in the
\(W^{\ell,\infty}\)-output setting. Let \(\mathcal F_G\) be the neural operator
class consisting of functions of the form
\[
\mathcal G_{\rm NN}(f,g)(\vz)
=
\sum_{s=1}^{J}
\sum_{h=1}^{H}
\sum_{p=1}^{P}
e_{h,p,s}\,
\mathcal E_h(\mathcal D_{\bar m}f)\,
\mathcal B_s(\mathcal D_m g)\,
\mathcal T_p(\vz).
\]
Let \(\widehat{\mathcal G}_{\mathcal S}\in\mathcal F_G\) be the Sobolev
empirical risk minimizer defined in Setting~\ref{set.1}.

Choose {$n_g,n_z$ so that $\log(n_gn_z)\le n_f^\delta$, and \(H=n_f^c\) with $c<1-\delta$}. Set
\(\Lambda_f:=\log n_f/\log\log n_f\). The discretization levels and branch
ranks can be chosen such that
\[
m_*\asymp \Lambda_f^{d/n_2},
\qquad
\bar m_*\asymp \Lambda_f^{d/n_1},
\qquad
\log J=\mathcal O\!\left(\Lambda_f^{d/n_2}\log\log n_f\right).
\]
The trunk rank satisfies \(P=(\log n_f)^{C_P+o(1)}\) for some constant
\(C_P>0\). In particular, \(P=n_f^{o(1)}\) and
\(\log P=\mathcal O(\log\log n_f)\).
The \(g\)-branch networks \(\mathcal B_s\) may be chosen with
\(p_B=\mathcal O(1)\),
\[
L_B,K_B
=
\mathcal O\!\left(
\Lambda_f^{2d/n_2}\log\log n_f
\right),
\qquad
\log \kappa_B
=
\mathcal O\!\left(
\Lambda_f^{d/n_2}\log\log n_f
\right).
\]
The \(f\)-branch networks \(\mathcal E_h\) may be chosen with
\(p_E=\mathcal O(1)\),
\[
L_E,K_E
=
\mathcal O\!\left(
\Lambda_f^{2d/n_1}\log\log n_f
+
\Lambda_f^{d/n_1}\log n_f
\right),
\qquad
\log \kappa_E
=
\mathcal O(\log n_f).
\]
The trunk networks \(\mathcal T_p\) may be chosen with
\[
L_T,K_T=\mathcal O(\log\log n_f),
\qquad
p_T=\mathcal O(1),
\qquad
\log\kappa_T=\mathcal O(\log\log n_f),
\qquad
M_T\le C .
\]

Then the Sobolev empirical risk minimizer satisfies
\begin{align}
&\mathbb E_{\mathcal S}
\mathbb E_{f\sim\mu_f}
\mathbb E_{g\sim\mu_g}
\mathbb E_{\vz\sim\mu_z}
\sum_{|\gamma|\le \ell}
\left|
\partial_{\vz}^{\gamma}
\mathcal G(f,g)(\vz)
-
\partial_{\vz}^{\gamma}
\widehat{\mathcal G}_{\mathcal S}
(\mathcal D_{\bar m}f,\mathcal D_m g)(\vz)
\right|^2\le
C
\left(
\frac{\log n_f}{\log\log n_f}
\right)^{-\frac{2}{\max\left\{\frac d{n_1},\frac d{n_2}\right\}}}
(\log\log n_f)^{4d}.
\notag
\end{align}
Here \(C>0\) is independent of \(n_f,n_g,n_z\).
\end{theorem}
Theorem \ref{thm:gene-sobolev} is proved in Section \ref{proof:thm:gene-sobolev}.
In Setting~\ref{set.1}, the
generalization error achieves the same convergence rate as the approximation
error. Combining this observation with
Corollary~\ref{cor:more-balanced-main-final-functionL}, we obtain the following
generalization error estimate.
\begin{corollary}[Sobolev generalization error for the \(\lambda\)-input case]
\label{cor:gene-lambda}
Consider Setting~\ref{set.1} for a general \(\lambda\ge2\) and
\(\ell\in\{0,1\}\). Suppose that the \(\lambda\)-input operator
\(\mathcal G:\prod_{i=1}^{\lambda}\mathcal X_i
\to W^{\ell,\infty}(\Omega_{\lambda+1})
\)
satisfies the same assumptions as in
Corollary~\ref{cor:more-balanced-main-final-functionL}. In particular, assume
that
\(
\alpha_i=\beta_i=0,
~ i=1,\ldots,\lambda .
\)
Let \(n_1^{\rm samp}\) denote the number of outer training samples in
Setting~\ref{set.1}, and define
\[
Q_{\max}
:=
\max_{1\le i\le\lambda}
\frac{d_i}{n_i-2k_i}.
\]
Then one can choose a shared branch--trunk ReLU neural operator class
\(\mathcal F_G^{(\lambda)}\) of the form
\[
\mathcal G_{\boldsymbol\theta}(f_1,\ldots,f_\lambda)(\vx)
=
\sum_{s_1=1}^{J_1}\cdots
\sum_{s_\lambda=1}^{J_\lambda}
\sum_{p=1}^{P}
e_{s_1,\ldots,s_\lambda,p}
\prod_{i=1}^{\lambda}
\mathcal B_{i,s_i}(\mathcal D_{m_i}f_i)
\,
\mathcal T_p(\vx),
\]
where the input-wise discretization levels, branch ranks, and trunk rank are
balanced as in Corollary~\ref{cor:more-balanced-main-final-functionL}. Let
\(\widehat{\mathcal G}_{\mathcal S}\in\mathcal F_G^{(\lambda)}\) be the
Sobolev empirical risk minimizer defined in Setting~\ref{set.1}. Then
\begin{align}
&\mathbb E_{\mathcal S}
\mathbb E_{\substack{
f_i\sim\mu_i,\ i=1,\ldots,\lambda\\
\vz\sim\mu_z
}}
\left[
\sum_{|\gamma|\le \ell}
\left|
\partial^\gamma
\mathcal G(f_1,\ldots,f_\lambda)(\vz)
-
\partial^\gamma
\widehat{\mathcal G}_{\mathcal S}
(\mathcal D_{m_1}f_1,\ldots,\mathcal D_{m_\lambda}f_\lambda)(\vz)
\right|^2
\right]
\notag\\
&\qquad\le
C
\left(
\frac{\log n_1^{\rm samp}}
{\log\log n_1^{\rm samp}}
\right)^{-\frac{2}{Q_{\max}}}
(\log\log n_1^{\rm samp})^{2\sum_{i=1}^{\lambda}d_i}.
\label{eq:gene-lambda-rate}
\end{align}
Here \(C>0\) is independent of the sample sizes.
\end{corollary}

\section{Proof of Main Results}
\subsection{Proof of Theorem \ref{thm:main-final-functionL}}
\label{proof:thm:main-final-functionL}
The proof of Theorem~\ref{thm:main-final-functionL} is based on four
approximation steps. Let
\[
\mathcal{P}_N:\mathbb{R}^{(N+1)^d}\to \overline{Q}_N
\]
be the multi-dimensional Chebyshev interpolation operator associated with the tensorized Chebyshev grids, and define
\[
f_{\bar m}:=\mathcal P_{\bar m}\mathcal D_{\bar m}f,
\qquad
g_m:=\mathcal P_m\mathcal D_mg .
\]
We decompose the error as
\begin{align}
&\left\|
\mathcal G(f,g)
-
\sum_{s=1}^J\sum_{h=1}^H\sum_{p=1}^P
e_{h,p,s}\,
\mathcal E_h(\mathcal D_{\bar m}f)\,
\mathcal B_s(\mathcal D_m g)\,
\mathcal T_p
\right\|_{W^{\ell,\infty}(\Omega)}
\notag\\
\leq\;&
\underbrace{
\|\mathcal G(f,g)-\mathcal G(f_{\bar m},g_m)\|_{W^{\ell,\infty}(\Omega)}
}_{\text{Step 1}}
+
\underbrace{
\left\|
\mathcal G(f_{\bar m},g_m)
-
\sum_{s=1}^J
\mathfrak d_{s}(F_m^{f_{\bar m},\vx})\,
\mathcal B_s(\mathcal D_m g)
\right\|_{W^{\ell,\infty}(\Omega)}
}_{\text{Step 2}}
\notag\\
&+
\underbrace{
\left\|
\sum_{s=1}^J
\left[
\mathfrak d_{s}(F_m^{f_{\bar m},\vx})
-
\sum_{h=1}^H e_{h,s}(\vx)\,
\mathcal E_h(\mathcal D_{\bar m}f)
\right]
\mathcal B_s(\mathcal D_m g)
\right\|_{W^{\ell,\infty}(\Omega)}
}_{\text{Step 3}}
\notag\\
&+
\underbrace{
\left\|
\sum_{s=1}^J\sum_{h=1}^H
\left[
e_{h,s}(\vx)
-
\sum_{p=1}^P e_{h,p,s}\,
\mathcal T_p(\vx)
\right]
\mathcal E_h(\mathcal D_{\bar m}f)
\mathcal B_s(\mathcal D_m g)
\right\|_{W^{\ell,\infty}(\Omega)}
}_{\text{Step 4}} .
\label{eq:equations-of-NN}
\end{align}
{where $\mathfrak d_{s}(F_m^{f_{\bar m},\vx})$ depending on $f_{\bar{m}},\vx$ is some coefficient used to approximate $\mathcal G(f_{\bar m},g_m)$ along the second argument. Its expression is presented in Section \ref{sec:step2}.}
The four terms correspond to the four main steps of the construction. In
Step~1, we discretize the input functions \(f\) and \(g\). The approximants
\(f_{\bar m}\) and \(g_m\) are determined solely by the finite pointwise
information \(\mathcal D_{\bar m}f\) and \(\mathcal D_m g\), respectively. In
Step~2, for each fixed \(f_{\bar m}\) and \(\vx\), we regard
\(\mathcal G(f_{\bar m},g_m)(\vx)\) as a function of the finite-dimensional
variable \(\mathcal D_m g\), and approximate it by a basis expansion whose
coefficients depend only on \(f_{\bar m}\) and \(\vx\). In Step~3, we view these
coefficients as functions of the finite-dimensional variable
\(\mathcal D_{\bar m}f\), and apply the same approximation procedure again. In
Step~4, the remaining coefficient functions depend only on the spatial variable
\(\vx\), and are approximated by trunk networks.

We now state the proposition corresponding to each step. The detailed proofs of
these four steps are deferred to the appendix.

\subsubsection{Step 1}

In this step, we estimate the discretization error. We have the following proposition for Step 1 (see a proof in Appendix \ref{proof:step1}):
\begin{proposition}[Approximation error by Chebyshev interpolation]
\label{prop:step1-approximation}
Suppose Assumption~\ref{asspde} and~\ref{assump:X-bound} hold. We have 
\begin{align}
\|\mathcal G(f,g)-\mathcal G(f_{\bar m},g_m)\|_{W^{\ell,\infty}(\Omega)}
\le
C U
\left[
L_1\,\bar m^{-n_1}(\log \bar m)^d
\left(1+S m^{2k_2}(\log m)^d\right)^{\alpha_1}
+
L_2\,m^{-n_2}(\log m)^d
\right].
\end{align}
\end{proposition}

\subsubsection{Step 2}
\label{sec:step2}
Next, we estimate the error arising from the approximation of the
\(g\)-dependent coefficient.

\begin{proposition}[Approximation of the \(g\)-dependent coefficient]
\label{prop:step2-g-coefficient}
Assume that Assumption~\ref{asspde} holds with the
\(W^{\ell,\infty}(\Omega)\)-norm on the left-hand side, where
\(\ell\in\mathbb N\). Assume also that \(J\) is large enough. Fix
\(f_{\bar m}\), and define
\(F_m^{f_{\bar m},\vx}(\vz):=
\mathcal G(f_{\bar m},\mathcal P_m\vz)(\vx)\) for
\(\vz\in[-S,S]^{m_*}\). Set
\[
A_{m,\bar m}^{(g)}
:=
 L_1 m^{2k_2}(\log m)^d
\left(
1+S\bar m^{2k_1}(\log \bar m)^d
\right)^{\alpha_2},
\]
and
\begin{align}
B_{m,\bar m}^{(g)}
:=
1
&+
L_1S m^{2k_2}(\log m)^d
+
L_2S\bar m^{2k_1}(\log \bar m)^d
\left(
1+S m^{2k_2}(\log m)^d
\right)^{\alpha_1}.
\end{align}
Then, for any \(J\in\mathbb N\), there exist points
\(\vz_1,\ldots,\vz_J\in[-S,S]^{m_*}\) and branch networks
\(\mathcal B_s\in
\mathcal F_{\mathrm{NN}}(m_*,1,L_B,p_B,K_B,\kappa_B,1)\),
\(s=1,\ldots,J\), such that, with
\[
\mathfrak d_s(F_m^{f_{\bar m}},\vx)
:=
F_m^{f_{\bar m},\vx}(\vz_s)
=
\mathcal G(f_{\bar m},\mathcal P_m\vz_s)(\vx),
\]
we have
\begin{align}
&\left\|
\mathcal G(f_{\bar m},g_m)
-
\sum_{s=1}^{J}
\mathfrak d_s(F_m^{f_{\bar m}},\cdot)
\mathcal B_s(\mathcal D_m g)
\right\|_{W^{\ell,\infty}(\Omega)}
\le
C S A_{m,\bar m}^{(g)}\sqrt{m_*}\,J^{-1/m_*}.
\label{eq:step2-Well-bound}
\end{align}
Moreover, the branch networks satisfy
\(\|\sum_{s=1}^{J}\mathcal B_s\|_{L^\infty([-S,S]^{m_*})}\le 2\) and
\(\mathcal B_s\ge0\). Furthermore, they can be chosen with
\(p_B=\mathcal O(1)\),
\[
L_B
=
\mathcal O\!\left(m_*^2\log m_*+m_*\log J\right),
\qquad
K_B
=
\mathcal O\!\left(m_*^2\log m_*+m_*\log J\right),
\]
and
\[
\kappa_B
=
\mathcal O\!\left(
\sqrt{m_*}B_{m,\bar m}^{(g)}J^{1+\frac1{m_*}}
\right).
\]
The implicit constants are independent of \(m\), \(m_*\), \(J\), and
\(f_{\bar m}\).
\end{proposition}
Proposition \ref{prop:step2-g-coefficient} is proved in Appendix \ref{proof:step2}.

\subsubsection{Step 3}

For the third term, we further approximate the dependence of
\(
c_s^{f_{\bar m}}(\vx)
:=
\mathfrak d_s(F_m^{f_{\bar m},\vx})
\)
on \(f_{\bar m}\). 

\begin{proposition}[Approximation of the \(f\)-dependent coefficient]
\label{prop:step3-f-coefficient}
Assume that Assumption~\ref{asspde} holds with the
\(W^{\ell,\infty}(\Omega)\)-norm on the left-hand side, where
\(\ell\in\mathbb N\). Assume that \(H\) is large enough. Fix
\(s=1,\ldots,J\), and define
\(\bar F_{\bar m,s}^{\vx}(\vw):=
\mathfrak d_s(F_m^{\mathcal P_{\bar m}\vw,\vx})\) for
\(\vw\in[-S,S]^{\bar m_*}\). Set
\[
A_{\bar m,m}^{(f)}
:=
 L_2
\bar m^{2k_1}(\log \bar m)^d
\left(
1+S m^{2k_2}(\log m)^d
\right)^{\alpha_1},
\]
and
\begin{align}
\bar B_{\bar m,m}^{(f)}
:=
1
&+
L_2 S\bar m^{2k_1}(\log \bar m)^d
\left(
1+S m^{2k_2}(\log m)^d
\right)^{\alpha_1}
+
L_1 S m^{2k_2}(\log m)^d.
\end{align}
Then, for any \(H\in\mathbb N\), there exist points
\(\vw_1,\ldots,\vw_H\in[-S,S]^{\bar m_*}\) and branch networks
\(\mathcal E_h\in
\mathcal F_{\mathrm{NN}}(\bar m_*,1,L_E,p_E,K_E,\kappa_E,1)\),
\(h=1,\ldots,H\), such that, with
\[
e_{h,s}(\vx):=\bar F_{\bar m,s}^{\vx}(\vw_h)
=
\mathcal G(\mathcal P_{\bar m}\vw_h,\mathcal P_m\vz_s)(\vx),
\]
we have
\begin{align}
&
\left\|
c_s^{f_{\bar m}}
-
\sum_{h=1}^{H}
e_{h,s}\mathcal E_h(\mathcal D_{\bar m}f)
\right\|_{W^{\ell,\infty}(\Omega)}
\le
C S A_{\bar m,m}^{(f)}
\sqrt{\bar m_*}\,H^{-1/\bar m_*}.
\label{eq:step3-Well-bound}
\end{align}
Moreover, the branch networks satisfy
\(\|\sum_{h=1}^{H}\mathcal E_h\|_{L^\infty([-S,S]^{\bar m_*})}\le2\) and
\(\mathcal E_h\ge0\). The networks \(\mathcal E_h\) can be chosen with
\(p_E=\mathcal O(1)\),
\[
L_E=
\mathcal O\!\left(\bar m_*^2\log \bar m_*+\bar m_*\log H\right),
\qquad
K_E=
\mathcal O\!\left(\bar m_*^2\log \bar m_*+\bar m_*\log H\right),
\]
and
\[
\kappa_E
=
\mathcal O\!\left(
\sqrt{\bar m_*}\,
\bar B_{\bar m,m}^{(f)}
H^{1+\frac1{\bar m_*}}
\right).
\]
The implicit constants are independent of \(m,\bar m,H,s\), and \(f\).
\end{proposition}
Proposition \ref{prop:step3-f-coefficient} is proved in Appendix \ref{proof:step3}.

\subsubsection{Step 4}

For the last term, we approximate the remaining spatial coefficient functions
\(e_{h,s}(\vx)\) by trunk networks. Moreover, the trunk networks satisfy the local-overlap bounds
\[
\sup_{\vx\in\Omega}
\sum_{p=1}^{P}
|\mathcal T_p(\vx)|
\le C,
\]
and, when \(\ell=1\),
\[
\sup_{\vx\in\Omega}
\sum_{p=1}^{P}
\sum_{r=1}^{d}
|\partial_{x_r}\mathcal T_p(\vx)|
\le
C P^{1/d}.
\]

\begin{proposition}[Trunk approximation of the spatial coefficients]
\label{prop:step4-trunk-coefficient}
Assume that Assumption~\ref{assump:G-output-regularity} holds. Let
\(\ell\in\{0,1\}\) and assume that \(n_3>\ell\). For every
\(h=1,\ldots,H\), \(s=1,\ldots,J\), and \(P\ge2\), there exist coefficients
\(e_{h,p,s}\in\mathbb R\) and trunk networks
\(\mathcal T_p\in
\mathcal F_{\mathrm{NN}}(d,1,L_T,p_T,K_T,\kappa_T,M_T)\),
\(p=1,\ldots,P\), such that
\begin{align}
\left\|
e_{h,s}
-
\sum_{p=1}^{P}
e_{h,p,s}\mathcal T_p
\right\|_{W^{\ell,\infty}(\Omega)}
\le
C R_{\bar m,m}^{(n_3)}
P^{-\frac{n_3-\ell}{d}},
\label{eq:trunk-approx-Well}
\end{align}
and \(|e_{h,p,s}|\le C R_{\bar m,m}^{(n_3)}\) for
\(p=1,\ldots,P\), where
\[
R_{\bar m,m}^{(n_3)}
:=
\left(
1+S\bar m^{2\beta_1}(\log \bar m)^d
\right)
\left(
1+S m^{2\beta_2}(\log m)^d
\right).
\]

Moreover, the trunk networks can be chosen to satisfy the local-overlap bounds
\[
\sup_{\vx\in\Omega}
\sum_{p=1}^{P}
|\mathcal T_p(\vx)|
\le C,\quad \sup_{\vx\in\Omega}
\sum_{p=1}^{P}
\sum_{r=1}^{d}
|\partial_{x_r}\mathcal T_p(\vx)|
\le
C P^{1/d}.
\]
The trunk networks can be chosen with
\(L_T\le C\log P\), \(p_T\le C\), \(K_T\le C\log P\),
\(\kappa_T\le C P^{(n_3+d-\ell)/d}\), and \(M_T\le C\). Here \(C>0\) is
independent of \(h,s,P,m,\bar m\).
\end{proposition}
Proposition \ref{prop:step4-trunk-coefficient} is proved in Appendix \ref{proof:step4}.

\subsubsection{Proof of Theorem~\ref{thm:main-final-functionL}}
\begin{proof}[Proof of Theorem~\ref{thm:main-final-functionL}]
We prove the result by combining the four approximation steps.

First, define
\(f_{\bar m}:=\mathcal P_{\bar m}\mathcal D_{\bar m}f\) and
\(g_m:=\mathcal P_m\mathcal D_m g\). By
Proposition~\ref{prop:step1-approximation}, we have
\begin{align}
&\|\mathcal G(f,g)-\mathcal G(f_{\bar m},g_m)\|_{W^{\ell,\infty}(\Omega)}
\le
C\Bigl[
L_2\bar m^{-n_1}(\log \bar m)^d
\left(
1+S m^{2k_2}(\log m)^d
\right)^{\alpha_1}
+
L_1m^{-n_2}(\log m)^d
\Bigr].
\label{eq:proof-step1}
\end{align}
Indeed, the first term corresponds to replacing \(f\) by \(f_{\bar m}\), while
the background second input is \(g_m\). The factor
\(\|g_m\|_{W^{k_2,\infty}}\) is controlled by the Sobolev stability estimate
\(\|g_m\|_{W^{k_2,\infty}(\Omega)}
\le C S m^{2k_2}(\log m)^d\).

Next, for fixed \(f_{\bar m}\) and \(\vx\in\Omega\), define
\(F_m^{f_{\bar m},\vx}(\vz)
:=\mathcal G(f_{\bar m},\mathcal P_m\vz)(\vx)\) for
\(\vz\in[-S,S]^{m_*}\). By
Proposition~\ref{prop:step2-g-coefficient}, there exist points
\(\vz_s\in[-S,S]^{m_*}\), \(s=1,\ldots,J\), and branch networks
\[
\mathcal B_s\in
\mathcal F_{\mathrm{NN}}(m_*,1,L_B,p_B,K_B,\kappa_B,1),
\qquad
s=1,\ldots,J,
\]
such that, with
\(c_s^{f_{\bar m}}(\vx):=\mathfrak d_s(F_m^{f_{\bar m},\vx})
=F_m^{f_{\bar m},\vx}(\vz_s)\), we have
\begin{align}
\left\|
\mathcal G(f_{\bar m},g_m)
-
\sum_{s=1}^{J}
c_s^{f_{\bar m}}\mathcal B_s(\mathcal D_m g)
\right\|_{W^{\ell,\infty}(\Omega)}
\le
C S A_{m,\bar m}^{(g)}\sqrt{m_*}\,J^{-1/m_*}.
\label{eq:proof-step2}
\end{align}
The network-size bounds for \(\mathcal B_s\) are precisely those stated in the
theorem. Moreover,
\[
\left\|
\sum_{s=1}^{J}\mathcal B_s
\right\|_{L^\infty([-S,S]^{m_*})}
\le 2,
\qquad
\mathcal B_s\ge0.
\]

We then approximate the dependence of \(c_s^{f_{\bar m}}\) on \(f_{\bar m}\).
For each \(s=1,\ldots,J\), Proposition~\ref{prop:step3-f-coefficient} gives
points \(\vw_h\in[-S,S]^{\bar m_*}\), \(h=1,\ldots,H\), and branch networks
\[
\mathcal E_h\in
\mathcal F_{\mathrm{NN}}(\bar m_*,1,L_E,p_E,K_E,\kappa_E,1),
\qquad
h=1,\ldots,H,
\]
such that
\begin{align}
\left\|
c_s^{f_{\bar m}}
-
\sum_{h=1}^{H}
e_{h,s}\mathcal E_h(\mathcal D_{\bar m}f)
\right\|_{W^{\ell,\infty}(\Omega)}
\le
C S A_{\bar m,m}^{(f)}
\sqrt{\bar m_*}\,H^{-1/\bar m_*}.
\end{align}
Using the nonnegativity of \(\mathcal B_s\) and the bound
\(\|\sum_{s=1}^{J}\mathcal B_s\|_{L^\infty([-S,S]^{m_*})}\le2\), we obtain
\begin{align}
&\left\|
\sum_{s=1}^{J}
\left[
c_s^{f_{\bar m}}
-
\sum_{h=1}^{H}
e_{h,s}\mathcal E_h(\mathcal D_{\bar m}f)
\right]
\mathcal B_s(\mathcal D_m g)
\right\|_{W^{\ell,\infty}(\Omega)}
\le
C S A_{\bar m,m}^{(f)}
\sqrt{\bar m_*}\,H^{-1/\bar m_*}.
\label{eq:proof-step3}
\end{align}
Here we used that \(\mathcal B_s(\mathcal D_m g)\) is independent of \(\vx\),
so the \(W^{\ell,\infty}\)-derivatives act only on the coefficient functions.
The corresponding bounds for \(L_E,p_E,K_E,\kappa_E\) are those in
Proposition~\ref{prop:step3-f-coefficient}. Moreover,
\[
\left\|
\sum_{h=1}^{H}\mathcal E_h
\right\|_{L^\infty([-S,S]^{\bar m_*})}
\le 2,
\qquad
\mathcal E_h\ge0.
\]

It remains to approximate the spatial coefficient functions \(e_{h,s}(\vx)\).
By Proposition~\ref{prop:step4-trunk-coefficient}, for each \(h=1,\ldots,H\)
and \(s=1,\ldots,J\), there exist coefficients \(e_{h,p,s}\in\mathbb R\) and
trunk networks
\[
\mathcal T_p\in
\mathcal F_{\mathrm{NN}}(d,1,L_T,p_T,K_T,\kappa_T,M_T),
\qquad p=1,\ldots,P,
\]
such that
\begin{align}
\left\|
e_{h,s}
-
\sum_{p=1}^{P}
e_{h,p,s}\mathcal T_p
\right\|_{W^{\ell,\infty}(\Omega)}
\le
C R_{\bar m,m}^{(n_3)}P^{-\frac{n_3-\ell}{d}},
\end{align}
and \(|e_{h,p,s}|\le C R_{\bar m,m}^{(n_3)}\). Therefore, using
\(\|\sum_{s=1}^{J}\mathcal B_s\|_{L^\infty([-S,S]^{m_*})}\le2\),
\(\|\sum_{h=1}^{H}\mathcal E_h\|_{L^\infty([-S,S]^{\bar m_*})}\le2\), and
\(\mathcal B_s,\mathcal E_h\ge0\), we obtain
\begin{align}
&\left\|
\sum_{s=1}^{J}\sum_{h=1}^{H}
\left[
e_{h,s}
-
\sum_{p=1}^{P}
e_{h,p,s}\mathcal T_p
\right]
\mathcal E_h(\mathcal D_{\bar m}f)
\mathcal B_s(\mathcal D_m g)
\right\|_{W^{\ell,\infty}(\Omega)}
\le
C R_{\bar m,m}^{(n_3)}P^{-\frac{n_3-\ell}{d}}
\left\|
\sum_{s=1}^{J}\sum_{h=1}^{H}
\mathcal E_h(\mathcal D_{\bar m}f)
\mathcal B_s(\mathcal D_m g)
\right\|_{L^\infty(\Omega)}
\notag\\
&\qquad\le
C R_{\bar m,m}^{(n_3)}P^{-\frac{n_3-\ell}{d}}
\left\|
\sum_{s=1}^{J}
\mathcal B_s(\mathcal D_m g)
\sum_{h=1}^{H}
\mathcal E_h(\mathcal D_{\bar m}f)
\right\|_{L^\infty(\Omega)}
\le
C R_{\bar m,m}^{(n_3)}P^{-\frac{n_3-\ell}{d}}.
\label{eq:proof-step4}
\end{align}
Here again the branch factors are independent of \(\vx\), and therefore the
\(W^{\ell,\infty}\)-derivatives act only on the trunk approximation error. The
stated bounds for \(L_T,p_T,K_T,\kappa_T,M_T\) also follow from
Proposition~\ref{prop:step4-trunk-coefficient}.

Finally, by the triangle inequality, the total error is bounded by the sum of
\eqref{eq:proof-step1}, \eqref{eq:proof-step2}, \eqref{eq:proof-step3}, and
\eqref{eq:proof-step4}. This yields \eqref{eq:main-final-error-functionL}.

We also record the output bound of the resulting neural operator. By
Proposition~\ref{prop:step4-trunk-coefficient}, for each \(h,s\),
\[
\sup_{\vx\in\Omega}
\sum_{|\gamma|\le \ell}
\left|
\partial^\gamma
\left(
\sum_{p=1}^{P}
e_{h,p,s}\mathcal T_p(\vx)
\right)
\right|
\le
C R_{\bar m,m}^{(n_3)}P^{\ell/d}.
\]
On the other hand, the partition-of-unity-type bounds for the branch networks
give
\[
\sup_{\vu\in[-S,S]^{m_*}}
\sum_{s=1}^{J}
|\mathcal B_s(\vu)|
\le C,
\qquad
\sup_{\vw\in[-S,S]^{\bar m_*}}
\sum_{h=1}^{H}
|\mathcal E_h(\vw)|
\le C .
\]
Since \(\mathcal E_h(\mathcal D_{\bar m}f)\) and
\(\mathcal B_s(\mathcal D_mg)\) are independent of the output variable \(\vx\),
the derivative \(\partial^\gamma\) only acts on the trunk expansion. Therefore,
\begin{align*}
&\sup_{\vx\in\Omega}
\sum_{|\gamma|\le \ell}
\left|
\partial^\gamma
\left[
\sum_{s=1}^{J}\sum_{h=1}^{H}\sum_{p=1}^{P}
e_{h,p,s}
\mathcal E_h(\mathcal D_{\bar m}f)
\mathcal B_s(\mathcal D_mg)
\mathcal T_p(\vx)
\right]
\right|
\\
&\qquad\le
\sum_{s=1}^{J}
|\mathcal B_s(\mathcal D_mg)|
\sum_{h=1}^{H}
|\mathcal E_h(\mathcal D_{\bar m}f)|
\sup_{\vx\in\Omega}
\sum_{|\gamma|\le \ell}
\left|
\partial^\gamma
\left(
\sum_{p=1}^{P}
e_{h,p,s}\mathcal T_p(\vx)
\right)
\right|
\\
&\qquad\le
C R_{\bar m,m}^{(n_3)}P^{\ell/d}
\left(
\sum_{s=1}^{J}
|\mathcal B_s(\mathcal D_mg)|
\right)
\left(
\sum_{h=1}^{H}
|\mathcal E_h(\mathcal D_{\bar m}f)|
\right)
\\
&\qquad\le
C R_{\bar m,m}^{(n_3)}P^{\ell/d}.
\end{align*}
Taking the supremum over \((f,g)\in\mathcal X_1\times\mathcal X_2\) gives
\eqref{eq:NN-output-bound-Well}. The proof is complete.
\end{proof}

\subsection{Proof of Corollary \ref{cor:balanced-main-final-functionL}}
\begin{proof}[Proof of Corollary \ref{cor:balanced-main-final-functionL}]
By Theorem~\ref{thm:main-final-functionL} with \(d_i=d\), \(k_i=0\),
\(\alpha_i=0\), and \(\beta_i=0\), together with the network-size bounds in the
construction, the approximation error satisfies
\begin{align}
E_{\mathrm{total}}
\lesssim\;&
\bar m_*^{-n_1/d}(\log \bar m_*)^d
+
m_*^{-n_2/d}(\log m_*)^d
+
S L_1\sqrt{m_*}(\log m_*)^dJ^{-1/m_*}
\notag\\
&+
S L_2\sqrt{\bar m_*}(\log \bar m_*)^dH^{-1/\bar m_*}
+
\left(1+S(\log \bar m_*)^d\right)
\left(1+S(\log m_*)^d\right)
P^{-\frac{n_3-\ell}{d}} .
\label{eq:cor-balance-start}
\end{align}

We first balance the discretization and branch approximation errors in each
input variable. Balancing
\(m_*^{-n_2/d}(\log m_*)^d\) with
\(\sqrt{m_*}(\log m_*)^dJ^{-1/m_*}\) gives
\(J^{1/m_*}\asymp m_*^{n_2/d+1/2}\), or equivalently
\[
\log J
\asymp
\left(\frac{n_2}{d}+\frac12\right)m_*\log m_* .
\]
Similarly, balancing
\(\bar m_*^{-n_1/d}(\log \bar m_*)^d\) with
\(\sqrt{\bar m_*}(\log \bar m_*)^dH^{-1/\bar m_*}\) gives
\(H^{1/\bar m_*}\asymp \bar m_*^{n_1/d+1/2}\), or equivalently
\[
\log H
\asymp
\left(\frac{n_1}{d}+\frac12\right)\bar m_*\log \bar m_* .
\]

Let \(0<\varepsilon<1\) be the target approximation accuracy. We choose
\(m_*\) and \(\bar m_*\) such that
\(m_*^{-n_2/d}(\log m_*)^d\asymp\varepsilon\) and
\(\bar m_*^{-n_1/d}(\log \bar m_*)^d\asymp\varepsilon\). Thus, up to
logarithmic factors,
\(m_*\asymp\varepsilon^{-d/n_2}\) and
\(\bar m_*\asymp\varepsilon^{-d/n_1}\). With the balanced choices of \(J\) and
\(H\) above, the first four terms in \eqref{eq:cor-balance-start} are bounded by
\(C\varepsilon\).

It remains to choose \(P\) so that the trunk approximation term is also of order
\(\varepsilon\). Since
\((\log m_*)^d\asymp(\log \bar m_*)^d\asymp(\log(1/\varepsilon))^d\), up to
constants, it suffices to impose
\[
\left(1+S\left(\log\frac1\varepsilon\right)^d\right)^2
P^{-\frac{n_3-\ell}{d}}
\lesssim
\varepsilon .
\]
For instance, we may take
\[
P
\asymp
\varepsilon^{-\frac{d}{n_3-\ell}}
\left(1+S\left(\log\frac1\varepsilon\right)^d\right)^{\frac{2d}{n_3-\ell}}.
\]
Then all five terms in \eqref{eq:cor-balance-start} are bounded by
\(C\varepsilon\), and hence \(E_{\mathrm{total}}\lesssim\varepsilon\).

We now express \(\varepsilon\) in terms of the total number of parameters. For
the shared architecture
\[
\mathcal{G}_{\boldsymbol\theta}(f,g)(\vx)
=
\sum_{s=1}^J\sum_{h=1}^H\sum_{p=1}^P
e_{h,p,s}\,
\mathcal E_h(\mathcal D_{\bar m}f)\,
\mathcal B_s(\mathcal D_m g)\,
\mathcal T_p(\vx),
\]
there are \(H\) networks of type \(\mathcal E\), \(J\) networks of type
\(\mathcal B\), and \(P\) networks of type \(\mathcal T\), together with
\(JHP\) scalar coefficients. Therefore,
\[
N_{\mathrm{tot}}
\lesssim
H K_E+J K_B+P K_T+JHP .
\]
The last term counts the coefficients \(e_{h,p,s}\).

From the choices above,
\[
\log J
\asymp
m_*\log m_*
\lesssim
\varepsilon^{-d/n_2}
\left(\log\frac1\varepsilon\right)^C,
\qquad
\log H
\asymp
\bar m_*\log \bar m_*
\lesssim
\varepsilon^{-d/n_1}
\left(\log\frac1\varepsilon\right)^C .
\]
Moreover, \(\log P\lesssim \log(1/\varepsilon)+
\log(1+S(\log(1/\varepsilon))^d)\), which is of lower order compared with
\(\log J+\log H\). Hence
\[
\log N_{\mathrm{tot}}
\lesssim
\left(
\varepsilon^{-d/n_1}
+
\varepsilon^{-d/n_2}
\right)
\left(\log\frac1\varepsilon\right)^C
\lesssim
\varepsilon^{-a}
\left(\log\frac1\varepsilon\right)^C ,
\]
where \(a:=\max\{d/n_1,d/n_2\}\).

Inverting the above relation gives
\[
\varepsilon
\lesssim
\left(
\frac{\log N_{\mathrm{tot}}}{\log\log N_{\mathrm{tot}}}
\right)^{-1/a}
(\log\log N_{\mathrm{tot}})^C .
\]
Moreover,
\[
\left(1+S(\log \bar m_*)^d\right)
\left(1+S(\log m_*)^d\right)
\lesssim
C(\log\log N_{\mathrm{tot}})^{2d}.
\]
Using this non-sharp but uniform logarithmic upper bound, we obtain
\[
E_{\mathrm{total}}
\lesssim
\left(
\frac{\log N_{\mathrm{tot}}}{\log\log N_{\mathrm{tot}}}
\right)^{-1/a}
(\log\log N_{\mathrm{tot}})^{2d}.
\]
Since \(1/a=(\max\{d/n_1,d/n_2\})^{-1}=\min\{n_1,n_2\}/d\), we conclude that
\[
\left\|
\mathcal G(f,g)
-
\mathcal G_{\boldsymbol\theta}(f,g)
\right\|_{W^{\ell,\infty}(\Omega)}
\le
C
\left(
\frac{\log N_{\mathrm{tot}}}{\log\log N_{\mathrm{tot}}}
\right)^{
-\frac{1}{\max\left\{\frac d{n_1},\frac d{n_2}\right\}}
}
(\log\log N_{\mathrm{tot}})^{2d}.
\]
This completes the proof.
\end{proof}

\subsection{Proof of Theorem \ref{thm:gene-sobolev}}
\label{proof:thm:gene-sobolev}
The
overall strategy of proofs follows the empirical-process argument in
\cite{liu2024neural}. The main difference is that we study Sobolev training,
where the loss involves both the operator output and its derivatives with
respect to the spatial variable. In addition, our networks use the ReLU
activation. Although ReLU is Lipschitz continuous, its derivative is the
Heaviside step function, which is not Lipschitz continuous. Therefore, the
covering-number estimate for the derivative classes cannot be obtained directly
from the perturbation argument used in \cite{chen2022nonparametric} and
\cite{liu2024neural}, which relies on Lipschitz continuity of the activation
function and its derivatives.

To overcome this difficulty, we use a uniform empirical covering argument. More
precisely, we regard the derivative-augmented neural operator class as a
real-valued function class on finite samples and bound its uniform empirical
covering number through pseudo-dimension estimates. The required empirical
covering estimate follows from \cite[Theorem~12.2]{anthony1999neural}. This
allows us to control the stochastic and covering-number terms even in the
presence of derivative observations.

After choosing the network size so that the stochastic and covering-number
terms are dominated by the approximation term, the resulting generalization
error becomes approximation-limited.

\begin{proof}[Proof of Theorem~\ref{thm:gene-sobolev}]
We first introduce the notation
\(\partial^\gamma \mathcal G(f,g)(\vz):=
\partial_{\vz}^{\gamma}\mathcal G(f,g)(\vz)\) for \(|\gamma|\le \ell\).
For any neural operator \(\mathcal H\), define the empirical Sobolev norm by
\[
\|\mathcal H\|_{n,\ell}^2
:=
\frac{1}{n_fn_gn_z}
\sum_{i=1}^{n_f}
\sum_{j=1}^{n_g}
\sum_{k=1}^{n_z}
\sum_{|\gamma|\le \ell}
\left|
\partial^\gamma
\mathcal H
(\mathcal D_{\bar m}f_i,\mathcal D_mg_{i,j})(\vz_{i,j,k})
\right|^2 .
\]

Since the test samples \(g_j\sim\mu_g\) and \(\vz_{j,k}\sim\mu_z\) are i.i.d.
and independent of the training set \(\mathcal S\), the expected batch test
error is equal to the population Sobolev prediction error:
\begin{align}
&\mathbb E_{\mathcal S}
\mathbb E_{f\sim\mu_f}
\mathbb E_{\{g_j\}_{j=1}^{n_g}}
\mathbb E_{\{\vz_{j,k}\}_{j,k}}
\Biggl[
\frac1{n_gn_z}
\sum_{j=1}^{n_g}
\sum_{k=1}^{n_z}
\sum_{|\gamma|\le \ell}
\left|
\partial^\gamma\mathcal G(f,g_j)(\vz_{j,k})
-
\partial^\gamma
\widehat{\mathcal G}_{\mathcal S}
(\mathcal D_{\bar m}f,\mathcal D_mg_j)(\vz_{j,k})
\right|^2
\Biggr]
\notag\\
&=
\mathbb E_{\mathcal S}
\mathbb E_{f\sim\mu_f}
\mathbb E_{g\sim\mu_g}
\mathbb E_{\vz\sim\mu_z}
\sum_{|\gamma|\le \ell}
\left|
\partial^\gamma\mathcal G(f,g)(\vz)
-
\partial^\gamma
\widehat{\mathcal G}_{\mathcal S}
(\mathcal D_{\bar m}f,\mathcal D_mg)(\vz)
\right|^2 .
\end{align}
Denote the left-hand side by \(\mathcal E_{\rm gen}\). We decompose it as
\(\mathcal E_{\rm gen}={\rm T}_1+{\rm T}_2\), where
\[
{\rm T}_1
:=
2\,
\mathbb E_{\mathcal S}
\left[
\|\widehat{\mathcal G}_{\mathcal S}-\mathcal G\|_{n,\ell}^2
\right],
\qquad
{\rm T}_2
:=
\mathcal E_{\rm gen}
-
2\,
\mathbb E_{\mathcal S}
\left[
\|\widehat{\mathcal G}_{\mathcal S}-\mathcal G\|_{n,\ell}^2
\right].
\]
We estimate these two terms separately.

\medskip
\noindent\textbf{Bounding \({\rm T}_1\).}
For the Sobolev training data, we have
\(y_{i,j,k}^{(\gamma)}
=
\partial^\gamma\mathcal G(f_i,g_{i,j})(\vz_{i,j,k})
+
\varepsilon_{i,j,k}^{(\gamma)}\) for \(|\gamma|\le \ell\). Recall that
\({\rm T}_1
=
2\mathbb E_{\mathcal S}
[
\|\widehat{\mathcal G}_{\mathcal S}-\mathcal G\|_{n,\ell}^2
]\), that is,
\begin{align}
{\rm T}_1
=
2\mathbb E_{\mathcal S}
\Biggl[
\frac{1}{n_fn_gn_z}
\sum_{i=1}^{n_f}
\sum_{j=1}^{n_g}
\sum_{k=1}^{n_z}
\sum_{|\gamma|\le \ell}
\Bigl|
&
\partial^\gamma
\widehat{\mathcal G}_{\mathcal S}
(\mathcal D_{\bar m}f_i,\mathcal D_mg_{i,j})(\vz_{i,j,k})
-
\partial^\gamma
\mathcal G(f_i,g_{i,j})(\vz_{i,j,k})
\Bigr|^2
\Biggr].
\end{align}
Using
\(\partial^\gamma
\mathcal G(f_i,g_{i,j})(\vz_{i,j,k})
=
y_{i,j,k}^{(\gamma)}
-
\varepsilon_{i,j,k}^{(\gamma)}\), we obtain
\begin{align}
{\rm T}_1
=&
2\mathbb E_{\mathcal S}
\Biggl[
\frac{1}{n_fn_gn_z}
\sum_{i,j,k}
\sum_{|\gamma|\le \ell}
\Bigl|
\partial^\gamma
\widehat{\mathcal G}_{\mathcal S}
(\mathcal D_{\bar m}f_i,\mathcal D_mg_{i,j})(\vz_{i,j,k})
-
y_{i,j,k}^{(\gamma)}
+
\varepsilon_{i,j,k}^{(\gamma)}
\Bigr|^2
\Biggr]
\notag\\
=&
2\mathbb E_{\mathcal S}
\Biggl[
\frac{1}{n_fn_gn_z}
\sum_{i,j,k}
\sum_{|\gamma|\le \ell}
\Bigl(
\partial^\gamma
\widehat{\mathcal G}_{\mathcal S}
(\mathcal D_{\bar m}f_i,\mathcal D_mg_{i,j})(\vz_{i,j,k})
-
y_{i,j,k}^{(\gamma)}
\Bigr)^2
\Biggr]
\notag\\
&+
4\mathbb E_{\mathcal S}
\Biggl[
\frac{1}{n_fn_gn_z}
\sum_{i,j,k}
\sum_{|\gamma|\le \ell}
\Bigl(
\partial^\gamma
\widehat{\mathcal G}_{\mathcal S}
(\mathcal D_{\bar m}f_i,\mathcal D_mg_{i,j})(\vz_{i,j,k})
-
y_{i,j,k}^{(\gamma)}
\Bigr)
\varepsilon_{i,j,k}^{(\gamma)}
\Biggr]
\notag\\
&+
2\mathbb E_{\mathcal S}
\Biggl[
\frac{1}{n_fn_gn_z}
\sum_{i,j,k}
\sum_{|\gamma|\le \ell}
\bigl(\varepsilon_{i,j,k}^{(\gamma)}\bigr)^2
\Biggr].
\label{eq:T1-sobolev-expand-1}
\end{align}
By the definition of the empirical risk minimizer
\(\widehat{\mathcal G}_{\mathcal S}\), the first term on the right-hand side is
bounded by
\begin{align}
&2
\mathbb E_{\mathcal S}
\inf_{\mathcal G_{\rm NN}\in\mathcal F_G}
\Biggl[
\frac{1}{n_fn_gn_z}
\sum_{i,j,k}
\sum_{|\gamma|\le \ell}
\Bigl(
\partial^\gamma
\mathcal G_{\rm NN}
(\mathcal D_{\bar m}f_i,\mathcal D_mg_{i,j})(\vz_{i,j,k})
-
y_{i,j,k}^{(\gamma)}
\Bigr)^2
\Biggr].
\end{align}
Therefore,
\begin{align}
{\rm T}_1
\le&
2
\inf_{\mathcal G_{\rm NN}\in\mathcal F_G}
\mathbb E_{\mathcal S}
\Biggl[
\frac{1}{n_fn_gn_z}
\sum_{i,j,k}
\sum_{|\gamma|\le \ell}
\Bigl(
\partial^\gamma
\mathcal G_{\rm NN}
(\mathcal D_{\bar m}f_i,\mathcal D_mg_{i,j})(\vz_{i,j,k})
-
y_{i,j,k}^{(\gamma)}
\Bigr)^2
\Biggr]
\notag\\
&+
4\mathbb E_{\mathcal S}
\Biggl[
\frac{1}{n_fn_gn_z}
\sum_{i,j,k}
\sum_{|\gamma|\le \ell}
\Bigl(
\partial^\gamma
\widehat{\mathcal G}_{\mathcal S}
(\mathcal D_{\bar m}f_i,\mathcal D_mg_{i,j})(\vz_{i,j,k})
-
y_{i,j,k}^{(\gamma)}
\Bigr)
\varepsilon_{i,j,k}^{(\gamma)}
\Biggr]
\notag\\
&+
2\mathbb E_{\mathcal S}
\Biggl[
\frac{1}{n_fn_gn_z}
\sum_{i,j,k}
\sum_{|\gamma|\le \ell}
\bigl(\varepsilon_{i,j,k}^{(\gamma)}\bigr)^2
\Biggr].
\label{eq:T1-sobolev-expand-2}
\end{align}
Using
\(y_{i,j,k}^{(\gamma)}
=
\partial^\gamma\mathcal G(f_i,g_{i,j})(\vz_{i,j,k})
+
\varepsilon_{i,j,k}^{(\gamma)}\), we rewrite the first term in
\eqref{eq:T1-sobolev-expand-2}. Subtracting and adding the noise square term
gives
\begin{align}
{\rm T}_1
\le&
2
\inf_{\mathcal G_{\rm NN}\in\mathcal F_G}
\mathbb E_{\mathcal S}
\Biggl[
\frac{1}{n_fn_gn_z}
\sum_{i,j,k}
\sum_{|\gamma|\le \ell}
\Bigl\{
\Bigl(
\partial^\gamma
\mathcal G_{\rm NN}
-
\partial^\gamma
\mathcal G
-
\varepsilon^{(\gamma)}
\Bigr)^2
-
\bigl(\varepsilon^{(\gamma)}\bigr)^2
\Bigr\}
\Biggr]
\notag\\
&+
4\mathbb E_{\mathcal S}
\Biggl[
\frac{1}{n_fn_gn_z}
\sum_{i,j,k}
\sum_{|\gamma|\le \ell}
\Bigl(
\partial^\gamma
\widehat{\mathcal G}_{\mathcal S}
-
\partial^\gamma
\mathcal G
\Bigr)
(\mathcal D_{\bar m}f_i,\mathcal D_mg_{i,j})(\vz_{i,j,k})
\varepsilon_{i,j,k}^{(\gamma)}
\Biggr].
\label{eq:T1-sobolev-expand-3}
\end{align}
Here, in the first line, \(\partial^\gamma\mathcal G_{\rm NN}\),
\(\partial^\gamma\mathcal G\), and \(\varepsilon^{(\gamma)}\) are evaluated at
\((\mathcal D_{\bar m}f_i,\mathcal D_mg_{i,j},\vz_{i,j,k})\). Expanding the
square in the first term of \eqref{eq:T1-sobolev-expand-3}, and using that the
noises are mean-zero and independent of the samples, the cross term vanishes
after taking expectation. Hence
\begin{align}
{\rm T}_1
\le&
2
\inf_{\mathcal G_{\rm NN}\in\mathcal F_G}
\mathbb E_{\mathcal S}
\Biggl[
\frac{1}{n_fn_gn_z}
\sum_{i,j,k}
\sum_{|\gamma|\le \ell}
\Bigl|
\partial^\gamma
\mathcal G_{\rm NN}
(\mathcal D_{\bar m}f_i,\mathcal D_mg_{i,j})(\vz_{i,j,k})
-
\partial^\gamma
\mathcal G(f_i,g_{i,j})(\vz_{i,j,k})
\Bigr|^2
\Biggr]
\notag\\
&+
4\mathbb E_{\mathcal S}
\Biggl[
\frac{1}{n_fn_gn_z}
\sum_{i,j,k}
\sum_{|\gamma|\le \ell}
\Bigl(
\partial^\gamma
\widehat{\mathcal G}_{\mathcal S}
-
\partial^\gamma
\mathcal G
\Bigr)
(\mathcal D_{\bar m}f_i,\mathcal D_mg_{i,j})(\vz_{i,j,k})
\varepsilon_{i,j,k}^{(\gamma)}
\Biggr].
\label{eq:T1-sobolev-1}
\end{align}
This gives the desired decomposition of \({\rm T}_1\).

The first term in \eqref{eq:T1-sobolev-1} is controlled by the Sobolev
approximation result in Corollary~\ref{cor:balanced-main-final-functionL}.
Under the balanced choice of \(\bar m,m,J,H,P\), we have
\begin{align}
&2\inf_{\mathcal G_{\rm NN}\in\mathcal F_G}
\mathbb E_{\mathcal S}
\left[
\|\mathcal G_{\rm NN}-\mathcal G\|_{n,\ell}^2
\right]\le
2C_1
\left(
\frac{\log H}{\log\log H}
\right)^{
-\frac{2}{\max\{\frac d{n_1},\frac d{n_2}\}}
}
(\log\log H)^{4d}.
\label{eq:T1-sobolev-approx}
\end{align}

We next bound the noise-correlation term in \eqref{eq:T1-sobolev-1}. Since
standard ReLU networks do not admit a uniform \(W^{1,\infty}\)-covering number
in general, we use a finite-sample empirical Sobolev covering instead. Let
\(N:=n_fn_gn_z\). For any finite sample set
\(\mathcal A=\{(f_a,g_a,\vz_a)\}_{a=1}^{N_{\mathcal A}}\) with
\(N_{\mathcal A}\le 2N\), define the empirical Sobolev sup-metric by
\[
d_{\mathcal A,\ell}(\mathcal H_1,\mathcal H_2)
:=
\max_{1\le a\le N_{\mathcal A}}
\sum_{|\gamma|\le \ell}
\left|
\partial^\gamma
(\mathcal H_1-\mathcal H_2)
(\mathcal D_{\bar m}f_a,\mathcal D_m g_a)(\vz_a)
\right|.
\]
Let \(\mathcal N(\eta,\mathcal F_G,d_{\mathcal A,\ell})\) denote the
corresponding covering number. Define the uniform empirical covering number
over at most \(2N\) sample points by
\[
\mathcal N_{\rm emp}^{(2N)}(\eta,\mathcal F_G)
:=
\sup_{\mathcal A:\,N_{\mathcal A}\le 2N}
\mathcal N
\left(
\eta,
\mathcal F_G,
d_{\mathcal A,\ell}
\right).
\]
For simplicity, we write
\(\mathcal N_{\rm emp}(\eta,\mathcal F_G):=
\mathcal N_{\rm emp}^{(2N)}(\eta,\mathcal F_G)\) throughout the proof.

This metric controls the empirical Sobolev norm used in the loss. Indeed, if
\(d_{\mathcal S,\ell}(\mathcal H_1,\mathcal H_2)\le \eta\), then
\(\|\mathcal H_1-\mathcal H_2\|_{n,\ell}\le\eta\). This follows because, at each
sample point,
\[
\sum_{|\gamma|\le\ell}
\left|
\partial^\gamma
(\mathcal H_1-\mathcal H_2)
(\mathcal D_{\bar m}f_i,\mathcal D_mg_{i,j})(\vz_{i,j,k})
\right|
\le \eta,
\]
and hence the corresponding sum of squares is bounded by \(\eta^2\). Moreover,
if \(\mathcal S'\) is an independent ghost sample, then \(\mathcal S\cup
\mathcal S'\) contains at most \(2N\) sample points. Hence
\[
\mathcal N
\left(
\eta,
\mathcal F_G,
d_{\mathcal S\cup\mathcal S',\ell}
\right)
\le
\mathcal N_{\rm emp}(\eta,\mathcal F_G).
\]
Thus the same deterministic covering number
\(\mathcal N_{\rm emp}(\eta,\mathcal F_G)\) can be used in both the one-sample
and two-sample empirical-process estimates.

The following lemma gives an upper bound of the second term in \eqref{eq:T1-sobolev-1} (see a proof in Appendix \ref{proof:lem:T1-sobolev-noise}).
\begin{lemma}[Noise correlation bound for Sobolev training]
\label{lem:T1-sobolev-noise}
Assume that, for each \(|\gamma|\le \ell\), the noise variables
\(\varepsilon_{i,j,k}^{(\gamma)}\) are independent, mean-zero, and
sub-Gaussian with variance proxy bounded by \(\sigma^2\). Then, for any
\(\eta>0\),
\begin{align}
&\mathbb E_{\mathcal S}
\Biggl[
\frac{1}{n_fn_gn_z}
\sum_{i=1}^{n_f}
\sum_{j=1}^{n_g}
\sum_{k=1}^{n_z}
\sum_{|\gamma|\le \ell}
\left(
\partial^\gamma
\widehat{\mathcal G}_{\mathcal S}
-
\partial^\gamma\mathcal G
\right)
(\mathcal D_{\bar m}f_i,\mathcal D_mg_{i,j})(\vz_{i,j,k})
\varepsilon_{i,j,k}^{(\gamma)}
\Biggr]
\notag\\
&\qquad\le
2\sigma
\left(
\sqrt{
\mathbb E_{\mathcal S}
\left[
\|\widehat{\mathcal G}_{\mathcal S}-\mathcal G\|_{n,\ell}^2
\right]
}
+\eta
\right)
\sqrt{
\frac{
4\log\mathcal N_{\rm emp}(\eta,\mathcal F_G)+6
}{
n_fn_gn_z
}
}
+\sigma\eta .
\label{eq:T1-sobolev-noise}
\end{align}
\end{lemma}

Applying Lemma~\ref{lem:T1-sobolev-noise} to the second term in
\eqref{eq:T1-sobolev-1}, and combining it with
\eqref{eq:T1-sobolev-approx}, we obtain
\begin{align}
{\rm T}_1
\le\;&
2C_1
\left(
\frac{\log H}{\log\log H}
\right)^{
-\frac{2}{\max\{\frac d{n_1},\frac d{n_2}\}}
}
(\log\log H)^{4d}
\notag\\
&+
8\sigma
\left(
\sqrt{
\mathbb E_{\mathcal S}
\left[
\|\widehat{\mathcal G}_{\mathcal S}-\mathcal G\|_{n,\ell}^2
\right]
}
+\eta
\right)
\sqrt{
\frac{
4\log
\mathcal N_{\rm emp}(\eta,\mathcal F_G)
+6
}{
n_fn_gn_z
}
}
+
4\sigma\eta .
\label{eq:T1-sobolev-bound-1}
\end{align}

Let
\[
\Gamma
:=
\sqrt{
\mathbb E_{\mathcal S}
\left[
\|\widehat{\mathcal G}_{\mathcal S}-\mathcal G\|_{n,\ell}^2
\right]
},
\qquad
\Delta_\eta
:=
\sqrt{
\frac{
4\log
\mathcal N_{\rm emp}(\eta,\mathcal F_G)
+6
}{
n_fn_gn_z
}
}.
\]
Since \({\rm T}_1=2\Gamma^2\), inequality
\eqref{eq:T1-sobolev-bound-1} implies
\[
\Gamma^2
\le
C_1
\left(
\frac{\log H}{\log\log H}
\right)^{
-\frac{2}{\max\{\frac d{n_1},\frac d{n_2}\}}
}
(\log\log H)^{4d}
+
4\sigma(\Gamma+\eta)\Delta_\eta
+
2\sigma\eta .
\]
Equivalently, \(\Gamma^2\le a+2b\Gamma\), where
\[
a
:=
C_1
\left(
\frac{\log H}{\log\log H}
\right)^{
-\frac{2}{\max\{\frac d{n_1},\frac d{n_2}\}}
}
(\log\log H)^{4d}
+
4\sigma\eta\Delta_\eta
+
2\sigma\eta,
\]
and \(b:=2\sigma\Delta_\eta\). Using the elementary implication
\(\Gamma^2\le a+2b\Gamma \Rightarrow \Gamma^2\le 2a+4b^2\), we obtain
\[
\Gamma^2
\le
C
\left(
\frac{\log H}{\log\log H}
\right)^{
-\frac{2}{\max\{\frac d{n_1},\frac d{n_2}\}}
}
(\log\log H)^{4d}
+
C\sigma\eta
+
C\sigma\eta\Delta_\eta
+
C\sigma^2\Delta_\eta^2 .
\]
Since \({\rm T}_1=2\Gamma^2\), we conclude that
\begin{align}
{\rm T}_1
\le\;&
C
\left(
\frac{\log H}{\log\log H}
\right)^{
-\frac{2}{\max\{\frac d{n_1},\frac d{n_2}\}}
}
(\log\log H)^{4d}
+
C\sigma\eta
\notag\\
&+
C\sigma\eta
\sqrt{
\frac{
\log
\mathcal N_{\rm emp}(\eta,\mathcal F_G)
+1
}{
n_fn_gn_z
}
}
+
C\sigma^2
\frac{
\log
\mathcal N_{\rm emp}(\eta,\mathcal F_G)
+1
}{
n_fn_gn_z
}.
\label{eq:T1-sobolev-bound}
\end{align}

\medskip
\noindent\textbf{Bounding \({\rm T}_2\).}
We next bound the deviation term \({\rm T}_2\). Recall that
\({\rm T}_2
=
\mathcal E_{\rm gen}
-
2\mathbb E_{\mathcal S}
[
\|\widehat{\mathcal G}_{\mathcal S}-\mathcal G\|_{n,\ell}^2
]\). Thus \({\rm T}_2\) measures the discrepancy between the population Sobolev
prediction error and the empirical Sobolev prediction error. The following
lemma is the Sobolev analogue of \cite[Lemma~4]{liu2024neural} (see a proof in Appendix \ref{proof:lem:T2-sobolev}).

\begin{lemma}[Uniform deviation bound for the Sobolev error]
\label{lem:T2-sobolev}
For any \(\eta>0\),
\begin{align}
{\rm T}_2
\le
\frac{
C\bigl(R_{\bar m,m}^{(n_3)}\bigr)^2P^{2\ell/d}
}{n_f}
\log
\mathcal N_{\rm emp}
\left(
\frac{\eta}{C R_{\bar m,m}^{(n_3)}P^{\ell/d}},
\mathcal F_G
\right)
+
C\eta .
\label{eq:T2-sobolev-bound}
\end{align}Here $C>0$ is independent of $n_f, \bar m, m$ and $\eta$.
\end{lemma}

Using \(n_gn_z\ge1\), we obtain
\begin{align}
\mathcal E_{\rm gen}
\le\;&
C
\left(
\frac{\log H}{\log\log H}
\right)^{
-\frac{2}{\max\{\frac d{n_1},\frac d{n_2}\}}
}
(\log\log H)^{4d}
+
C\eta
\notag\\
&+
C\sigma\eta
\sqrt{
\frac{
\log
\mathcal N_{\rm emp}(\eta,\mathcal F_G)
+1
}{
n_fn_gn_z
}
}
+
\frac{
C\left(1+\bigl(R_{\bar m,m}^{(n_3)}\bigr)^2P^{2\ell/d}\right)
}{n_f}
\log
\mathcal N_{\rm emp}
\left(
\frac{\eta}{C R_{\bar m,m}^{(n_3)}P^{\ell/d}},
\mathcal F_G
\right).
\label{eq:sobolev-total-before-covering}
\end{align}
Indeed, since \(n_gn_z\ge1\), the term of order
\((n_fn_gn_z)^{-1}\log\mathcal N_{\rm emp}(\eta,\mathcal F_G)\) is dominated by
the outer-sampling term of order
\(n_f^{-1}\log
\mathcal N_{\rm emp}(\eta/(C R_{\bar m,m}^{(n_3)}P^{\ell/d}),\mathcal F_G)\),
after increasing the constant.

\medskip
\noindent\textbf{Covering number estimate.}
We next estimate the uniform covering number of \(\mathcal F_G\).

\begin{remark}
In the next lemma, we apply the empirical covering number estimate from
\cite{anthony1999neural}. That result is stated for real-valued function
classes. Although the elements of \(\mathcal F_G\) are neural operators, in the
present finite-dimensional setting each operator can be viewed as a real-valued
function of the variables \((\mathcal D_{\bar m}f,\mathcal D_mg,\vx)\). Indeed,
after discretizing the input functions by \(\mathcal D_{\bar m}f\) and
\(\mathcal D_mg\), each \(\mathcal H\in\mathcal F_G\) induces the real-valued map
\((f,g,\vx)\mapsto
\mathcal H(\mathcal D_{\bar m}f,\mathcal D_mg)(\vx)\) on a finite-dimensional
input space. Therefore, the standard empirical covering number estimate for
real-valued function classes can be applied to this induced class.

The covering estimate also involves the pseudo-dimension of the corresponding
function class. To keep the presentation concise, we do not recall the
definition of pseudo-dimension here. Instead, we use existing pseudo-dimension
estimates for neural networks and their derivative classes. More details can be
found in
\cite{abu1989vapnik,bartlett1998almost,anthony1999neural,yang2022approximation,yang2025deeponet}.
\end{remark}

\begin{definition}[Uniform empirical covering number]
\label{def:uniform-emp-cover}
Let \(\mathcal F\) be a class of real-valued functions on a set
\(\mathcal A\). For any \(n\in\mathbb N\) and any sample
\(\mathcal A_n=\{a_1,\ldots,a_n\}\subset \mathcal A\), define the empirical
\(L^\infty\)-metric by
\[
d_{\mathcal A_n}(f,g)
:=
\max_{1\le i\le n}|f(a_i)-g(a_i)|,
\qquad f,g\in\mathcal F.
\]
Let
\(\mathcal N(\varepsilon,\mathcal F,d_{\mathcal A_n})\) denote the covering
number of \(\mathcal F\) with respect to \(d_{\mathcal A_n}\). The uniform
empirical covering number of \(\mathcal F\) over all \(n\)-point samples is
defined by
\[
\mathcal N(\varepsilon,\mathcal F,n)
:=
\sup_{\mathcal A_n\subset\mathcal A:\,|\mathcal A_n|=n}
\mathcal N
\left(
\varepsilon,
\mathcal F,
d_{\mathcal A_n}
\right).
\]
\end{definition}

\begin{lemma}[{\cite[Theorem~12.2]{anthony1999neural}}]
\label{lem:cover-pdim}
Let \(\mathcal F\) be a class of real-valued functions on a set
\(\mathcal A\). Assume that \(|f(a)|\le B\) for all
\(f\in\mathcal F\) and \(a\in\mathcal A\). Let
\(V:=\operatorname{Pdim}(\mathcal F)\) be the pseudo-dimension of
\(\mathcal F\). Then, for any \(\varepsilon>0\) and any \(n\ge V\), the uniform
empirical covering number in Definition~\ref{def:uniform-emp-cover} satisfies
\[
\mathcal N(\varepsilon,\mathcal F,n)
\le
\left(
\frac{2enB}{\varepsilon V}
\right)^V .
\]
\end{lemma}

For each multi-index \(\gamma\) with \(|\gamma|\le \ell\), define the scalar
derivative class
\[
\mathcal F_G^\gamma
:=
\left\{
(f,g,\vx)\mapsto
\partial^\gamma
\mathcal H(\mathcal D_{\bar m}f,\mathcal D_mg)(\vx)
:
\mathcal H\in\mathcal F_G
\right\}.
\]
Let \(q_\ell:=\#\{\gamma:|\gamma|\le \ell\}\). Since
\(\ell\in\{0,1\}\), we have \(q_0=1\) and \(q_1=d+1\).

Fix a finite sample set
\(\mathcal A=\{(f_a,g_a,\vx_a)\}_{a=1}^{N_{\mathcal A}}\). For each
\(|\gamma|\le\ell\), let \(\mathcal C_\gamma\) be an
\(\eta/(2q_\ell)\)-cover of \(\mathcal F_G^\gamma\) with respect to the
empirical \(L^\infty\)-metric
\[
d_{\mathcal A,\infty}(F_1,F_2)
:=
\max_{1\le a\le N_{\mathcal A}}
|F_1(f_a,g_a,\vx_a)-F_2(f_a,g_a,\vx_a)|.
\]

We now construct an \(\eta\)-cover of \(\mathcal F_G\) with respect to
\(d_{\mathcal A,\ell}\). For each \(\mathcal H\in\mathcal F_G\) and each
\(|\gamma|\le\ell\), choose \(F_{\gamma,\mathcal H}\in\mathcal C_\gamma\) such
that
\(d_{\mathcal A,\infty}(\partial^\gamma \mathcal H,F_{\gamma,\mathcal H})
\le \eta/(2q_\ell)\). This assigns to \(\mathcal H\) a tuple
\((F_{\gamma,\mathcal H})_{|\gamma|\le\ell}\in
\prod_{|\gamma|\le\ell}\mathcal C_\gamma\). The tuples induce a partition of
\(\mathcal F_G\). For each nonempty cell of this partition, choose one
representative \(\mathcal H_\eta\in\mathcal F_G\). The collection of all such
representatives is denoted by \(\mathcal C\). Clearly,
\[
|\mathcal C|
\le
\prod_{|\gamma|\le\ell}|\mathcal C_\gamma|.
\]
We claim that \(\mathcal C\) is an \(\eta\)-cover of \(\mathcal F_G\) with
respect to \(d_{\mathcal A,\ell}\). Indeed, let \(\mathcal H\in\mathcal F_G\),
and let \(\mathcal H_\eta\in\mathcal C\) be the representative chosen from the
same cell as \(\mathcal H\). Then, for each \(|\gamma|\le\ell\),
\[
d_{\mathcal A,\infty}
\left(
\partial^\gamma \mathcal H,
\partial^\gamma \mathcal H_\eta
\right)
\le
d_{\mathcal A,\infty}
\left(
\partial^\gamma \mathcal H,
F_{\gamma,\mathcal H}
\right)
+
d_{\mathcal A,\infty}
\left(
F_{\gamma,\mathcal H},
\partial^\gamma \mathcal H_\eta
\right)
\le
\frac{\eta}{q_\ell}.
\]
Therefore, for every \(a=1,\ldots,N_{\mathcal A}\),
\[
\sum_{|\gamma|\le\ell}
\left|
\partial^\gamma
(\mathcal H-\mathcal H_\eta)
(\mathcal D_{\bar m}f_a,\mathcal D_mg_a)(\vx_a)
\right|
\le
\eta .
\]
Hence \(d_{\mathcal A,\ell}(\mathcal H,\mathcal H_\eta)\le\eta\).
Consequently,
\begin{align}
\mathcal N
\left(
\eta,\mathcal F_G,d_{\mathcal A,\ell}
\right)
\le
\prod_{|\gamma|\le\ell}
\mathcal N
\left(
\frac{\eta}{2q_\ell},
\mathcal F_G^\gamma,
N_{\mathcal A}
\right).
\label{eq:cover-split-derivatives}
\end{align}
Taking the supremum over all sample sets \(\mathcal A\) with
\(N_{\mathcal A}\le 2n_fn_gn_z\), we obtain
\begin{align}
\mathcal N_{\rm emp}(\eta,\mathcal F_G)
\le
\prod_{|\gamma|\le\ell}
\mathcal N
\left(
\frac{\eta}{2q_\ell},
\mathcal F_G^\gamma,
2n_fn_gn_z
\right).
\label{eq:emp-cover-split-derivatives}
\end{align}
Since \(q_\ell\) depends only on \(d\) and \(\ell\), the factor \(2q_\ell\) can
be absorbed into the constants in the logarithmic covering estimates.

\begin{lemma}[Pseudo-dimension bound for the derivative classes]
\label{lem:pdim-derivative-class}
Let \(\mathcal F_G^\gamma\) be defined as above. For any multi-index
\(\gamma\) with \(|\gamma|\le 1\), we have
\[
\operatorname{Pdim}(\mathcal F_G^\gamma)
\le
C JHP\left(
L_B^2p_B^2
+
L_E^2p_E^2
+
L_T^2p_T^2
\right),
\]
where \(C>0\) is independent of the weight magnitudes of the neural networks.
\end{lemma}
Lemma \ref{lem:pdim-derivative-class} is proved in Appendix \ref{proof:lem:pdim-derivative-class}.

\noindent\textbf{Covering number estimate and choice of parameters.}
We first estimate the uniform empirical covering number under the balanced
choice of the architecture parameters. By \eqref{eq:emp-cover-split-derivatives}
and Lemma~\ref{lem:cover-pdim}, for any \(\eta>0\),
\begin{align}
\log
\mathcal N_{\rm emp}(\eta,\mathcal F_G)
&\le
C
\sum_{|\gamma|\le\ell}
\operatorname{Pdim}(\mathcal F_G^\gamma)
\log
\left(
\frac{
C n_fn_gn_z B_\ell
}{
\eta
}
\right),
\label{eq:cover-choice-unscaled}
\end{align}
where \(B_\ell:=C R_{\bar m,m}^{(n_3)}P^{\ell/d}\). Similarly,
\begin{align}
&\log
\mathcal N_{\rm emp}
\left(
\frac{\eta}{C R_{\bar m,m}^{(n_3)}P^{\ell/d}},
\mathcal F_G
\right)
\le
C V_\ell
\log
\left(
\frac{
C n_fn_gn_z
\bigl(R_{\bar m,m}^{(n_3)}P^{\ell/d}\bigr)^2
}{
\eta
}
\right),
\label{eq:cover-choice-scaled}
\end{align}
where \(V_\ell:=
\sum_{|\gamma|\le \ell}
\operatorname{Pdim}(\mathcal F_G^\gamma)\). By
Lemma~\ref{lem:pdim-derivative-class},
\[
V_\ell
\le
C JHP
\left(
L_B^2p_B^2
+
L_E^2p_E^2
+
L_T^2p_T^2
\right).
\]

We now choose the architecture parameters. The discretization levels
\(m_*,\bar m_*\) and the branch ranks \(J,H\) are chosen according to the
balanced construction in Corollary~\ref{cor:balanced-main-final-functionL}.
The trunk rank \(P\) is chosen so that the trunk approximation error is of the
same order as the branch approximation error. By the balancing argument in
Corollary~\ref{cor:balanced-main-final-functionL}, this gives
\(P=(\log H)^{C+o(1)}=H^{o(1)}\) for some constant \(C>0\). Hence every
polynomial factor involving \(P\) is of order \(H^{o(1)}\).

With these choices, the squared Sobolev approximation error is bounded by
\[
C
\left(
\frac{\log H}{\log\log H}
\right)^{
-\frac{2}{\max\{\frac d{n_1},\frac d{n_2}\}}
}
(\log\log H)^{4d}.
\]
Moreover, up to logarithmic factors,
\(m_*
\asymp
(\log H/\log\log H)^{d/n_2}\) and
\(\bar m_*
\asymp
(\log H/\log\log H)^{d/n_1}\), while
\(\log J\asymp m_*\log m_*\) and
\(\log H\asymp \bar m_*\log \bar m_*\). Therefore,
\(J=H^{o(1)}\) and \(P=H^{o(1)}\). Since
\(p_B,p_E,p_T=\mathcal O(1)\), and \(L_B,L_E,L_T\) grow at most polynomially in
\(\log H\), we obtain
\[
V_\ell
\le
C JHP
\left(
L_B^2p_B^2
+
L_E^2p_E^2
+
L_T^2p_T^2
\right)
\le
H^{1+o(1)}.
\]
Furthermore, \(R_{\bar m,m}^{(n_3)}
=
\mathcal O((\log\log H)^{2d})\), and hence
\(R_{\bar m,m}^{(n_3)}P^{\ell/d}=H^{o(1)}\).

Taking \(\eta=H^{-1}\) in \eqref{eq:cover-choice-scaled}, we obtain
\begin{align}
&\log
\mathcal N_{\rm emp}
\left(
\frac{\eta}{C R_{\bar m,m}^{(n_3)}P^{\ell/d}},
\mathcal F_G
\right)
\le
C H^{1+o(1)}
\log(n_fn_gn_zH).
\label{eq:cover-scaled-H}
\end{align}
Similarly,
\begin{align}
\log
\mathcal N_{\rm emp}(\eta,\mathcal F_G)
\le
C H^{1+o(1)}
\log(n_fn_gn_zH).
\label{eq:cover-unscaled-H}
\end{align}

Substituting \eqref{eq:cover-scaled-H} and \eqref{eq:cover-unscaled-H} into
\eqref{eq:sobolev-total-before-covering}, and using \(\eta=H^{-1}\), gives
\begin{align}
\mathcal E_{\rm gen}
\le\;&
C
\left(
\frac{\log H}{\log\log H}
\right)^{
-\frac{2}{\max\{\frac d{n_1},\frac d{n_2}\}}
}
(\log\log H)^{4d}
+
CH^{-1}
\notag\\
&+
C\sigma H^{-1}
\sqrt{
\frac{
H^{1+o(1)}
\log(n_fn_gn_zH)
}{
n_fn_gn_z
}
}
+
C
\frac{
H^{1+o(1)}
\log(n_fn_gn_zH)
}{n_f}.
\label{eq:sobolev-rate-before-H-choice}
\end{align}
Here we used
\(1+\bigl(R_{\bar m,m}^{(n_3)}\bigr)^2P^{2\ell/d}=H^{o(1)}\).

We now choose \(H=n_f^c\) with a sufficiently small constant \(c>0\). More
precisely, assume \(\log(n_gn_z)\le n_f^\delta\) for some sufficiently small
\(\delta>0\), and choose \(c>0\) such that \(c+\delta<1\). Then the stochastic
terms in \eqref{eq:sobolev-rate-before-H-choice} are dominated by the
approximation term for sufficiently large \(n_f\). Hence
\[
\mathcal E_{\rm gen}
\le
C
\left(
\frac{\log H}{\log\log H}
\right)^{
-\frac{2}{\max\{\frac d{n_1},\frac d{n_2}\}}
}
(\log\log H)^{4d}.
\]
Finally, since \(H=n_f^c\), we have
\(\log H\asymp \log n_f\) and
\(\log\log H\asymp \log\log n_f\). Therefore,
\[
\mathcal E_{\rm gen}
\le
C
\left(
\frac{\log n_f}{\log\log n_f}
\right)^{
-\frac{2}{\max\{\frac d{n_1},\frac d{n_2}\}}
}
(\log\log n_f)^{4d}.
\]
Equivalently,
\[
\mathcal E_{\rm gen}
\le
C
\left(
\frac{\log n_f}{\log\log n_f}
\right)^{
-\frac{2\min\{n_1,n_2\}}{d}
}
(\log\log n_f)^{4d}.
\]
This completes the proof.
\end{proof}
\section{Conclusion}

In this paper, we studied scaling laws for multi-input neural operator learning
and quantified how the complexity of each input space affects the final
approximation and generalization rates. Our results show that the dimension,
regularity, and Lipschitz sensitivity associated with each input function enter
explicitly into the final error bound. In particular, in the balanced regime,
the rate is governed by a harmonic-type interaction among the input dimensions
and regularities. This provides a theoretical explanation of how different
inputs contribute to the overall learning complexity and also gives guidance
for allocating the network capacity across different branch subnetworks.

There are several directions for future work. First, the architecture studied in
this paper uses a construction in which the trunk network is introduced at the
final stage. This is different from the standard DeepONet
architecture~\cite{lu2019deeponet}, where the trunk network is constructed at
the beginning and then combined with the branch networks. If one instead
constructs the trunk network at an earlier stage, the approximation procedure
and the corresponding error analysis would be substantially different.
Comparing these two types of constructions may help identify which architecture
is more suitable for different classes of multi-input operator-learning
problems.

Second, the discretization procedure in this paper is based on global Chebyshev
interpolation. Due to the Markov--Bernstein inequality, the Sobolev stability of
this reconstruction introduces a loss of order \(2k_i\), leading to the effective
smoothness \(n_i-2k_i\) in the final rate. Ideally, this loss could be reduced
from \(2k_i\) to \(k_i\) by using more stable local sampling and reconstruction
procedures, such as B-spline quasi-interpolation. For instance, one-dimensional
spline approximation estimates of this type can be found in
\cite[Theorem~6.20]{schumaker2007spline}. Developing a multi-dimensional
sample-based reconstruction scheme with improved Sobolev stability would lead
to sharper rates for multi-input operator learning.

Third, the generalization analysis in this work is carried out under a paired
hierarchical sampling setting. In this setting, for each outer sample \(f_i\),
one observes a collection of inner samples \(g_{i,j}\) and spatial points
\(\vz_{i,j,k}\). As a consequence, the dominant statistical error is mainly
controlled by the number of outer samples \(n_f\), rather than by \(n_g\) alone.
This reflects the fact that the samples sharing the same \(f_i\) are not fully
independent at the outer level. In many practical scientific computing tasks,
however, one may sample many input functions \(f\) and \(g\) simultaneously and
form different pairings between them. In such settings, the effective sample
size and the resulting generalization error may depend on the joint sampling
and pairing structure of the inputs. Understanding generalization error for
multi-input neural operators under such non-nested or partially paired sampling
schemes is an important direction for future work. Understanding generalization error for multi-input neural operators under such
non-nested or partially paired sampling schemes is an important direction for
future work \cite{pinelis1994optimum,caponnetto2007optimal}.

\section*{Acknowledgement}
W.Z.\ acknowledges support from the National Science Foundation under awards DMS-2502900 and DMS-2540370, and from the Air Force Office of Scientific Research under Grant No. FA9550-25-1-0079. W.L.\ acknowledges support from the National Science Foundation under the NSF DMS 2145167 and the U.S. Department of Energy under the DOE SC0024348. H.L. \ acknowledges support from HKRGC ECS 22302123, HKRGC GRF 12301925 and Guangdong and Hong Kong Universities ``1+1+1'' Joint Research Collaboration Scheme 2025A0505000007.
Z.Z.\ acknowledges support from the U.S. Department of Energy under the DE-SC0025440.

\bibliographystyle{abbrv}
\bibliography{sample_author_initials}

\begin{thebibliography}{10}

\bibitem{abu1989vapnik}
Y.~Abu-Mostafa.
\newblock The {Vapnik-Chervonenkis} dimension: Information versus complexity in learning.
\newblock {\em Neural Computation}, 1(3):312--317, 1989.

\bibitem{anthony1999neural}
M.~Anthony, P.~Bartlett, et~al.
\newblock {\em Neural network learning: Theoretical foundations}, volume~9.
\newblock cambridge university press Cambridge, 1999.

\bibitem{back2002universal}
A.~D. Back and T.~Chen.
\newblock Universal approximation of multiple nonlinear operators by neural networks.
\newblock {\em Neural Computation}, 14(11):2561--2566, 2002.

\bibitem{bagby2002multivariate}
T.~Bagby, L.~Bos, and N.~Levenberg.
\newblock Multivariate simultaneous approximation.
\newblock {\em Constructive approximation}, 18(4):569--577, 2002.

\bibitem{bartlett1998almost}
P.~Bartlett, V.~Maiorov, and R.~Meir.
\newblock Almost linear {VC} dimension bounds for piecewise polynomial networks.
\newblock {\em Advances in neural information processing systems}, 11, 1998.

\bibitem{brenner2008mathematical}
S.~Brenner, L.~Scott, and L.~Scott.
\newblock {\em The mathematical theory of finite element methods}, volume~3.
\newblock Springer, 2008.

\bibitem{caponnetto2007optimal}
A.~Caponnetto and E.~De~Vito.
\newblock Optimal rates for the regularized least-squares algorithm.
\newblock {\em Foundations of Computational mathematics}, 7(3):331--368, 2007.

\bibitem{chen2022nonparametric}
M.~Chen, H.~Jiang, W.~Liao, and T.~Zhao.
\newblock Nonparametric regression on low-dimensional manifolds using deep {ReLU} networks: Function approximation and statistical recovery.
\newblock {\em Information and Inference: A Journal of the IMA}, 11(4):1203--1253, 2022.

\bibitem{chen1993approximations}
T.~Chen and H.~Chen.
\newblock Approximations of continuous functionals by neural networks with application to dynamic systems.
\newblock {\em IEEE Transactions on Neural networks}, 4(6):910--918, 1993.

\bibitem{chen1995universal}
T.~Chen and H.~Chen.
\newblock Universal approximation to nonlinear operators by neural networks with arbitrary activation functions and its application to dynamical systems.
\newblock {\em IEEE transactions on neural networks}, 6(4):911--917, 1995.

\bibitem{czarnecki2017sobolev}
W.~Czarnecki, S.~Osindero, M.~Jaderberg, G.~Swirszcz, and R.~Pascanu.
\newblock Sobolev training for neural networks.
\newblock {\em Advances in neural information processing systems}, 30, 2017.

\bibitem{devore1993constructive}
R.~A. DeVore and G.~G. Lorentz.
\newblock {\em Constructive approximation}, volume 303.
\newblock Springer Science \& Business Media, 1993.

\bibitem{dong2022w}
H.~Dong and Z.~Li.
\newblock On the {$W^{2,p}$} estimate for oblique derivative problem in lipschitz domains.
\newblock {\em International Mathematics Research Notices}, 2022(5):3602--3635, 2022.

\bibitem{gilbarg1998elliptic}
D.~Gilbarg, N.~S. Trudinger, D.~Gilbarg, and N.~Trudinger.
\newblock {\em Elliptic partial differential equations of second order}, volume~2.
\newblock Springer, 1998.

\bibitem{goswami2022physics}
S.~Goswami, M.~Yin, Y.~Yu, and G.~Karniadakis.
\newblock A physics-informed variational {DeepONet} for predicting crack path in quasi-brittle materials.
\newblock {\em Computer Methods in Applied Mechanics and Engineering}, 391:114587, 2022.

\bibitem{guhring2020error}
I.~G{\"u}hring, G.~Kutyniok, and P.~Petersen.
\newblock Error bounds for approximations with deep {{ReLU}} neural networks in ${W}^{s, p}$ norms.
\newblock {\em Analysis and Applications}, 18(05):803--859, 2020.

\bibitem{guhring2021approximation}
I.~G{\"u}hring and M.~Raslan.
\newblock Approximation rates for neural networks with encodable weights in smoothness spaces.
\newblock {\em Neural Networks}, 134:107--130, 2021.

\bibitem{hao2026multiscale}
W.~Hao, R.~P. Li, Y.~Xi, T.~Xu, and Y.~Yang.
\newblock Multiscale neural networks for approximating green’s functions.
\newblock {\em SIAM Journal on Scientific Computing}, 48(2):C240--C270, 2026.

\bibitem{he2024mgno}
J.~He, X.~Liu, and J.~Xu.
\newblock Mgno: Efficient parameterization of linear operators via multigrid.
\newblock In {\em International Conference on Learning Representations}, volume 2024, pages 53409--53428, 2024.

\bibitem{hill2026geometric}
S.~Hill and F.~X.-F. Ye.
\newblock Geometric regularization of autoencoders via observed stochastic dynamics.
\newblock {\em arXiv preprint arXiv:2604.16282}, 2026.

\bibitem{hu2025hybrid}
J.~Hu and P.~Jin.
\newblock A hybrid iterative method based on mionet for pdes: Theory and numerical examples.
\newblock {\em Mathematics of Computation}, 2025.

\bibitem{jin2022mionet}
P.~Jin, S.~Meng, and L.~Lu.
\newblock Mionet: Learning multiple-input operators via tensor product.
\newblock {\em SIAM Journal on Scientific Computing}, 44(6):A3490--A3514, 2022.

\bibitem{kovachki2021universal}
N.~Kovachki, S.~Lanthaler, and S.~Mishra.
\newblock On universal approximation and error bounds for fourier neural operators.
\newblock {\em Journal of Machine Learning Research}, 22(290):1--76, 2021.

\bibitem{lanthaler2023operator}
S.~Lanthaler.
\newblock Operator learning with pca-net: upper and lower complexity bounds.
\newblock {\em Journal of Machine Learning Research}, 24(318):1--67, 2023.

\bibitem{lanthaler2022error}
S.~Lanthaler, S.~Mishra, and G.~Karniadakis.
\newblock Error estimates for {DeepONet}s: A deep learning framework in infinite dimensions.
\newblock {\em Transactions of Mathematics and Its Applications}, 6(1):tnac001, 2022.

\bibitem{li2026sparse}
J.~Li, S.~Huang, H.~Feng, D.-X. Zhou, and G.~Kutyniok.
\newblock Sparse-aware neural networks for nonlinear functionals: Mitigating the exponential dependence on dimension.
\newblock {\em arXiv preprint arXiv:2604.06774}, 2026.

\bibitem{li2023fourier}
Z.~Li, D.~Z. Huang, B.~Liu, and A.~Anandkumar.
\newblock Fourier neural operator with learned deformations for pdes on general geometries.
\newblock {\em Journal of Machine Learning Research}, 24(388):1--26, 2023.

\bibitem{li2020fourier}
Z.~Li, N.~Kovachki, K.~Azizzadenesheli, B.~Liu, K.~Bhattacharya, A.~Stuart, and A.~Anandkumar.
\newblock Fourier neural operator for parametric partial differential equations.
\newblock {\em arXiv preprint arXiv:2010.08895}, 2020.

\bibitem{liu2025generalization}
H.~Liu, B.~Dahal, R.~Lai, and W.~Liao.
\newblock Generalization error guaranteed auto-encoder-based nonlinear model reduction for operator learning.
\newblock {\em Applied and Computational Harmonic Analysis}, 74:101717, 2025.

\bibitem{liu2022deep}
H.~Liu, H.~Yang, M.~Chen, T.~Zhao, and W.~Liao.
\newblock Deep nonparametric estimation of operators between infinite dimensional spaces.
\newblock {\em Journal of Machine Learning Research}, 25(24):1--67, 2024.

\bibitem{liu2024neural}
H.~Liu, Z.~Zhang, W.~Liao, and H.~Schaeffer.
\newblock Neural scaling laws of deep {ReLU} and deep operator network: A theoretical study.
\newblock {\em arXiv preprint arXiv:2410.00357}, 2024.

\bibitem{lu2019deeponet}
L.~Lu, P.~Jin, and G.~E. Karniadakis.
\newblock Deeponet: Learning nonlinear operators for identifying differential equations based on the universal approximation theorem of operators.
\newblock {\em arXiv preprint arXiv:1910.03193}, 2019.

\bibitem{marcati2023exponential}
C.~Marcati and C.~Schwab.
\newblock Exponential convergence of deep operator networks for elliptic partial differential equations.
\newblock {\em SIAM Journal on Numerical Analysis}, 61(3):1513--1545, 2023.

\bibitem{mhaskar1997neural}
H.~N. Mhaskar and N.~Hahm.
\newblock Neural networks for functional approximation and system identification.
\newblock {\em Neural Computation}, 9(1):143--159, 1997.

\bibitem{opschoor2020deep}
J.~A.~A. Opschoor, P.~C. Petersen, and C.~Schwab.
\newblock Deep {ReLU} networks and high-order finite element methods.
\newblock {\em Analysis and Applications}, 18(05):715--770, 2020.

\bibitem{pinelis1994optimum}
I.~Pinelis.
\newblock Optimum bounds for the distributions of martingales in banach spaces.
\newblock {\em The Annals of Probability}, pages 1679--1706, 1994.

\bibitem{schumaker2007spline}
L.~Schumaker.
\newblock {\em Spline functions: basic theory}.
\newblock Cambridge university press, 2007.

\bibitem{schwab2026deep}
C.~Schwab, A.~Stein, and J.~Zech.
\newblock Deep operator network approximation rates for lipschitz operators.
\newblock {\em Analysis and Applications}, 24(01):199--239, 2026.

\bibitem{shi2025nonlinear}
Z.~Shi, J.~Fan, L.~Song, D.-X. Zhou, and J.~A. Suykens.
\newblock Nonlinear functional regression by functional deep neural network with kernel embedding.
\newblock {\em Journal of Machine Learning Research}, 26(284):1--49, 2025.

\bibitem{siegel2023optimal}
J.~W. Siegel.
\newblock Optimal approximation rates for deep {ReLU} neural networks on sobolev and besov spaces.
\newblock {\em Journal of Machine Learning Research}, 24(357):1--52, 2023.

\bibitem{song2023approximation}
L.~Song, Y.~Liu, J.~Fan, and D.~Zhou.
\newblock Approximation of smooth functionals using deep {ReLU} networks.
\newblock {\em Neural Networks}, 166:424--436, 2023.

\bibitem{srinivas2018knowledge}
S.~Srinivas and F.~Fleuret.
\newblock Knowledge transfer with jacobian matching.
\newblock In {\em International conference on machine learning}, pages 4723--4731. PMLR, 2018.

\bibitem{vlassis2021sobolev}
N.~Vlassis and W.~Sun.
\newblock Sobolev training of thermodynamic-informed neural networks for interpretable elasto-plasticity models with level set hardening.
\newblock {\em Computer Methods in Applied Mechanics and Engineering}, 377:113695, 2021.

\bibitem{vlassis2020geometric}
N.~N. Vlassis, R.~Ma, and W.~Sun.
\newblock Geometric deep learning for computational mechanics part i: Anisotropic hyperelasticity.
\newblock {\em Computer Methods in Applied Mechanics and Engineering}, 371:113299, 2020.

\bibitem{wang2021learning}
S.~Wang, H.~Wang, and P.~Perdikaris.
\newblock Learning the solution operator of parametric partial differential equations with physics-informed {DeepONet}s.
\newblock {\em Science advances}, 7(40):eabi8605, 2021.

\bibitem{weihs2026generalization}
A.~Weihs and H.~Schaeffer.
\newblock Generalization bounds and statistical guarantees for multi-task and multiple operator learning with mno networks.
\newblock {\em arXiv preprint arXiv:2604.01961}, 2026.

\bibitem{weihs2026multiple}
A.~Weihs and H.~Schaeffer.
\newblock Multiple neural operators achieve near-optimal rates for multi-task learning.
\newblock {\em arXiv preprint arXiv:2605.22724}, 2026.

\bibitem{weihs2025deep}
A.~Weihs, J.~Sun, Z.~Zhang, and H.~Schaeffer.
\newblock A deep learning framework for multi-operator learning: Architectures and approximation theory.
\newblock {\em arXiv preprint arXiv:2510.25379}, 2025.

\bibitem{yang2026efficient}
J.-Q. Yang and L.~Shi.
\newblock Efficient approximation for encoder--decoder neural operators via variation spaces.
\newblock {\em arXiv preprint arXiv:2606.01244}, 2026.

\bibitem{yang2025deeponet}
Y.~Yang.
\newblock {DeepONet} for solving nonlinear partial differential equations with physics-informed training.
\newblock {\em Neural Networks}, page 108490, 2025.

\bibitem{yang2025deep}
Y.~Yang and J.~He.
\newblock Deep neural networks with general activations: Super-convergence in sobolev norms.
\newblock {\em arXiv preprint arXiv:2508.05141}, 2025.

\bibitem{yang2023nearlys}
Y.~Yang, Y.~Wu, H.~Yang, and Y.~Xiang.
\newblock Nearly optimal approximation rates for deep super {{ReLU}} networks on {Sobolev} spaces.
\newblock {\em arXiv preprint arXiv:2310.10766}, 2023.

\bibitem{yang2022approximation}
Y.~Yang and Y.~Xiang.
\newblock Approximation of functionals by neural network without curse of dimensionality.
\newblock {\em Journal of Machine Learning}, 1(4):342--372, 2022.

\bibitem{yang2023nearly}
Y.~Yang, H.~Yang, and Y.~Xiang.
\newblock Nearly optimal {VC}-dimension and pseudo-dimension bounds for deep neural network derivatives.
\newblock In {\em Thirty-seventh Conference on Neural Information Processing Systems}, 2023.

\bibitem{yang2026spherical}
Z.~Yang, S.~Huang, H.~Feng, and D.-X. Zhou.
\newblock Spherical analysis of learning nonlinear functionals.
\newblock {\em Constructive Approximation}, pages 1--29, 2026.

\bibitem{yarotsky2017error}
D.~Yarotsky.
\newblock Error bounds for approximations with deep {{ReLU}} networks.
\newblock {\em Neural Networks}, 94:103--114, 2017.

\bibitem{yarotsky2020phase}
D.~Yarotsky and A.~Zhevnerchuk.
\newblock The phase diagram of approximation rates for deep neural networks.
\newblock {\em Advances in neural information processing systems}, 33:13005--13015, 2020.

\end{thebibliography}

\appendix
\section{Proofs for Approximation Error}

This section contains the detailed proof of the approximation error estimate.
The construction consists of four steps. Step~1, which concerns the
discretization of the input functions, has already been discussed in the main
text. We therefore focus here on the remaining steps, especially the coefficient
approximations in Steps~2 and~3 and the trunk-network approximation in Step~4.
We begin by recalling two preliminary results: the pseudo-spectral projection
estimate associated with the tensorized Chebyshev discretization, and a ReLU
approximation theorem for Sobolev functions in Sobolev norms.
\subsection{Step 0}
\subsubsection{Pseudo-spectral projection}

As discussed above, a key step in the construction is to discretize the input functions \(f\) and \(g\). In practical operator learning, the input functions are usually accessed only through finitely many pointwise samples. Therefore, although basis expansions provide a clean mathematical formulation, it is important to develop a discretization procedure based directly on sampling, in the same spirit as DeepONet~\cite{chen1993approximations,lu2019deeponet}. Pseudo-spectral projection based on spectral bases has been used in \cite{kovachki2021universal,yang2025deeponet}. However, Fourier-type spectral projections are naturally suited to periodic domains. For non-periodic problems, one typically needs to introduce extensions outside the physical domain, which is not always natural. On the other hand, the construction in \cite{liu2024neural} is designed for Lipschitz-continuous input spaces, while in this paper we work with Sobolev input classes. Therefore, we introduce a sampling-based discretization--reconstruction pair using Chebyshev interpolation.

For \(N\in\mathbb{N}\), define the tensorized Chebyshev grid by
\begin{align}
    \vx_{k_1,\dots,k_d}
    =
    \left(
    \cos\left(\frac{(k_1+\frac12)\pi}{N+1}\right),
    \dots,
    \cos\left(\frac{(k_d+\frac12)\pi}{N+1}\right)
    \right),
    \qquad
    0\le k_1,\dots,k_d\le N.
    \label{eq:cheb-grid}
\end{align}
The total number of grid points is \((N+1)^d\).

Let \(Q_N\) denote the space of polynomials on \(\mathbb{R}^d\) of total degree at most \(N\), and let \(\overline{Q}_N\) denote the space of polynomials on \(\mathbb{R}^d\) such that each variable has degree at most \(N\). Clearly,
\[
Q_N\subset \overline{Q}_N.
\]

We define the sampling operator
\[
\mathcal{D}_N(g)
=
\bigl(g(\vx_{k_1,\dots,k_d})\bigr)_{0\le k_1,\dots,k_d\le N}
\in \mathbb{R}^{(N+1)^d},
\]
and let
\[
\mathcal{P}_N:\mathbb{R}^{(N+1)^d}\to \overline{Q}_N
\]
be the multi-dimensional Chebyshev interpolation operator associated with the grid \eqref{eq:cheb-grid}. Then \(\mathcal{P}_N\circ \mathcal{D}_N g\) is the tensor-product Chebyshev interpolant of \(g\).

To derive the approximation estimate, we first recall several standard lemmas.

\begin{lemma}[Jackson's inequality~\cite{bagby2002multivariate}]
\label{lem:jackson}
Let \(\Omega = (-1,1)^d\subset\mathbb{R}^d\), and let \(g\in W^{m,\infty}(\Omega)\). Then for each integer \(N\), there exists a polynomial \(p_N^*\in Q_N\subset \overline{Q}_N\) such that for every integer \(0\le k\le \min\{m,N\}\),
\[
\|g-p_N^*\|_{W^{k,\infty}(\Omega)}
\le
\frac{C(d,m)}{N^{m-k}}\|g\|_{W^{m,\infty}(\Omega)},
\]
where \(C(d,m)\) depends only on \(d\) and \(m\).
\end{lemma}

\begin{lemma}[Markov--Bernstein inequality~\cite{devore1993constructive}]
\label{lem:markov}
For any polynomial \(p_N\in \overline{Q}_N\) and any multi-index \(\bm{\alpha}\), one has
\[
\|\partial^{\bm{\alpha}}p_N\|_{L^\infty(\Omega)}
\le
N^{2|\bm{\alpha}|}\|p_N\|_{L^\infty(\Omega)}.
\]
\end{lemma}

\begin{lemma}[Sobolev stability of the Chebyshev interpolation operator]
\label{lem:lebesgue}
Let \(k\in\mathbb N_0\). The interpolation operator
\(\mathcal{P}_N\circ\mathcal{D}_N:C(\Omega)\to \overline Q_N\) is linear and
satisfies
\[
\|\mathcal{P}_N\circ\mathcal{D}_N u\|_{W^{k,\infty}(\Omega)}
\le
C N^{2k}(\log N)^d
\|u\|_{L^\infty(\Omega)}
\]
for all \(u\in C(\Omega)\), where \(C>0\) depends only on \(d\) and \(k\).
\end{lemma}

\begin{proposition}
\label{prop:pseudo-spectral}
Let \(g\in W^{m,\infty}(\Omega)\). Then for every integer \(0\le k\le \min\{m,N\}\),
\begin{align}
    \|g-\mathcal{P}_N\circ\mathcal{D}_N g\|_{W^{k,\infty}(\Omega)}
    \le
    C(d,m)\,N^{2k-m}(\log N)^d
    \|g\|_{W^{m,\infty}(\Omega)}.
    \label{eq:pseudo-spectral-error}
\end{align}
Moreover, the reconstruction operator
\(
\mathcal{P}_N:
(\mathbb{R}^{(N+1)^d},\|\cdot\|_\infty)
\to
(W^{k,\infty}(\Omega),\|\cdot\|_{W^{k,\infty}(\Omega)})
\)
is Lipschitz continuous with
\begin{align}
    \|\mathcal{P}_N\|_{\ell^\infty\to W^{k,\infty}(\Omega)}
    \lesssim
    N^{2k}(\log N)^d.
    \label{eq:PN-Lipschitz}
\end{align}
\end{proposition}

\begin{proof}
Fix \(N\in\mathbb{N}\), and let \(p_N^*\in Q_N\) be given by Lemma~\ref{lem:jackson}. For any multi-index \(\bm{\alpha}\) with \(|\bm{\alpha}|=s\le k\), we write
\begin{align*}
    \partial^{\bm{\alpha}}\bigl(g-\mathcal{P}_N\circ\mathcal{D}_N g\bigr)
    =
    \partial^{\bm{\alpha}}(g-p_N^*)
    +
    \partial^{\bm{\alpha}}\bigl(p_N^*-\mathcal{P}_N\circ\mathcal{D}_N g\bigr).
\end{align*}
Since \(p_N^*\in \overline{Q}_N\) and \(\mathcal{P}_N\circ\mathcal{D}_N p_N^*=p_N^*\), we have
\[
p_N^*-\mathcal{P}_N\circ\mathcal{D}_N g
=
\mathcal{P}_N\circ\mathcal{D}_N(p_N^*-g).
\]
Hence,
\begin{align*}
    \|\partial^{\bm{\alpha}}(g-\mathcal{P}_N\circ\mathcal{D}_N g)\|_{L^\infty}
    \le
    \underbrace{\|\partial^{\bm{\alpha}}(g-p_N^*)\|_{L^\infty}}_{=:I_1}
    +
    \underbrace{\|\partial^{\bm{\alpha}}(\mathcal{P}_N\circ\mathcal{D}_N(p_N^*-g))\|_{L^\infty}}_{=:I_2}.
\end{align*}

By Lemma~\ref{lem:jackson},
\[
I_1
\le
\frac{C(d,m)}{N^{m-s}}\|g\|_{W^{m,\infty}(\Omega)}.
\]

For \(I_2\), Lemma~\ref{lem:markov} and Lemma~\ref{lem:lebesgue} imply
\begin{align*}
    I_2
    &\le
    N^{2s}\|\mathcal{P}_N\circ\mathcal{D}_N(p_N^*-g)\|_{L^\infty} \le
    C\,N^{2s}(\log N)^d\|p_N^*-g\|_{L^\infty}.
\end{align*}
Applying Lemma~\ref{lem:jackson} again with \(k=0\), we get
\[
\|p_N^*-g\|_{L^\infty}
\le
C(d,m)N^{-m}\|g\|_{W^{m,\infty}(\Omega)}.
\]
Therefore,
\[
I_2
\le
C(d,m)\,N^{2s-m}(\log N)^d\|g\|_{W^{m,\infty}(\Omega)}.
\]

Combining the above estimates and taking the maximum over all \(|\bm{\alpha}|\le k\), we obtain
\[
\|g-\mathcal{P}_N\circ\mathcal{D}_N g\|_{W^{k,\infty}(\Omega)}
\le
C(d,m)\,N^{2k-m}(\log N)^d\|g\|_{W^{m,\infty}(\Omega)},
\]
which proves \eqref{eq:pseudo-spectral-error}.

To prove \eqref{eq:PN-Lipschitz}, let \(\boldsymbol{f}\in \mathbb{R}^{(N+1)^d}\). For any \(|\bm{\alpha}|\le k\), Lemma~\ref{lem:markov} gives
\[
\|\partial^{\bm{\alpha}}(\mathcal{P}_N\boldsymbol{f})\|_{L^\infty}
\le
N^{2|\bm{\alpha}|}\|\mathcal{P}_N\boldsymbol{f}\|_{L^\infty}
\le
N^{2k}\|\mathcal{P}_N\boldsymbol{f}\|_{L^\infty}.
\]
For any given \(\boldsymbol{f}\in \mathbb{R}^{(N+1)^d}\), we can construct a piecewise linear function that interpolates $\boldsymbol{f}$ while satisfying $\|f\|_{L^{\infty}}=\|\boldsymbol{f}\|_{\infty}$. Using Lemma~\ref{lem:lebesgue}, we further obtain
\[
\|\mathcal{P}_N\boldsymbol{f}\|_{L^\infty}
\le C(\log N)^d\|f\|_{L^{\infty}}=
C(\log N)^d\|\boldsymbol{f}\|_\infty.
\]
Hence,
\[
\|\partial^{\bm{\alpha}}(\mathcal{P}_N\boldsymbol{f})\|_{L^\infty}
\le
C\,N^{2k}(\log N)^d\|\boldsymbol{f}\|_\infty.
\]
Taking the maximum over all \(|\bm{\alpha}|\le k\) yields \eqref{eq:PN-Lipschitz}.
\end{proof}

\subsubsection{Neural network approximation in Sobolev spaces}

As discussed above, the final step of our construction requires approximating
Sobolev-regular coefficient functions by neural networks. This problem has been
extensively studied in approximation theory and learning theory; see, for example,
\cite{yarotsky2017error,yarotsky2020phase,guhring2020error,siegel2023optimal,
opschoor2020deep,yang2023nearly,yang2025deep}. Our argument follows the
standard ReLU approximation framework of \cite{yarotsky2017error}. The only
additional point needed for our later generalization analysis is an explicit
bookkeeping of the network complexity. In particular, we need bounds on the
depth, width, sparsity, and the magnitude of the weights and biases. For this
reason, we recall and slightly rederive the construction with all relevant
network-size parameters made explicit.

\begin{lemma}[Approximate multiplication in \(W^{\ell,\infty}\)]
\label{lem.prodY-W1}
Let \(M>0\), \(0<\varepsilon<1\), and \(\ell\in\{0,1\}\). There exists a ReLU
network \(\widetilde{\times}_{M,\varepsilon}:\mathbb R^2\to\mathbb R\) such
that
\[
\left\|
\widetilde{\times}_{M,\varepsilon}(x,y)-xy
\right\|_{W^{\ell,\infty}((-M,M)^2)}
\le \varepsilon .
\]
Moreover,
\[
\widetilde{\times}_{M,\varepsilon}(0,y)=0
\quad\text{for all }y\in(-M,M),
\]
and similarly \(\widetilde{\times}_{M,\varepsilon}(x,0)=0\) for all
\(x\in(-M,M)\).
The network has width bounded by a universal constant, depth at most
\(C(M)\log(\varepsilon^{-1})\), and at most
\(C(M)\log(\varepsilon^{-1})\) nonzero parameters. In addition, all weights and
biases are bounded by \(C(M)\varepsilon^{-1}\).
\end{lemma}

\begin{proof}
For \(\ell=0\), this is the standard ReLU multiplication construction; see, for
example, Proposition~3 of \cite{yarotsky2017error}. The proof for \(\ell=1\) is
similar to the construction in \cite[Proposition~4]{yang2023nearly}. Since the
present argument requires explicit bounds on the network size and parameter
magnitudes, we include the details.

We first construct a ReLU approximation to the square function. For
\(\widetilde x\in[-1,1]\), define the tent map
\[
T_1(\widetilde x)
=
\begin{cases}
2|\widetilde x|, & |\widetilde x|\le \frac12,\\
2(1-|\widetilde x|), & \frac12<|\widetilde x|\le 1,
\end{cases}
\]
and set \(T_j=T_1\circ T_{j-1}\) for \(j\ge2\). For \(L\in\mathbb N_+\), define
\[
\widetilde\psi_L(\widetilde x)
=
\widetilde x
-
\sum_{j=1}^{L}\frac{T_j(\widetilde x)}{2^{2j}} .
\]
The function \(\widetilde\psi_L\) can be represented by a ReLU network with
width bounded by a universal constant and depth \(CL\). Moreover, the standard
sawtooth approximation estimate gives
\[
\left\|
\widetilde\psi_L-\widetilde x^2
\right\|_{W^{1,\infty}((-1,1))}
\le C2^{-L},
\qquad
\|\widetilde\psi_L\|_{W^{1,\infty}((-1,1))}\le C,
\]
and \(\widetilde\psi_L(0)=0\).

We rescale this approximation to \((-2M,2M)\) by setting
\(\psi_L(t):=(2M)^2\widetilde\psi_L(t/(2M))\). Then
\[
\left\|
\psi_L-t^2
\right\|_{W^{1,\infty}((-2M,2M))}
\le C(M)2^{-L},
\qquad
\psi_L(0)=0.
\]
Now define
\begin{equation}
\widetilde{\times}_{M,\varepsilon}(x,y)
=
\frac12
\left[
\psi_L(x+y)-\psi_L(x)-\psi_L(y)
\right].
\label{eq:mult-def}
\end{equation}
Since \(xy=\frac12((x+y)^2-x^2-y^2)\), we have
\[
\left\|
\widetilde{\times}_{M,\varepsilon}(x,y)-xy
\right\|_{W^{1,\infty}((-M,M)^2)}
\le C(M)2^{-L}.
\]
Choosing \(L\asymp\log(\varepsilon^{-1})\) gives the desired
\(W^{1,\infty}\)-estimate, and the \(L^\infty\)-estimate follows immediately.

The zero-preserving identities follow directly from the definition
\eqref{eq:mult-def} and \(\psi_L(0)=0\):
\[
\widetilde{\times}_{M,\varepsilon}(0,y)
=
\frac12[\psi_L(y)-\psi_L(0)-\psi_L(y)]
=0,
\]
and the identity at \(x=0\) is proved in the same way.
The stated bounds on width, depth, number of parameters, and weights follow
from the construction of \(\psi_L\) and the rescaling.
This completes the proof.
\end{proof}
\begin{proposition}[Local-basis ReLU approximation of Sobolev functions in
\(W^{\ell,\infty}\)]
\label{prop:relu-sobolev-local-basis-Well}
Let \(\Omega=[-1,1]^d\), \(n\in\mathbb N_+\), \(R>0\), and
\(\ell\in\{0,1\}\) with \(n>\ell\). Then there exists a constant
\(C=C(d,n,\ell)>0\) such that, for every \(P\ge2\), there exist ReLU networks
\(f_j:\mathbb R^d\to\mathbb R\), \(j=1,\ldots,P\), independent of \(f\), such
that for every \(f\in W^{n,\infty}(\Omega)\) with
\(\|f\|_{W^{n,\infty}(\Omega)}\le R\), there exist coefficients
\(a_j\in\mathbb R\), \(j=1,\ldots,P\), satisfying
\[
\left\|
f-\sum_{j=1}^{P}a_j f_j
\right\|_{W^{\ell,\infty}(\Omega)}
\le
C R P^{-(n-\ell)/d},
\qquad
|a_j|\le C R .
\]
Moreover, the basis networks can be chosen so that
\[
\sup_{\vx\in\Omega}
\sum_{j=1}^{P}|f_j(\vx)|
\le C,
\qquad
\sup_{\vx\in\Omega}
\sum_{j=1}^{P}\sum_{r=1}^{d}
|\partial_{x_r}f_j(\vx)|
\le
C P^{1/d}.
\]
Each \(f_j\) belongs to
\(\mathcal F_{\mathrm{NN}}
(d,1,L_{\rm loc},p_{\rm loc},K_{\rm loc},\kappa_{\rm loc},M_{\rm loc})\),
where
\[
L_{\rm loc}\le C\log P,\qquad
p_{\rm loc}\le C,\qquad
K_{\rm loc}\le C\log P,
\]
and
\(\kappa_{\rm loc}\le C P^{(n+d-\ell)/d}\), \(M_{\rm loc}\le C\).
\end{proposition}

\begin{proof}
We prove the result for \(\ell\in\{0,1\}\). The construction follows the local
average polynomial approximation argument in \cite{guhring2020error}.

Let \(N\in\mathbb N\) be chosen later and set
\(\mathcal I_N=\{0,1,\ldots,N\}^d\). For
\(\bm m=(m_1,\ldots,m_d)\in\mathcal I_N\), define
\(\vx_{\bm m}:=-{\bf 1}+2\bm m/N\in[-1,1]^d\). Let
\[
\psi(t)
=
\begin{cases}
1, & |t|<1,\\
2-|t|, & 1\le |t|\le 2,\\
0, & |t|>2,
\end{cases}
\]
and define the local cutoff function
\[
\phi_{\bm m}(\vx)
=
\prod_{r=1}^{d}
\psi\left(\frac{3N}{2}(x_r-x_{\bm m,r})\right).
\]
Then \(\phi_{\bm m}\) is supported in a cube of side length \(O(N^{-1})\), and
the family \(\{\phi_{\bm m}\}_{\bm m\in\mathcal I_N}\) has uniformly bounded
overlap. Moreover,
\[
\|\phi_{\bm m}\|_{L^\infty(\Omega)}\le C,
\qquad
\|\nabla\phi_{\bm m}\|_{L^\infty(\Omega)}\le C N .
\]

Using local average polynomials \cite{brenner2008mathematical} of order
\(n-1\), define
\[
f_N(\vx)
=
\sum_{\bm m\in\mathcal I_N}
\sum_{|\bm\alpha|<n}
a_{\bm m,\bm\alpha}
\phi_{\bm m}(\vx)
(\vx-\vx_{\bm m})^{\bm\alpha},
\]
where \(a_{\bm m,\bm\alpha}\) are the corresponding average polynomial
coefficients of \(f\) at \(\vx_{\bm m}\). Since
\(\|f\|_{W^{n,\infty}(\Omega)}\le R\), we have
\(|a_{\bm m,\bm\alpha}|\le C(d,n)R\). The standard local average polynomial
estimate, namely the Bramble--Hilbert lemma
\cite[Lemma~4.3.8]{brenner2008mathematical}, gives
\[
\|f-f_N\|_{W^{\ell,\infty}(\Omega)}
\le
C R N^{-(n-\ell)},
\qquad
\ell=0,1.
\]

For each pair \((\bm m,\bm\alpha)\), set
\(q_{\bm m,\bm\alpha}(\vx):=\phi_{\bm m}(\vx)(\vx-\vx_{\bm m})^{\bm\alpha}\).
The function \(q_{\bm m,\bm\alpha}\) is a product of at most \(d+n\) bounded
piecewise-linear or affine factors. The cutoff \(\psi\) is exactly
representable by a fixed-size ReLU network, and the affine factors are exactly
representable by affine layers. Applying Lemma~\ref{lem.prodY-W1} repeatedly,
for any \(0<\delta<1\), there exists a ReLU network
\(\widetilde q_{\bm m,\bm\alpha}\), independent of \(f\), such that
\[
\|q_{\bm m,\bm\alpha}
-
\widetilde q_{\bm m,\bm\alpha}\|_{W^{\ell,\infty}(\Omega)}
\le C\delta .
\]
Moreover,
\(\widetilde q_{\bm m,\bm\alpha}
\in
\mathcal F_{\mathrm{NN}}
(d,1,L_{\rm loc},p_{\rm loc},K_{\rm loc},
\kappa_{\rm loc},M_{\rm loc})\), where
\[
L_{\rm loc}\le C\log(\delta^{-1}),\qquad
p_{\rm loc}\le C,\qquad
K_{\rm loc}\le C\log(\delta^{-1}),
\]
and \(\kappa_{\rm loc}\le C\max\{N,\delta^{-1}\}\), \(M_{\rm loc}\le C\). Here
the factor \(N\) comes from the affine rescaling in the local cutoffs.

We now enumerate all pairs
\((\bm m,\bm\alpha)\in
\mathcal I_N\times\{\bm\alpha:|\bm\alpha|<n\}\) by
\(j=1,\ldots,P_N\). Since the number of multi-indices \(\bm\alpha\) with
\(|\bm\alpha|<n\) depends only on \(d\) and \(n\), we have \(P_N\le C N^d\).
Write \(a_j:=a_{\bm m,\bm\alpha}\) and
\(f_j:=\widetilde q_{\bm m,\bm\alpha}\). The networks \(f_j\) are constructed
only from the grid points, the cutoff function, and the monomial basis. Hence
they are independent of the target function \(f\), and the dependence on \(f\)
enters only through the coefficients \(a_j\).

Using \(|a_j|\le CR\) and \(P_N\le C N^d\), we obtain
\[
\left\|
f_N-\sum_{j=1}^{P_N}a_j f_j
\right\|_{W^{\ell,\infty}(\Omega)}
\le
C R N^d\delta .
\]
Choose \(\delta=N^{-(n-\ell)-d}\). Then
\[
\left\|
f-\sum_{j=1}^{P_N}a_j f_j
\right\|_{W^{\ell,\infty}(\Omega)}
\le
C R N^{-(n-\ell)}.
\]
With this choice of \(\delta\), each \(f_j\) satisfies
\(L_{\rm loc}\le C\log N\), \(p_{\rm loc}\le C\),
\(K_{\rm loc}\le C\log N\), \(\kappa_{\rm loc}\le C N^{n+d-\ell}\), and
\(M_{\rm loc}\le C\).

It remains to record the local-overlap bounds for the constructed basis. Since
\(\phi_{\bm m}\) is supported in a cube of side length \(O(N^{-1})\), the
support of \(q_{\bm m,\bm\alpha}\) is contained in the same local cube. The
multiplication networks in Lemma~\ref{lem.prodY-W1} may be chosen to preserve
the zero property when one factor is zero. Therefore,
\(\widetilde q_{\bm m,\bm\alpha}\), and hence \(f_j\), has the same local
support property, up to a fixed enlargement independent of \(N\).

For every \(\vx\in\Omega\), only \(C(d)\) cutoffs \(\phi_{\bm m}\) are nonzero.
Since the number of multi-indices \(\bm\alpha\) with \(|\bm\alpha|<n\) depends
only on \(d\) and \(n\), at each \(\vx\) only \(C(d,n)\) basis functions \(f_j\)
can be nonzero. Furthermore, \(\|f_j\|_{L^\infty(\Omega)}\le C\). Hence
\[
\sup_{\vx\in\Omega}
\sum_{j=1}^{P_N}|f_j(\vx)|
\le C .
\]

When \(\ell=1\), we also estimate the derivatives. Since
\(\|\nabla\phi_{\bm m}\|_{L^\infty(\Omega)}\le C N\), and the monomial factors
\((\vx-\vx_{\bm m})^{\bm\alpha}\) and their first derivatives are uniformly
bounded on \(\Omega\), with constants depending only on \(d\) and \(n\), we have
\[
\|\nabla q_{\bm m,\bm\alpha}\|_{L^\infty(\Omega)}
\le C N .
\]
Since
\(\|q_{\bm m,\bm\alpha}-\widetilde q_{\bm m,\bm\alpha}\|_{W^{1,\infty}(\Omega)}
\le C\delta\) and \(\delta<1\), after increasing the constant we also have
\(\|\nabla \widetilde q_{\bm m,\bm\alpha}\|_{L^\infty(\Omega)}\le C N\).
Using again the uniformly bounded overlap of the local supports, we obtain
\[
\sup_{\vx\in\Omega}
\sum_{j=1}^{P_N}
\sum_{r=1}^{d}
|\partial_{x_r}f_j(\vx)|
\le C N .
\]

Finally, choose \(N=\lfloor cP^{1/d}\rfloor\), where \(c=c(d,n)>0\) is
sufficiently small so that \(P_N\le P\). If \(P_N<P\), add \(P-P_N\) zero
networks and zero coefficients. This does not affect the approximation error or
the local-overlap bounds. Since \(N^{-(n-\ell)}\le C P^{-(n-\ell)/d}\),
\(\log N\le C\log P\), and \(N^{n+d-\ell}\le C P^{(n+d-\ell)/d}\), the claimed
approximation and network-parameter bounds follow. Moreover,
\[
\sup_{\vx\in\Omega}
\sum_{j=1}^{P}|f_j(\vx)|
\le C,
\]
and, when \(\ell=1\),
\[
\sup_{\vx\in\Omega}
\sum_{j=1}^{P}
\sum_{r=1}^{d}
|\partial_{x_r}f_j(\vx)|
\le
C N
\le
C P^{1/d}.
\]
This completes the proof.
\end{proof}

\subsection{Proofs in Step 1}
\label{proof:step1}
\begin{proof}[Proof of Proposition \ref{prop:step1-approximation}]

We use
Assumption~\ref{asspde} with \(k_i=0\) and \(\alpha_i=0\) together with the pseudo-spectral projection estimate
in Proposition~\ref{prop:pseudo-spectral}. Let
\[
f_{\bar m}:=\mathcal{P}_{\bar m}\mathcal{D}_{\bar m}f,
\qquad
g_m:=\mathcal{P}_{m}\mathcal{D}_{m}g,
\]
where \(m,\bar m\in\mathbb{N}_+\). Here, \(f_{\bar m}\) depends only on the
\(\bar m_*:=(1+\bar m)^d\) sampled values of \(f\), while \(g_m\) depends only on
the \(m_*:=(1+m)^d\) sampled values of \(g\).
By the separate Lipschitz continuity of \(\mathcal G\) with respect to each
input, we have
\begin{align}
&\|\mathcal G(f,g)-\mathcal G(f_{\bar m},g_m)\|_{W^{\ell,\infty}(\Omega)}
\le
\|\mathcal G(f,g)-\mathcal G(f,g_m)\|_{W^{\ell,\infty}(\Omega)}
+
\|\mathcal G(f,g_m)-\mathcal G(f_{\bar m},g_m)\|_{W^{\ell,\infty}(\Omega)}
\notag\\
\le&
L_2\|g-g_m\|_{L^\infty(\Omega)}
+
L_1\|f-f_{\bar m}\|_{L^\infty(\Omega)} .\notag
\end{align}
Applying Proposition~\ref{prop:pseudo-spectral} with \(k=0\), together with
Assumptions~\ref{asspde} and~\ref{assump:X-bound}, we obtain$\footnote{For simplicity of notation, the constant \(C\) may vary from line to line.}$
\begin{align}
&\|\mathcal G(f,g)-\mathcal G(f_{\bar m},g_m)\|_{W^{\ell,\infty}(\Omega)}
\notag\\
\le&
C U
\left[
L_1\,
\|f-f_{\bar m}\|_{L^\infty(\Omega)}
\left(1+\|g_m\|_{W^{k_2,\infty}(\Omega)}\right)^{\alpha_1}
+
L_2\,
\|g-g_m\|_{L^\infty(\Omega)}
\left(1+\|f\|_{W^{k_1,\infty}(\Omega)}\right)^{\alpha_2}
\right]
\notag\\
\le&
CU
\left[
L_1\,\bar m^{-n_1}(\log \bar m)^d
\left(1+\|g_m\|_{W^{k_2,\infty}(\Omega)}\right)^{\alpha_1}
+
L_2\,m^{-n_2}(\log m)^d
\right].\notag
\end{align}
Here \(U>0\) denotes a uniform bound for the Sobolev norms of the input classes,
and \(C>0\) depends only on \(d,n_1,n_2,k_1,k_2\).

It remains to control the Sobolev norm of \(g_m\). By Lemma~\ref{lem:lebesgue},
\begin{align}
\|g_m\|_{W^{k_2,\infty}(\Omega)}
=
\|\mathcal P_m\mathcal D_m g\|_{W^{k_2,\infty}(\Omega)}
\le
C S m^{2k_2}(\log m)^d .
\end{align}
Therefore,
\begin{align}
\|\mathcal G(f,g)-\mathcal G(f_{\bar m},g_m)\|_{W^{\ell,\infty}(\Omega)}
\le
C U
\left[
L_1\,\bar m^{-n_1}(\log \bar m)^d
\left(1+S m^{2k_2}(\log m)^d\right)^{\alpha_1}
+
L_2\,m^{-n_2}(\log m)^d
\right].
\end{align}
\end{proof}
\subsection{Proofs in Step 2}
\label{proof:step2}
Next, we estimate the error arising from the approximation of the
\(g\)-dependent coefficient. Fix \(f_{\bar m}\) and \(\vx\in\Omega\). Define
\[
F_m^{f_{\bar m},\vx}(\vz)
:=
\fG(f_{\bar m},\fP_m\vz)(\vx),
\qquad
\vz\in[-S,S]^{m_*}.
\]
Then \(F_m^{f_{\bar m},\vx}\) is a function of the finite-dimensional variable
\(\vz=\fD_m g\in\mathbb R^{m_*}\). The following lemma shows that this
function is Lipschitz uniformly in \(f_{\bar m}\) and \(\vx\).

\begin{lemma}[Sobolev regularity of the \(g\)-dependent coefficient map]
\label{lem:Lip-coefficient-g}
Assume that Assumption~\ref{asspde} holds with the
\(W^{\ell,\infty}(\Omega)\)-norm on the left-hand side, where \(\ell\in\mathbb N\). For fixed
\(f_{\bar m}\) and \(\vx\in\Omega\), define
\[
F_m^{f_{\bar m},\vx}(\vz)
:=
\mathcal G(f_{\bar m},\mathcal P_m\vz)(\vx),
\qquad
\vz\in[-S,S]^{m_*}.
\]
Then, for every multi-index \(\gamma\) with \(|\gamma|\le \ell\), the function
\[
F_{m,\gamma}^{f_{\bar m},\vx}(\vz)
:=
\partial_{\vx}^\gamma
\mathcal G(f_{\bar m},\mathcal P_m\vz)(\vx)
\]
belongs to \(W^{1,\infty}([-S,S]^{m_*})\). More precisely,
\[
\operatorname{Lip}\!\left(F_{m,\gamma}^{f_{\bar m},\vx}\right)
\le
C L_1
m^{2k_2}(\log m)^d
\left(
1+
S\bar m^{2k_1}(\log \bar m)^d
\right)^{\alpha_2},
\]
uniformly for all \(|\gamma|\le \ell\), and
\begin{align}
\|F_{m,\gamma}^{f_{\bar m},\vx}\|_{L^\infty([-S,S]^{m_*})}
\le
C\Bigl[
1
&+
L_2 S\bar m^{2k_1}(\log \bar m)^d
\left(
1+
S m^{2k_2}(\log m)^d
\right)^{\alpha_1}+
L_1 S m^{2k_2}(\log m)^d
\Bigr],
\end{align}
where \(C>0\) is independent of \(m,\bar m,f_{\bar m},\vx\), and \(\gamma\).
\end{lemma}

\begin{proof}
Fix a multi-index \(\gamma\) with \(|\gamma|\le \ell\). For any
\(\vz_1,\vz_2\in[-S,S]^{m_*}\), by the strengthened version of
Assumption~\ref{asspde} applied to the second input, we have
\begin{align*}
&\left|
F_{m,\gamma}^{f_{\bar m},\vx}(\vz_1)
-
F_{m,\gamma}^{f_{\bar m},\vx}(\vz_2)
\right|
\le
\left\|
\mathcal G(f_{\bar m},\mathcal P_m\vz_1)
-
\mathcal G(f_{\bar m},\mathcal P_m\vz_2)
\right\|_{W^{\ell,\infty}(\Omega)}
\\
\le&
L_1
\|\mathcal P_m(\vz_1-\vz_2)\|_{W^{k_2,\infty}(\Omega)}
\left(
1+\|f_{\bar m}\|_{W^{k_1,\infty}(\Omega)}
\right)^{\alpha_2}.
\end{align*}
By the Sobolev stability of the reconstruction operator,
\[
\|\mathcal P_m(\vz_1-\vz_2)\|_{W^{k_2,\infty}(\Omega)}
\le
C m^{2k_2}(\log m)^d
\|\vz_1-\vz_2\|_{\ell^\infty},
\]
and
\[
\|f_{\bar m}\|_{W^{k_1,\infty}(\Omega)}
\le
C\bar m^{2k_1}(\log \bar m)^d
\|\mathcal D_{\bar m}f\|_{\ell^\infty}
\le
C S\bar m^{2k_1}(\log \bar m)^d.
\]
Therefore,
\[
\operatorname{Lip}\!\left(F_{m,\gamma}^{f_{\bar m},\vx}\right)
\le
C L_1
m^{2k_2}(\log m)^d
\left(
1+
S\bar m^{2k_1}(\log \bar m)^d
\right)^{\alpha_2}.
\]

It remains to prove the \(L^\infty\)-bound. For any
\(\vz\in[-S,S]^{m_*}\), we write
\begin{align*}
\left|
F_{m,\gamma}^{f_{\bar m},\vx}(\vz)
\right|
&\le
\left\|
\mathcal G(f_{\bar m},\mathcal P_m\vz)
-
\mathcal G(0,\mathcal P_m\vz)
\right\|_{W^{\ell,\infty}(\Omega)}
\\
&\quad+
\left\|
\mathcal G(0,\mathcal P_m\vz)
-
\mathcal G(0,0)
\right\|_{W^{\ell,\infty}(\Omega)}
+
\left\|
\mathcal G(0,0)
\right\|_{W^{\ell,\infty}(\Omega)} .
\end{align*}
Applying the strengthened Assumption~\ref{asspde} to the first input gives
\[
\left\|
\mathcal G(f_{\bar m},\mathcal P_m\vz)
-
\mathcal G(0,\mathcal P_m\vz)
\right\|_{W^{\ell,\infty}(\Omega)}
\le
L_2
\|f_{\bar m}\|_{W^{k_1,\infty}(\Omega)}
\left(
1+\|\mathcal P_m\vz\|_{W^{k_2,\infty}(\Omega)}
\right)^{\alpha_1}.
\]
Similarly, applying it to the second input gives
\[
\left\|
\mathcal G(0,\mathcal P_m\vz)-\mathcal G(0,0)
\right\|_{W^{\ell,\infty}(\Omega)}
\le
L_1
\|\mathcal P_m\vz\|_{W^{k_2,\infty}(\Omega)}.
\]
Using
\[
\|f_{\bar m}\|_{W^{k_1,\infty}(\Omega)}
\le
C S\bar m^{2k_1}(\log \bar m)^d
\]
and
\[
\|\mathcal P_m\vz\|_{W^{k_2,\infty}(\Omega)}
\le
C S m^{2k_2}(\log m)^d,
\]
we obtain
\begin{align*}
\|F_{m,\gamma}^{f_{\bar m},\vx}\|_{L^\infty([-S,S]^{m_*})}
\le
C\Bigl[
1
&+
L_2 S\bar m^{2k_1}(\log \bar m)^d
\left(
1+
S m^{2k_2}(\log m)^d
\right)^{\alpha_1}
\\
&+
L_1 S m^{2k_2}(\log m)^d
\Bigr],
\end{align*}
where the term
\(\|\mathcal G(0,0)\|_{W^{\ell,\infty}(\Omega)}\) is absorbed into the
constant. This completes the proof.
\end{proof}

We next approximate this Lipschitz coefficient map by a finite sum of neural
network basis functions.

\begin{lemma}[Neural network approximation of Lipschitz coefficient maps]
\label{lem:NN-approx-coefficient-g}
Let \(F:[-S,S]^{m_*}\to\mathbb R\) satisfy
\[
\|F\|_{L^\infty([-S,S]^{m_*})}\le B,
\qquad
\operatorname{Lip}(F)\le L_F .
\]
Then, for any \(0\le\varepsilon\le1\), if
\[
N=
\left\lceil
C\sqrt{m_*}\,L_F S\,\varepsilon^{-1}
\right\rceil+1,
\]
there exist points \(\vz_1,\ldots,\vz_{N^{m_*}}\in[-S,S]^{m_*}\) and ReLU
networks
\[
q_s\in
\mathcal F_{\mathrm{NN}}(m_*,1,L_q,p_q,K_q,\kappa_q,1),
\qquad s=1,\ldots,N^{m_*},
\]
such that
\[
\left\|
F-\sum_{s=1}^{N^{m_*}}F(\vz_s)q_s
\right\|_{L^\infty([-S,S]^{m_*})}
\le \varepsilon,\qquad \left\|
\sum_{s=1}^{N^{m_*}}q_s
\right\|_{L^\infty([-S,S]^{m_*})}\le 2,\qquad q_s\ge0.
\]
Moreover, one may choose \(p_q=\mathcal O(1)\),
\[
L_q=\mathcal O\!\left(m_*^2\log m_*+m_*^2\log(\varepsilon^{-1})\right),
\qquad
K_q=\mathcal O\!\left(m_*^2\log m_*+m_*^2\log(\varepsilon^{-1})\right),
\]
and
\[
\kappa_q
=
\mathcal O\!\left(
m_*B\,\varepsilon^{-1}
\left(\sqrt{m_*}L_FS\varepsilon^{-1}\right)^{m_*}
\right).
\]
The implicit constants are independent of \(m_*\), \(F\), and \(\varepsilon\).
\end{lemma}

\begin{proof}
The result follows from the Lipschitz-function approximation construction in
\cite[Theorem~5]{liu2024neural}. We recall the parameter dependence for
completeness. Partition \([-S,S]^{m_*}\) into \(N^{m_*}\) cubes with side length
of order \(S/N\), and choose one point \(\vz_s\) in each cube. Let \(q_s\) be a
ReLU realization of the corresponding localized hat basis function. Since
\(\operatorname{Lip}(F)\le L_F\), choosing
\(N\ge C\sqrt{m_*}L_FS\varepsilon^{-1}\) ensures that the oscillation of \(F\)
on each cube is at most \(\varepsilon\). Therefore,
\[
\left\|
F-\sum_{s=1}^{N^{m_*}}F(\vz_s)q_s
\right\|_{L^\infty([-S,S]^{m_*})}
\le \varepsilon .
\]Taking \(F\equiv 1\), we have \(F(\vz_s)=1\) for all \(s\). Hence, since
\(\varepsilon\le 1\),
\[
\left\|
\sum_{s=1}^{N^{m_*}} q_s
\right\|_{L^\infty([-S,S]^{m_*})}
=
\left\|
\sum_{s=1}^{N^{m_*}} F(\vz_s) q_s
\right\|_{L^\infty([-S,S]^{m_*})}
\le 1+\varepsilon \le 2 .
\]

It remains to record the network parameters. By the construction in
\cite[Theorem~5]{liu2024neural}, each \(q_s\) can be represented by a ReLU
network with \(M=1\), \(p_q=\mathcal O(1)\), and
\[
L_q=\mathcal O(m_*\log(\delta^{-1})),
\qquad
K_q=\mathcal O(m_*\log(\delta^{-1})),
\]
where \(\delta=\varepsilon/(2m_*N^{m_*}B)\). Hence
\[
\delta^{-1}
=
\mathcal O(m_*N^{m_*}B\varepsilon^{-1})
=
\mathcal O\!\left(
m_*B
\left(\sqrt{m_*}L_FS\varepsilon^{-1}\right)^{m_*}
\varepsilon^{-1}
\right).
\]
The weights and biases are bounded by \(\mathcal O(\delta^{-1}+N)\), which gives
the stated bound for \(\kappa_q\). Substituting this estimate into
\(L_q\) and \(K_q\) gives the displayed bounds.
\end{proof}

\begin{proof}[Proof of Proposition~\ref{prop:step2-g-coefficient}]
By Lemma~\ref{lem:Lip-coefficient-g}, for every multi-index \(\gamma\) with
\(|\gamma|\le \ell\), the function
\(F_{m,\gamma}^{f_{\bar m},\vx}(\vz):=
\partial_{\vx}^\gamma
\mathcal G(f_{\bar m},\mathcal P_m\vz)(\vx)\) satisfies
\(\operatorname{Lip}(F_{m,\gamma}^{f_{\bar m},\vx})
\le CA_{m,\bar m}^{(g)}\) and
\(\|F_{m,\gamma}^{f_{\bar m},\vx}\|_{L^\infty([-S,S]^{m_*})}
\le CB_{m,\bar m}^{(g)}\), uniformly in \(\vx\in\Omega\) and
\(|\gamma|\le \ell\). Here $C$ is independent of $m,\bar m$.

Apply Lemma~\ref{lem:NN-approx-coefficient-g} with
\(L_F=CA_{m,\bar m}^{(g)}\) and \(B=CB_{m,\bar m}^{(g)}\). Let
\(\varepsilon_B=C S A_{m,\bar m}^{(g)}\sqrt{m_*}\,J^{-1/m_*}\), and choose
\(J\) large enough so that \(\varepsilon_B\le1\). Equivalently, we choose
\(N\asymp J^{1/m_*}\), so that \(N^{m_*}\asymp J\). Since the points
\(\vz_s\) and the networks \(q_s\) in Lemma~\ref{lem:NN-approx-coefficient-g}
depend only on the grid and the parameters \(L_F,B,\varepsilon_B\), but not on
the particular function \(F\), the same points and the same networks can be used
simultaneously for all derivative coefficient maps
\(F_{m,\gamma}^{f_{\bar m},\vx}\), \(|\gamma|\le \ell\).

Thus, for every \(|\gamma|\le \ell\),
\[
\left\|
F_{m,\gamma}^{f_{\bar m},\vx}
-
\sum_{s=1}^{J}
F_{m,\gamma}^{f_{\bar m},\vx}(\vz_s)
\mathcal B_s
\right\|_{L^\infty([-S,S]^{m_*})}
\le
\varepsilon_B,
\]
and the branch networks satisfy
\(\|\sum_{s=1}^{J}\mathcal B_s\|_{L^\infty([-S,S]^{m_*})}\le2\) and
\(\mathcal B_s\ge0\).

Define
\[
\mathfrak d_s(F_m^{f_{\bar m}},\vx)
:=
F_m^{f_{\bar m},\vx}(\vz_s)
=
\mathcal G(f_{\bar m},\mathcal P_m\vz_s)(\vx).
\]
Then, for every \(|\gamma|\le \ell\),
\(\partial_{\vx}^\gamma\mathfrak d_s(F_m^{f_{\bar m}},\vx)
=
F_{m,\gamma}^{f_{\bar m},\vx}(\vz_s)\). Taking
\(\vz=\mathcal D_m g\) and using \(g_m=\mathcal P_m\mathcal D_m g\), we obtain
\begin{align*}
&\left|
\partial_{\vx}^\gamma
\left[
\mathcal G(f_{\bar m},g_m)(\vx)
-
\sum_{s=1}^{J}
\mathfrak d_s(F_m^{f_{\bar m}},\vx)
\mathcal B_s(\mathcal D_m g)
\right]
\right|
\\
&\qquad=
\left|
F_{m,\gamma}^{f_{\bar m},\vx}(\mathcal D_m g)
-
\sum_{s=1}^{J}
F_{m,\gamma}^{f_{\bar m},\vx}(\vz_s)
\mathcal B_s(\mathcal D_m g)
\right|
\le
\varepsilon_B .
\end{align*}
Taking the supremum over \(\vx\in\Omega\) and all \(|\gamma|\le \ell\) gives
\[
\left\|
\mathcal G(f_{\bar m},g_m)
-
\sum_{s=1}^{J}
\mathfrak d_s(F_m^{f_{\bar m}},\cdot)
\mathcal B_s(\mathcal D_m g)
\right\|_{W^{\ell,\infty}(\Omega)}
\le
C\varepsilon_B.
\]
This proves \eqref{eq:step2-Well-bound}.

It remains to record the network parameters. By
Lemma~\ref{lem:NN-approx-coefficient-g}, we may choose \(p_B=\mathcal O(1)\) and
\[
L_B=\mathcal O(m_*\log(\delta_B^{-1})),
\qquad
K_B=\mathcal O(m_*\log(\delta_B^{-1})),
\]
where
\[
\delta_B
=
\frac{\varepsilon_B}{2Cm_*J B_{m,\bar m}^{(g)}}.
\]
Since
\(\varepsilon_B
=
C S A_{m,\bar m}^{(g)}\sqrt{m_*}J^{-1/m_*}\), we have
\[
\delta_B^{-1}
=
\mathcal O\!\left(
m_*B_{m,\bar m}^{(g)}
\left(S A_{m,\bar m}^{(g)}\sqrt{m_*}\right)^{-1}
J^{1+\frac1{m_*}}
\right).
\]
Thus
\[
L_B
=
\mathcal O\!\left(m_*^2\log m_*+m_*\log J\right),
\qquad
K_B
=
\mathcal O\!\left(m_*^2\log m_*+m_*\log J\right),
\]
where logarithmic factors from \(A_{m,\bar m}^{(g)}\) and
\(B_{m,\bar m}^{(g)}\) are absorbed into the \(m_*^2\log m_*\) term.

Finally, by Lemma~\ref{lem:NN-approx-coefficient-g}, the weights and biases
satisfy
\[
\kappa_B
=
\mathcal O\!\left(
\sqrt{m_*}B_{m,\bar m}^{(g)}
J^{1+\frac1{m_*}}
\right).
\]
This completes the proof.
\end{proof}

Here, we need the following lemma, which will be used in the final
trunk-approximation step.

\begin{lemma}[Spatial regularity of the coefficient functions]
\label{lem:spatial-Wninf-coeff-local}
Fix \(f_{\bar m}\). For \(\vx\in\Omega\), define
\[
F_m^{f_{\bar m},\vx}(\vz)
:=
\mathcal G(f_{\bar m},\mathcal P_m\vz)(\vx),
\qquad \vz\in[-S,S]^{m_*}.
\]
Let \(c_s^{f_{\bar m}}(\vx):=\mathfrak d_s(F_m^{f_{\bar m},\vx})\), where
\(\mathfrak d_s(F):=F(\vz_s)\) for some fixed
\(\vz_s\in[-S,S]^{m_*}\). Assume that
\[
K_m^{(n_3)}(f_{\bar m})
:=
\sup_{\vz\in[-S,S]^{m_*}}
\|\mathcal G(f_{\bar m},\mathcal P_m\vz)\|_{W^{n_3,\infty}(\Omega)}
<\infty .
\]
Then \(c_s^{f_{\bar m}}\in W^{n_3,\infty}(\Omega)\), and
\[
\|c_s^{f_{\bar m}}\|_{W^{n_3,\infty}(\Omega)}
\le
K_m^{(n_3)}(f_{\bar m}) .
\]
\end{lemma}

\begin{proof}
Since \(\mathfrak d_s\) is point evaluation at \(\vz_s\), we have
\[
c_s^{f_{\bar m}}(\vx)
=
F_m^{f_{\bar m},\vx}(\vz_s)
=
\mathcal G(f_{\bar m},\mathcal P_m\vz_s)(\vx).
\]
Thus, for every multi-index \(\beta\) with \(|\beta|\le n_3\),
\[
D_{\vx}^{\beta}c_s^{f_{\bar m}}(\vx)
=
D_{\vx}^{\beta}
\mathcal G(f_{\bar m},\mathcal P_m\vz_s)(\vx).
\]
Taking the \(L^\infty(\Omega)\)-norm and then the maximum over
\(|\beta|\le n_3\), we obtain
\[
\|c_s^{f_{\bar m}}\|_{W^{n_3,\infty}(\Omega)}
\le
\sup_{\vz\in[-S,S]^{m_*}}
\|\mathcal G(f_{\bar m},\mathcal P_m\vz)\|_{W^{n_3,\infty}(\Omega)}
=
K_m^{(n_3)}(f_{\bar m}).
\]
\end{proof}

\subsection{Proofs in Step 3}
\label{proof:step3}

The proof of Proposition~\ref{prop:step3-f-coefficient} is similar to that of
Proposition~\ref{prop:step2-g-coefficient}, and is therefore omitted. It remains
to record the spatial regularity of the coefficient functions \(e_{h,s}\), which
will be used in the trunk approximation step.

\begin{lemma}[Spatial regularity of \(e_{h,s}\)]
\label{lem:spatial-regularity-ehs}
Assume that
\[
K_{\bar m,m}^{(n_3)}
:=
\sup_{\substack{\vw\in[-S,S]^{\bar m_*}\\
\vz\in[-S,S]^{m_*}}}
\left\|
\mathcal G(\mathcal P_{\bar m}\vw,\mathcal P_m\vz)
\right\|_{W^{n_3,\infty}(\Omega)}
<\infty .
\]
Then, for every \(h=1,\ldots,H\) and \(s=1,\ldots,J\), we have
\(e_{h,s}\in W^{n_3,\infty}(\Omega)\), and
\[
\|e_{h,s}\|_{W^{n_3,\infty}(\Omega)}
\le
K_{\bar m,m}^{(n_3)} .
\]
\end{lemma}

\begin{proof}
By construction, both \(\mathfrak d_s\) and \(\mathfrak e_{h,s}\) are
point-evaluation functionals. Hence there exist points
\(\vz_s\in[-S,S]^{m_*}\) and \(\vw_{h,s}\in[-S,S]^{\bar m_*}\) such that
\(\mathfrak d_s(F)=F(\vz_s)\) and
\(\mathfrak e_{h,s}(\bar F)=\bar F(\vw_{h,s})\). Therefore,
\[
e_{h,s}(\vx)
=
\bar F_{\bar m,s}^{\vx}(\vw_{h,s})
=
F_m^{\mathcal P_{\bar m}\vw_{h,s},\vx}(\vz_s)
=
\mathcal G(\mathcal P_{\bar m}\vw_{h,s},\mathcal P_m\vz_s)(\vx).
\]
Thus, for every multi-index \(\beta\) with \(|\beta|\le n_3\),
\[
D_{\vx}^{\beta}e_{h,s}(\vx)
=
D_{\vx}^{\beta}
\mathcal G(\mathcal P_{\bar m}\vw_{h,s},\mathcal P_m\vz_s)(\vx).
\]
Taking the \(L^\infty(\Omega)\)-norm and then the maximum over
\(|\beta|\le n_3\), we obtain
\[
\|e_{h,s}\|_{W^{n_3,\infty}(\Omega)}
\le
K_{\bar m,m}^{(n_3)} .
\]
This completes the proof.
\end{proof}

\subsection{Proofs in Step 4}
\label{proof:step4}
For the last term, we approximate the remaining spatial coefficient functions
\(e_{h,s}(\vx)\) by trunk networks. The trunk basis is chosen with uniformly
bounded local overlap, which will be used to control the sum over the branch
indices.

\begin{proof}
For the reconstructed inputs, define
\[
K_{\bar m,m}^{(n_3)}
:=
\sup_{\substack{\vw\in[-S,S]^{\bar m_*}\\
\vz\in[-S,S]^{m_*}}}
\left\|
\mathcal G(\mathcal P_{\bar m}\vw,\mathcal P_m\vz)
\right\|_{W^{n_3,\infty}(\Omega)} .
\]
By the Sobolev stability of the reconstruction operators, for
\(\vw\in[-S,S]^{\bar m_*}\) and \(\vz\in[-S,S]^{m_*}\), we have
\[
\|\mathcal P_{\bar m}\vw\|_{W^{\beta_1,\infty}(\Omega)}
\le
C S\bar m^{2\beta_1}(\log \bar m)^d,
\qquad
\|\mathcal P_m\vz\|_{W^{\beta_2,\infty}(\Omega)}
\le
C S m^{2\beta_2}(\log m)^d.
\]
Therefore, Assumption~\ref{assump:G-output-regularity} gives
\[
K_{\bar m,m}^{(n_3)}
\le
C
\left(
1+S\bar m^{2\beta_1}(\log \bar m)^d
\right)
\left(
1+S m^{2\beta_2}(\log m)^d
\right)
=
C R_{\bar m,m}^{(n_3)} .
\]

By Lemma~\ref{lem:spatial-regularity-ehs}, for every
\(h=1,\ldots,H\) and \(s=1,\ldots,J\), we have
\(e_{h,s}\in W^{n_3,\infty}(\Omega)\) and
\(\|e_{h,s}\|_{W^{n_3,\infty}(\Omega)}
\le K_{\bar m,m}^{(n_3)}\). Applying
Proposition~\ref{prop:relu-sobolev-local-basis-Well} to \(e_{h,s}\), with
\(n=n_3\), \(R=K_{\bar m,m}^{(n_3)}\), and the chosen \(\ell\in\{0,1\}\), gives
coefficients \(e_{h,p,s}\in\mathbb R\) and trunk networks
\(\mathcal T_p\in
\mathcal F_{\mathrm{NN}}(d,1,L_T,p_T,K_T,\kappa_T,M_T)\),
\(p=1,\ldots,P\), such that
\[
\left\|
e_{h,s}
-
\sum_{p=1}^{P}
e_{h,p,s}\mathcal T_p
\right\|_{W^{\ell,\infty}(\Omega)}
\le
C K_{\bar m,m}^{(n_3)}
P^{-\frac{n_3-\ell}{d}},
\]
and \(|e_{h,p,s}|\le C K_{\bar m,m}^{(n_3)}\). Since
\(K_{\bar m,m}^{(n_3)}\le C R_{\bar m,m}^{(n_3)}\), this proves
\eqref{eq:trunk-approx-Well} and the stated coefficient bound.

The same application of Proposition~\ref{prop:relu-sobolev-local-basis-Well}
also gives the local-overlap bounds
\[
\sup_{\vx\in\Omega}
\sum_{p=1}^{P}
|\mathcal T_p(\vx)|
\le C,
\]
and, when \(\ell=1\),
\[
\sup_{\vx\in\Omega}
\sum_{p=1}^{P}
\sum_{r=1}^{d}
|\partial_{x_r}\mathcal T_p(\vx)|
\le
C P^{1/d}.
\]
The bounds \(L_T\le C\log P\), \(p_T\le C\), \(K_T\le C\log P\),
\(\kappa_T\le C P^{(n_3+d-\ell)/d}\), and \(M_T\le C\) also follow directly
from Proposition~\ref{prop:relu-sobolev-local-basis-Well}. This completes the
proof.
\end{proof}

\subsection{Proof of Corollaries \ref{cor:balanced-main-final-functionL} and \ref{cor:more-balanced-main-final-functionL}}
\label{proof:corollary}

The proof of Corollary~\ref{cor:balanced-main-final-functionL} is a special
case of Corollary~\ref{cor:more-balanced-main-final-functionL}. Therefore, it
suffices to prove the latter.

\begin{proof}[Proof of Corollary~\ref{cor:more-balanced-main-final-functionL}]
Applying the general pseudo-spectral projection estimate in Sobolev norms and
following the same shared branch--trunk approximation argument as in the
two-input case gives \eqref{eq:most-general-shared}. We now consider the
special case \(\alpha_i=\beta_i=0\), \(i=1,\ldots,\lambda\). Then
\eqref{eq:most-general-shared} reduces to
\begin{align}
E_{\mathrm{total}}
\lesssim\;&
\sum_{i=1}^{\lambda}
M_i^{-\frac{n_i-2k_i}{d_i}}(\log M_i)^{d_i}
+
\sum_{i=1}^{\lambda}
S L_i
\sqrt{M_i}\,
M_i^{\frac{2k_i}{d_i}}
(\log M_i)^{d_i}
J_i^{-1/M_i}
\notag\\
&+
P^{-\frac{n_{\lambda+1}-\ell}{d_{\lambda+1}}}
\prod_{i=1}^{\lambda}
(1+S(\log M_i)^{d_i}).
\label{eq:lambda-error-alpha-beta-zero-shared}
\end{align}
Set \(A_i:=(n_i-2k_i)/d_i\), \(B_i:=2k_i/d_i\), and
\(\nu_\ell:=(n_{\lambda+1}-\ell)/d_{\lambda+1}\). Since \(n_i>2k_i\), we have
\(A_i>0\).

We first balance the \(i\)-th discretization term with the corresponding branch
approximation term:
\[
M_i^{-A_i}(\log M_i)^{d_i}
\asymp
S L_i\sqrt{M_i}\,M_i^{B_i}(\log M_i)^{d_i}J_i^{-1/M_i}.
\]
Canceling the common logarithmic factor gives
\(J_i^{1/M_i}\asymp S L_iM_i^{A_i+B_i+1/2}
=S L_iM_i^{n_i/d_i+1/2}\). Hence
\[
\log J_i
\asymp
M_i\log M_i,
\qquad
i=1,\ldots,\lambda,
\]
up to constants depending only on \(d_i,n_i,k_i,S,L_i\).

Let \(0<\varepsilon<1\) be the target accuracy. Choose \(M_i\) so that
\(M_i^{-A_i}(\log M_i)^{d_i}\asymp\varepsilon\). Equivalently, up to logarithmic
factors,
\[
M_i
\asymp
\varepsilon^{-q_i}
\left(\log\frac1\varepsilon\right)^{d_iq_i},
\qquad
q_i:=\frac1{A_i}=\frac{d_i}{n_i-2k_i}.
\]
Consequently,
\[
\log J_i
\lesssim
\varepsilon^{-q_i}
\left(\log\frac1\varepsilon\right)^{d_iq_i+1}.
\]
Define
\(Q_{\max}:=\max_{1\le i\le\lambda}q_i
=\max_{1\le i\le\lambda}d_i/(n_i-2k_i)\). Since \(\lambda\) is fixed,
\[
\sum_{i=1}^{\lambda}\log J_i
\lesssim
\varepsilon^{-Q_{\max}}
\left(\log\frac1\varepsilon\right)^C .
\]

It remains to choose \(P\). Since \(\log M_i\asymp\log(1/\varepsilon)\), we have
\[
\prod_{i=1}^{\lambda}(1+S(\log M_i)^{d_i})
\lesssim
C\left(\log\frac1\varepsilon\right)^{\sum_{i=1}^{\lambda}d_i}.
\]
Thus it is enough to impose
\[
P^{-\nu_\ell}
\left(\log\frac1\varepsilon\right)^{\sum_{i=1}^{\lambda}d_i}
\lesssim
\varepsilon.
\]
For instance, we may take
\[
P
\asymp
\varepsilon^{-1/\nu_\ell}
\left(\log\frac1\varepsilon\right)^{
\frac{1}{\nu_\ell}\sum_{i=1}^{\lambda}d_i
}.
\]
With these choices of \(M_i,J_i\), and \(P\), all terms in
\eqref{eq:lambda-error-alpha-beta-zero-shared} are bounded by \(C\varepsilon\).
Hence \(E_{\mathrm{total}}\lesssim\varepsilon\).

We now express \(\varepsilon\) in terms of \(N_{\mathrm{tot}}\). For the shared
architecture \eqref{eq:lambda-shared-architecture}, there are \(J_i\) branch
networks for the \(i\)-th input, \(P\) trunk networks, and
\(P\prod_{i=1}^{\lambda}J_i\) scalar coefficients. Hence
\[
N_{\mathrm{tot}}
\lesssim
\sum_{i=1}^{\lambda}J_iK_i
+
PK_T
+
P\prod_{i=1}^{\lambda}J_i,
\]
where \(K_i\) and \(K_T\) denote the sizes of the corresponding branch and trunk
networks. These factors grow at most polynomially in \(M_i,\log J_i\), and
\(\log P\), and therefore only contribute lower-order logarithmic factors.
Thus
\[
\log N_{\mathrm{tot}}
\lesssim
\sum_{i=1}^{\lambda}\log J_i+\log P
+
\text{lower-order logarithmic factors}
\lesssim
\varepsilon^{-Q_{\max}}
\left(\log\frac1\varepsilon\right)^C .
\]
Inverting this relation gives, up to logarithmic factors,
\[
\varepsilon
\lesssim
\left(
\frac{\log N_{\mathrm{tot}}}{\log\log N_{\mathrm{tot}}}
\right)^{-1/Q_{\max}}.
\]
Moreover, since \(\log(1/\varepsilon)\asymp\log\log N_{\mathrm{tot}}\), the
explicit logarithmic factor from the trunk term satisfies
\[
\prod_{i=1}^{\lambda}(1+S(\log M_i)^{d_i})
\lesssim
C(\log\log N_{\mathrm{tot}})^{\sum_{i=1}^{\lambda}d_i}.
\]
Therefore,
\[
E_{\mathrm{total}}
\lesssim
\left(
\frac{\log N_{\mathrm{tot}}}{\log\log N_{\mathrm{tot}}}
\right)^{-1/Q_{\max}}
(\log\log N_{\mathrm{tot}})^{\sum_{i=1}^{\lambda}d_i}.
\]
Since
\(1/Q_{\max}
=
\min_{1\le i\le\lambda}(n_i-2k_i)/d_i\), this can be written as
\[
E_{\mathrm{total}}
\lesssim
\left(
\frac{\log N_{\mathrm{tot}}}{\log\log N_{\mathrm{tot}}}
\right)^{
-\min_{1\le i\le\lambda}\frac{n_i-2k_i}{d_i}
}
(\log\log N_{\mathrm{tot}})^{\sum_{i=1}^{\lambda}d_i}.
\]
This completes the proof.
\end{proof}

\section{Proof for Generalization error}
\subsection{Proof of Lemma \ref{lem:T1-sobolev-noise}}
\label{proof:lem:T1-sobolev-noise}
\begin{proof}[Proof of Lemma \ref{lem:T1-sobolev-noise}]
Condition on the training inputs and locations. Then the empirical sample set
is fixed. Applying the same finite-net and sub-Gaussian maximal inequality
argument as in \cite[Lemma~3]{liu2024neural} gives the desired estimate with
the covering number taken with respect to the empirical metric on this fixed
sample. Since a \(d_{\mathcal S,\ell}\)-cover also controls the empirical
Sobolev norm \(\|\cdot\|_{n,\ell}\), and since
\(\mathcal N(\eta,\mathcal F_G,d_{\mathcal S,\ell})
\le
\mathcal N_{\rm emp}(\eta,\mathcal F_G)\), the same estimate holds with
\(\mathcal N_{\rm emp}(\eta,\mathcal F_G)\). Taking expectation over the
training samples completes the proof.
\end{proof}

\subsection{Proof of Lemma \ref{lem:T2-sobolev}}
\label{proof:lem:T2-sobolev}
\begin{proof}[Proof of Lemma \ref{lem:T2-sobolev}]
The proof follows the ghost-sample and finite-covering argument in
\cite[Lemma~4]{liu2024neural}. For a more detailed presentation of the same
empirical-process argument, see also \cite[Proof of Lemma~29]{liu2022deep}.
\end{proof}

\subsection{Proof of Lemma \ref{lem:pdim-derivative-class}}
\label{proof:lem:pdim-derivative-class}
\begin{proof}[Proof of Lemma \ref{lem:pdim-derivative-class}]
The proof follows the same argument as
\cite[Proposition~5]{yang2025deeponet}. We only indicate the minor difference.
In \cite[Proposition~5]{yang2025deeponet}, the pseudo-dimension estimate for
the trunk network contains a cubic dependence on the depth because the trunk
network uses the squared ReLU activation. In the present setting, the trunk
network uses the standard ReLU activation, and therefore the corresponding
depth dependence is quadratic.

Applying the same counting argument to the \(g\)-branch networks, the
\(f\)-branch networks, the trunk networks, and the scalar coefficients
\(e_{h,p,s}\), we obtain
\[
\operatorname{Pdim}(\mathcal F_G^\gamma)
\le
C JHP\left(
L_B^2p_B^2
+
L_E^2p_E^2
+
L_T^2p_T^2
\right).
\]
A more careful inspection of the proof of
\cite[Proposition~5]{yang2025deeponet} yields the sharper estimate
\[
\operatorname{Pdim}(\mathcal F_G^\gamma)
\le
C\left(
J L_B^2p_B^2
+
H L_E^2p_E^2
+
P L_T^2p_T^2
+
JHP
\right).
\]
However, this sharper estimate is not needed for the final balanced rate. The
coarser bound above already gives stochastic terms that can be dominated by the
approximation error after choosing \(H\) as a sufficiently small power of the
outer sample size. Therefore, for readability, we use the simpler bound.
\end{proof}
\end{document}